\documentclass{article}

\usepackage{microtype}
\usepackage{graphicx}
\usepackage{subfigure}
\usepackage{booktabs} % for professional tables

\usepackage{hyperref}

\usepackage[accepted]{icml2024}

\usepackage{amsmath}
\usepackage{amssymb}
\usepackage{mathtools}
\usepackage{amsthm}

\usepackage[capitalize,noabbrev]{cleveref}

\theoremstyle{plain}
\usepackage{amsmath,amsfonts,bm}
\usepackage{amsthm}
\usepackage{amssymb}
\def\eqref#1{equation~\ref{#1}}
\def\1{\bm{1}}

\def \Poly{\operatorname{Poly}(L_0,L_1, \sigma_0,\sigma_1, \frac{1}{1-\bone}, f(\bw_1)-f^*)}
\def\bnu{\boldsymbol{\nu}}
\DeclareMathAlphabet{\mathsfit}{\encodingdefault}{\sfdefault}{m}{sl}
\SetMathAlphabet{\mathsfit}{bold}{\encodingdefault}{\sfdefault}{bx}{n}

\def\gF{{\mathcal{F}}}

\def\bw{\boldsymbol{w}}
\def\bg{\boldsymbol{g}}
\def\bG{\boldsymbol{G}}
\def\bom{\boldsymbol{m}}

\newcommand{\E}{\mathbb{E}}

\newcommand{\btwo}{\beta_2}

\newcommand{\dE}{\mathbb{E}}

\newcommand{\fil}{\mathcal{F}}
\def\bone{\beta_1}
\def\btwo{\beta_2}
\def\bO{\boldsymbol{O}}

\newcommand{\bnut}{\widetilde{\bnu}}

\newtheorem{assumption}{Assumption}
\newtheorem{lemma}{Lemma}
\newtheorem{proposition}{Proposition}

\newtheorem{theorem}{Theorem}
\newtheorem{remark}{Remark}
\usepackage{dsfont}

\def\bu{\boldsymbol{u}}

\usepackage[textsize=tiny]{todonotes}

\icmltitlerunning{On the Convergence of Adam under Non-uniform Smoothness}

\begin{document}

\twocolumn[
\icmltitle{On the Convergence of Adam under Non-uniform Smoothness: Separability from SGDM and Beyond}



\icmlsetsymbol{equal}{*}

\begin{icmlauthorlist}
\icmlauthor{Bohan Wang}{yyy}
\icmlauthor{Huishuai Zhang}{xxx}
\icmlauthor{Qi Meng}{cas}
\icmlauthor{Ruoyu Sun}{cuhk}
\icmlauthor{Zhi-Ming Ma}{cas}
\icmlauthor{Wei Chen}{cas}
\end{icmlauthorlist}

\icmlaffiliation{yyy}{University of Science and Technology of China \& Microsoft Research}
\icmlaffiliation{xxx}{Peking University}
\icmlaffiliation{cas}{Chinese Academy of Science}
\icmlaffiliation{cuhk}{CUHK (Shenzhen)}

\icmlcorrespondingauthor{Huishuai Zhang}{zhanghuishuai@pku.edu.cn}
\icmlcorrespondingauthor{Wei Chen}{chenwei2022@ict.ac.cn}

\icmlkeywords{Machine Learning, ICML}

\vskip 0.3in
]



\printAffiliationsAndNotice{\icmlEqualContribution} % otherwise use the standard text.

\begin{abstract}
%Despite the significant success of the Adam optimizer in practical deep learning applications, few works have theoretically demonstrated its advantages. This paper aims to delineate the distinction between SGDM and Adam from the perspective of convergence rate, proving that Adam's convergence is faster than that of SGDM. Under the assumption of non-uniformly bounded smoothness, we show that (1) in deterministic settings, Adam can match the known lower bound of the convergence rate of deterministic  first-order optimizers, yet GDM's convergence rate has a higher dependence over the initial function value; (2) in stochastic settings, the upper bound of Adam's convergence rate can match the lower bounds of stochastic first-order optimizers with respect to the initial function value and the final error, but there exist counterexamples where SGDM diverges with any learning rate. These results successfully separate Adam and SGDM in terms of convergence rate. With the aid of a novel stopping time, we further prove that if we consider the minimum error point during iterations, the corresponding convergence rate can match the lower bounds across all problem hyperparameters. We demonstrate that this stopping time can be of independent interest by proving that Adam with a specific hyperparameter scheduler can also be parameter-agnostic based on it.

This paper aims to clearly distinguish between Stochastic Gradient Descent with Momentum (SGDM) and Adam in terms of their convergence rates. We demonstrate that Adam achieves a faster convergence compared to SGDM under the condition of non-uniformly bounded smoothness. Our findings reveal that: (1) in deterministic environments, Adam can attain the known lower bound for the convergence rate of deterministic first-order optimizers, whereas the convergence rate of Gradient Descent with Momentum (GDM) has higher order dependence on the initial function value; (2) in stochastic setting, Adam's convergence rate upper bound matches the lower bounds of stochastic first-order optimizers, considering both the initial function value and the final error, whereas there are instances where SGDM fails to converge with any learning rate. These insights distinctly differentiate Adam and SGDM regarding their convergence rates. Additionally, by introducing a novel stopping-time based technique, we further prove that if we consider the minimum gradient norm during iterations, the corresponding convergence rate can match the lower bounds across all problem hyperparameters. The technique  can also help proving that Adam with a specific hyperparameter scheduler is parameter-agnostic, which hence can be of independent interest.
\end{abstract}

\section{Introduction}
Among various optimization techniques, the Adam optimizer~\cite{kingma2014adam,loshchilov2019adamw} stands out due to its empirical success in a wide range of deep learning applications, especially for  pre-training large foundation models with enormous data~\cite{touvron2023llama,brown2020gpt3,zhang2022opt,rae2021scaling,chowdhery2022palm,du2021glm}. This popularity of Adam can be attributed to its adaptive learning rate mechanism, which smartly adjusts the step size for each parameter, allowing flexible and robust learning rate choices. Adam's versatility is further highlighted by its consistent performance in training various kinds of models, making it a preferred optimizer in both academic and industrial settings \cite{schneider2021HITY}. Its empirical success extends beyond standard benchmarks to real-world challenges, where it often delivers state-of-the-art results. This track record solidifies Adam's position as a fundamental tool for deep learning practitioners.

%In the ever-evolving field of deep learning, the quest for efficient and robust optimization algorithms is perpetual. Among the myriad of optimization techniques, the Adam optimizer has emerged as a particularly potent tool, distinguishing itself through its empirical success across a diverse array of deep learning applications. Since its inception by \cite{kingma2014adam}, Adam has been widely adopted by the deep learning community, largely due to its adaptive learning rate mechanism which cleverly adjusts the step size for each parameter. The versatility of Adam is further underscored by its consistent performance in training complex models \cite{xxx}, making it a go-to optimizer in both academic research and industry applications. Its empirical success is not just limited to standard benchmarks \cite{xxx}, but also extends to real-world problems where it often achieves state-of-the-art results \cite{xxx}, thereby solidifying its position as a cornerstone in the toolkit of modern deep learning practitioners.

Exploring the theoretical foundations of the Adam optimizer, particularly why it often outperforms traditional optimizers like Stochastic Gradient Descent with Momentum (SGDM), is an intriguing yet complex task. Understanding Adam's convergence behavior is challenging, especially in settings defined by standard convergence rate analysis. In these settings, assumptions include uniformly bounded smoothness and finite gradient noise variance. Current research indicates that under these conditions, SGDM can attain the lower bound of the convergence rate for all first-order optimizers \cite{carmon2017lower}. This finding implies that, theoretically, Adam's convergence rate should not exceed that of SGDM. This theoretical result contrasts with practical observations where Adam frequently excels, presenting a fascinating challenge for researchers. It highlights the need for more refined theoretical models that can bridge the gap between Adam's empirical success and its theoretical understanding.

%Building on the practical triumphs of the Adam optimizer, it is indeed tempting to delve into the theoretical underpinnings that could explain why Adam often converges faster than more traditional optimizers such as Stochastic Gradient Descent with Momentum (SGDM). Nevertheless, unraveling the intricacies of its convergence behavior is far from straightforward. In scenarios characterized by classic convergence rate analysis, where assumptions of smoothness are uniformly bounded and gradient noise variance is constrained, existing literature has substantiated that SGDM can achieve the convergence rate lower bound for all first-order optimizers \cite{xx}. This suggests that under such conditions, the convergence rate of Adam should not surpass that of SGDM. This disconnect between theory and practice poses a compelling challenge for researchers, driving investigations into more nuanced theoretical models that can reconcile Adam's empirical effectiveness with a solid theoretical foundation.

Recent research by \citet{zhang2019gradient} has provided valuable insights into the complexity of neural network optimization, particularly challenging the assumption of uniform bounded smoothness. Their observations indicate that smoothness often varies, showing a positive correlation with the norm of the gradient and experiencing considerable fluctuations during the optimization process. Building on this, they introduce the \((L_0, L_1)\)-smooth condition (detailed in our Assumption \ref{assum: objective}), which posits that local smoothness can be bounded in relation to the gradient norm. This concept presents an exciting opportunity to theoretically demonstrate that Adam could potentially converge faster than SGDM. However, even in the relatively simpler deterministic settings, no study has yet conclusively shown this to be the case.

%Recent insights provided by \citep{zhang2019gradient} have shed light on the complex nature of neural network optimization, where the assumption of uniform bounded smoothness does not hold. Instead, they observed that smoothness is often positively correlated with the norm of the gradient, exhibiting significant fluctuations throughout the optimization process. Based on this observation, they propose the $(L_0, L_1)$-smooth condition (please see our Assumption \ref{assum: objective}), suggesting that local smoothness could be bounded by the gradient norm. This condition opens a promising avenue for proving that Adam could potentially converges faster than SGDM. However, to date, even in the more tractable deterministic settings, no work has successfully demonstrated this.

To effectively compare the convergence rates of Adam and Stochastic Gradient Descent with Momentum (SGDM), it's essential to establish an upper bound on Adam's convergence rate and a lower bound for SGDM, and then prove Adam's superiority. This endeavor faces several challenges. First, the known lower bound for SGDM's convergence rate is only available in deterministic settings without momentum \cite{zhang2019gradient,crawshaw2022robustness}. Moreover, this result is based on a scenario where the counter-example objective function is selected after fixing the learning rate. This procedure deviates from more common practices where the learning rate is adjusted after defining the objective function \cite{drori2020complexity, carmon2017lower,arjevani2022lower}, casting doubts on the standard applicability of this lower bound. Secondly, for Adam, the current assumptions required to derive an upper bound for its convergence rate are quite strict. These include assumptions like bounded adaptive learning rates or deterministically bounded noise \cite{wang2022provable,li2023convergence}. However, even under these constraints, the convergence rates obtained for Adam are weaker than those of algorithms like clipped SGDM~\cite{zhang2019gradient}.

These complexities hinder a straightforward comparison between the convergence rates of Adam and SGDM, highlighting a significant gap in the theoretical understanding that remains to be bridged.

%Specifically, to achieve this goal, one need to establish an upper bound on the convergence rate for Adam and a lower bound for SGDM and then demonstrate that the former is superior to the latter. Several issues compound the difficulty of this task: (1). The lower bound on the convergence rate for SGDM is only available in the deterministic and without-momentum setting \cite{xxx}. What's worse, this result assumes that the counter-example objective function is chosen after the learning rate is picked, which deviates from common lower bounds and practices where the learning rate is tuned after the objective function \cite{xxx}, thereby questioning the standard validity of the convergence rate lower bound. (2). As for Adam, the assumptions required to establish its convergence rate upper bound are rather stringent, such as assuming bounded adaptive learning rates or deterministically bounded noise \cite{xxx}, and even under these conditions, the resulting convergence rates are weaker than those of algorithms like clipped SGD \cite{xxx}. These challenges obscure the direct comparison of convergence rates between Adam and SGDM, leaving a gap in our theoretical understanding that is yet to be bridged.

\textbf{Our contributions.} In this paper, we aim to bridge the gap and summarize our contributions as follows.
\vspace{-3mm}
\begin{itemize}
    \item We separate the convergence rate of Adam and SGDM under $(L_0,L_1)$-smooth condition both in the deterministic setting and in the stochastic setting. 
    \begin{itemize}
        \item 

        In the deterministic setting,  for the first time, we prove that under the \((L_0, L_1)\)-smooth condition, the convergence rate of the Adam optimizer can match the existing lower bound for first-order deterministic optimizers, up to numerical constants. Additionally, we establish a new lower bound for the convergence rate of  GDM, where one is allowed to tune the learning rate and the momentum coefficient after the problem is fixed. The lower bound exhibits a higher order dependence on the initial function value gap compared to the upper bound of Adam. This distinction clearly separates Adam and GDM for the deterministic setting. 
        
        \item 
        In the stochastic setting, for the first time, we prove that under the \((L_0, L_1)\)-smooth condition, the convergence rate of Adam matches the existing lower bound for first-order stochastic optimizers regarding the initial function value $f(\bw_1)-f^*$ and the final error $\varepsilon$. In contrast, counterexamples exist where SGDM fails to converge, irrespective of the learning rate and momentum coefficient. These findings distinctly separate the convergence properties of Adam and SGDM in stochastic settings.
    \end{itemize}
    \vspace{-2mm}
    \item With the aid of a novel stopping time based technique, we further demonstrate that the convergence rate of minimum error point of Adam can match the lower bound across all problem hyperparameters. We demonstrate that such a technique can be of independent interest by proving that Adam with specific scheduler is parameter-agnostic based on the stopping time.
\end{itemize}

% \begin{algorithm}[H]
%     \caption{Adam (Norm Version)} \label{alg: adam}
%     \hspace*{0.02in} {\bf Input:}
% Stochastic oracle $\bO$, learning rate $\eta>0$, initial point $\bw_{1} \in \mathbb{R}^d$, initial conditioner $\bnu_{0
% }\in \mathbb{R}^{+}$, initial momentum $\bom_0$, momentum parameter $\beta_1$, conditioner parameter $\beta_2$, number of epoch $T$ 
%     \begin{algorithmic}[1]
%     \State Sample $r\sim\operatorname{Unif}\{1,\cdots,T\} $
%        \State \textbf{For} $t=1\rightarrow T$:
% \State ~~~~~Generate a random $z_t$, and query stochastic oracle $\bg_t =\bO_f(\bw_t,z_t)$
% \State ~~~~~Calculate $\bnu_{t}=\btwo\bnu_{t-1}+ (1-\btwo)\bg_t^{\odot 2}$
% \State ~~~~~Calculate $\bom_{t}=\bone\bom_{t-1}+ (1-\bone)\bg_t$
% \State ~~~~~Update $\bw_{t+1}= \bw_{t}-\eta \frac{1}{\sqrt{\bnu_{t}}}  \odot  \bom_t$
% \State \textbf{EndFor}
%     \end{algorithmic}
%     {\bf Output:} $\bw_r$
% \end{algorithm}

\section{Related Works}
\textbf{Convergence analysis under non-uniform smoothness.}  Observations from empirical studies on deep neural network training indicate that local smoothness can vary significantly throughout the optimization process. In response to this, \citet{zhang2019gradient} introduced the $(L_0,L_1)$-smooth condition, which posits that local smoothness can be bounded by a linear function of the gradient norm. Subsequent works have extended this concept by generalizing the linear function to polynomials \cite{chen2023generalized, li2023convergence}, or to more general functions \cite{mei2021leveraging}.
Under non-uniform smoothness, convergence properties of various optimizers have been studied. For instance, upper bounds on the convergence rate have been established for optimizers such as Clipped SGDM \cite{zhang2020improved}, sign-based optimizers \cite{jin2021non,hubler2023parameter,sun2023rethinking}, AdaGrad \cite{faw2023beyond,wang2023convergence}, variance-reduction methods \cite{reisizadeh2023variance,chen2023generalized}, and trust-region methods \cite{xie2023trust}. However, research on lower bounds has been comparatively limited, with results primarily focusing on Gradient Descent.

\textbf{Convergence analysis of Adam.} The development of convergence analysis for Adam has been quite tortuous. While Adam was originally proposed with a convergence guarantee \citep{kingma2014adam}, subsequent analysis by \citet{reddi2018amsgrad} pointed out flaws in this initial analysis and provided counterexamples claiming that Adam could fail to converge. Only recently, \citet{shirmsprop} and \citet{zhang2022adam} have shown that the counterexamples in \citet{reddi2018amsgrad} only rule out the possibility that Adam can converge problem-agnostically, and it is still possible that Adam can converge with problem-dependent hyperparameters.

So far, several works have established the convergence of Adam under the $L$-smooth condition. \citet{zaheer2018adaptive} proved that Adam without momentum can converge to the neighborhood of stationary points by additionally assuming that $\lambda$ is large. \citet{de2018convergence} showed that Adam without momentum can converge to stationary points but under the strong assumption that the sign of gradients does not change during the optimization. \citet{zou2019sufficient}, \citet{defossezsimple}, and \citet{guo2021novel} derived the convergence of Adam by assuming the stochastic gradient is bounded. \citet{shirmsprop} and \citet{zhang2022adam} characterized the convergence of random-reshuffling Adam but suffer from sub-optimal rates. \citet{he2023convergence} studied the non-ergodic convergence of Adam under a bounded gradient assumption, while \citet{hong2023high} provided high-probability guarantees for Adam under a deterministically bounded noise assumption. A concurrent work by \citet{wang2023closing} shows that Adam can achieve the lower bound of first-order optimizers with respect to the final error $\varepsilon$ under standard assumptions, but it is unknown whether Adam can match the lower bound with respect to other problem specifics.

On the other hand, closely related to our work, there are only two works studying the convergence of Adam under non-uniform smoothness \citep{wang2022provable,li2023convergence}, both with restricted assumptions and results. We will provide a detailed discussion in Section \ref{sec: separate}.

\textbf{Parameter-agnostic optimization.} The term "parameter-agnostic" implies that the optimizer is capable of converging without the need for extensive hyperparameter tuning or detailed knowledge of the task characteristics. Designing parameter-agnostic or parameter-free optimizers is a significant challenge, as it can help avoid the extensive cost associated with hyperparameter search. Existing works on parameter-agnostic optimization can be categorized into several streams based on the settings they are predicated upon.
In the deterministic offline setting, it is widely acknowledged that GD is not parameter-agnostic, even under an $L$-smooth condition \citep{nesterov2018lectures}. However, this can be rectified by combining the GD with the Backtracking Line Search technique \citep{armijo1966minimization}.
In the stochastic offline setting, under the $L$-smooth condition, multiple algorithms have been shown to be parameter-agnostic \citep{yang2023two, ward2020adagrad, faw2022power, wang2023convergence,cutkosky2020momentum}. More recently, \citet{hubler2023parameter} demonstrated that Normalized-SGDM can be parameter-agnostic even under an $(L_0, L_1)$-smooth condition.
In the realm of online convex optimization, \citet{orabona2016coin, orabona2017training} have shown  there exist parameter-free algorithms achieving optimal dependence regarding not only the final error but also  other problem specifics.

% In stochastic environments, however, line search techniques may not guarantee convergence, a situation explored by Vaswani et al. (2022). Even without parameter knowledge, SGD can attain a near-optimal complexity of $Oe(\epsilon^{-4})$ using a diminishing stepsize proportional to $1/\sqrt{t}$, albeit with an unavoidable exponential dependency on $L$, as demonstrated by Yang et al. (2023). Adaptive optimization methods such as AdaGrad (Duchi et al., 2011; McMahan and Streeter, 2010), along with its extensions AdaGrad-Norm (Streeter and McMahan, 2010) and NSGD-M (Cutkosky and Mehta, 2020), manage to circumvent this exponential dependence, as recently elucidated by Faw et al. (2022) and Yang et al. (2023). These methods are lauded for their robustness across varying problem parameters, due to their dynamic tuning of hyperparameters during training, as supported by Ward et al. (2019) and Kavis et al. (2019).

% Moreover, certain convex optimization algorithms have demonstrated the ability to achieve (near-)optimal convergence rates without specific knowledge of problem characteristics (Lan, 2015; Nesterov, 2015; Levy et al., 2018). A distinct area of research in online convex optimization has been focusing on "parameter-free" algorithms, with a particular emphasis on achieving optimal results based on the norm of the difference between the initial point and the optimal solution, $|x* - x_0|$, where $x*$ denotes the predictor in the regret bound (Orabona and Pal, 2016; Cutkosky and Orabona, 2018).

\section{Preliminary}
\textbf{Notations.} In this paper, we will use asymptotic notations $\mathcal{O}, \Omega, \Theta$ to respectively denote asymptotically smaller, larger , and equivalent. We also use $\tilde{\mathcal{O}}, \tilde{\Omega}, \tilde{\Theta}$ to indicate that there is logarithmic factor hidden. We denote $\gF_t$ as the filter given by $\bw_1,\cdots,\bw_t$.

\textbf{Problem and Algorithm.} We study the unconstrained minimization problem $\min_{\bw} f(\bw)$. We present the psedo-code of Adam as follows.
\begin{algorithm}[H]
   \caption{Adam Optimizer}
   \label{alg: adam}
\begin{algorithmic}
   \STATE {\bfseries Input:} Stochastic oracle $\bO$, learning rate $\eta>0$, initial point $\bw_{1} \in \mathbb{R}^d$, initial conditioner $\bnu_{0
}\in \mathbb{R}^{+}$, initial momentum $\bom_0$, momentum parameter $\beta_1$, conditioner parameter $\beta_2$, number of epoch $T$ 
   \FOR{$t=1$ {\bfseries to} $T$}
   \STATE Generate a random $z_t$, and query stochastic oracle $\bg_t =\bO_f(\bw_t,z_t)$
   \STATE Calculate $\bnu_{t}=\btwo\bnu_{t-1}+ (1-\btwo)\Vert \bg_t \Vert^2 $
   \STATE Calculate $\bom_{t}=\bone\bom_{t-1}+ (1-\bone)\bg_t$
   \STATE Update $\bw_{t+1}= \bw_{t}-\eta \frac{1}{\lambda+\sqrt{\bnu_{t}}}   \bom_t$
   \ENDFOR
\end{algorithmic}
\end{algorithm}
In our paper, we present a slightly altered version of the Adam optimizer as delineated in Algorithm \ref{alg: adam}, diverging from the canonical form described by \cite{kingma2014adam}. We assert that these modifications are implemented not to undermine the generality of the algorithm but to facilitate more streamlined proofs. Furthermore, our analysis retains applicability to the conventional Adam algorithm, largely following the same approach, albeit with a more elaborate proof process. Specifically, our first modification involves the omission of the bias-correction factors $1-\bone^t$ and $1-\btwo^t$ from the first and second momentum terms, respectively. It is important to note that incorporating bias correction would not alter the convergence rate, as these terms approach unity at an exponential rate, thus having a negligible impact on convergence.

Our second adjustment pertains to adopting a scalar-based adaptive learning rate, in contrast to the per-coordinate modification utilized in the typical Adam algorithm. Employing a scalar—or "norm"—version of adaptive optimizers is a recognized simplification strategy in the analysis of adaptive optimizers, as evidenced by literature such as \cite{xing2021convergence,faw2022power,faw2023beyond,wang2023convergence}. Our proof is readily adaptable to the per-coordinate version by entailing a separate analysis for each dimension \footnote{It is worth mentioning that the convergence rate for the per-coordinate Adam is subject to the dimensionality $d$. Addressing the challenge of decoupling the convergence rate of per-coordinate adaptive optimizers from dimensionality remains an unresolved issue, one that we acknowledge but reserve for future investigation.}. 

We would like to highlight that all the analysis in this paper is for $\lambda = 0$. This is because $\lambda=0$  means we do not require the adaptive learning rate to be upper bounded (a restrictive assumption in existing works \cite{li2023convergence, guo2021novel}) and is most challenging. The proof can be immediately extended to $\lambda >0 $ without any modification.

Meanwhile, we briefly state the SGDM optimizer as follows: with initial point $\bw_1$ and initial momentum $\bom_0$, the update of $t$-th iteration of SGDM is given by
\small
\begin{gather*}
    \bom_t =  \beta \bom_{t-1} +(1-\beta )\bg_t,
    \bw_{t+1} = \bw_t - \eta \bom_t.
\end{gather*}
\normalsize
\textbf{Assumptions.} In this paper, all the analyses are established under the following two standard assumptions. 
\begin{assumption}[$(L_0,L_1)$-smooth condition]
\label{assum: objective}
We assume $f$ is differentiable and lower bounded, and there exist non-negative constants $L_0,L_1>0$, such that $\forall \bw_1, \bw_2 \in \mathbb{R}^d$ satisfying $\Vert \bw_1 -\bw_2 \Vert \le \frac{1}{L_1}$,
\begin{equation*}
    \Vert \nabla f(\bw_1) -\nabla f(\bw_2) \Vert \le (L_0+L_1 \Vert \nabla f(\bw_1) \Vert)\Vert \bw_1 -\bw_2 \Vert.
\end{equation*}
\end{assumption}
\begin{assumption}[Affine noise variance]
\label{assum: noise}
We assume that the stochastic noise $\bg_t$ is unbiased, i.e.,  $\mathbb{E}^{|\mathcal{F}_t} \bg_t=\bG_t$. We further assume $\bg_t$ has affine variance, i.e., there exists $\sigma_0 \ge 0, \sigma_1 \ge 1$,   $\mathbb{E}^{|\gF_t} [\Vert\bg_t \Vert^2 ] \le \sigma_0^2 +\sigma_1^2 \Vert \nabla f(\bw_t) \Vert^2$.
\end{assumption}
Assumption \ref{assum: objective} is a more general form of $(L_0,L_1)$-smooth condition and is equivalent to the Hessian-bound form \cite{zhang2019gradient} when Hessian exists. Assumption \ref{assum: noise} is one of the weakest assumptions on the noise in existing literature, and generalizes bounded variance assumption \cite{li2023convex}, bounded gradient assumption \cite{defossezsimple}, bounded noise assumption \cite{li2023convergence}.

\section{Separating the convergence rates of Adam and (S)GD}
\label{sec: separate}
In this section, we elucidate the disparate convergence rates of Adam and (S)GD under Assumptions \ref{assum: objective} and \ref{assum: noise}, examining both deterministic and stochastic settings. We commence with the deterministic scenario before delving into the stochastic complexities.

\subsection{Analysis for the deterministic setting}
\label{sec: gd}
As discussed in the introduction section, to discern the differential convergence rates of deterministic Adam and GD, it is necessary to establish not only Adam's upper bound but also GD's lower bound, given a consistent set of assumptions. Crucially, these bounds must be sufficiently tight to ensure that Adam's upper bound is indeed the lesser. To date, only a couple of studies have addressed the convergence of deterministic Adam. The first, referenced in \cite{wang2022provable}, indicates a convergence rate of $ \mathcal{O}(\frac{(f(\bw_1)-f^*)^2}{\varepsilon^4})$, which is sub-optimal compared to the classical deterministic rate of $\mathcal{O}(\frac{f(\bw_1)-f^*}{\varepsilon^2})$  \cite{zhang2019gradient,zhang2020improved} regarding both final error $\varepsilon$  and the initial function value gap $(f(\bw_1)-f^*)$. The second study, \cite{li2023convergence}, presents a convergence rate that  depends polynomially on $\frac{1}{\lambda}$, where $\lambda$ is the small constant introduced to prevent the adaptive learning rate from becoming infinity. Therefore, their result is only non-vacuous when $\lambda$ is large, which deviates from practical settings.   Additionally, their bound exhibits an exaggerated dependency on the initial function value gap, yielding $\min_{t\in [T]} \Vert\nabla f (\bw_t) \Vert = \mathcal{O}(\frac{(f(\bw_1)-f^*)^3}{\varepsilon^2})$. As we will see later, 
such dependencies create upper bounds that surpass the lower bounds of GD, making them unable to serve our purpose. To overcome these limitations and accurately assess the performance of deterministic Adam, we propose a new theorem that establishes an improved convergence rate for deterministic Adam.

\textbf{An upper bound for the convergence rate of  deterministic Adam.} 

\begin{theorem}[Informal]
\label{thm: deterministic Adam}
    Let Assumption \ref{assum: objective} hold. Then, $\forall \beta_1, \beta_2 \ge 0$ satisfying $ \bone^2 < \btwo <1$, $\lambda =0$, and $\varepsilon = \mathcal{O}(L_0/L_1)$, if $T \ge \Theta\left(\frac{L_0 (f(\bw_1)-f^*)}{\varepsilon^2}\right)$, then Algorithm \ref{alg: adam} satisfies
    \begin{equation*}
        \frac{1}{T}\sum_{t=1}^T \Vert \nabla f(\bw_t) \Vert \le \varepsilon.
    \end{equation*}
\end{theorem}
\begin{proof}
Please see Appendix \ref{appen: deter_adam} for the formal statement of theorem and the proof.
\end{proof}
\vspace{-2mm}
Our result offers a tighter bound than those presented in prior studies \cite{wang2022provable,li2023convergence}. It is noteworthy that under the uniform smoothness constraint—where the objective function's smoothness is capped at $L$ (that is, when $L_0= L$ and $L_1 = 0$ as per Assumption 1, referred to as the $L$-smooth condition in existing literature \cite{arjevani2022lower,carmon2017lower,faw2022power})—Assumption 1 is met with $L_0 = L$ and any $L_1 \ge 0$. Consequently, the established lower bound for all first-order optimizers \cite{carmon2017lower}  pertaining to the $L$-smooth condition inherently provides a lower bound for the $(L_0, L_1)$-smooth condition, which is $\Omega\left(\frac{\sqrt{L_0 (f(\mathbf{w}_1)-f^*)}}{\sqrt{T}}\right)$. This coincides with our upper bound up to numerical constants. Such correspondence suggests that our proposed bound is, in fact, optimal.

Our proof strategy utilizes a distinctive Lyapunov function, $f(\bw_t) + \frac{\beta_1}{2(1-\beta_1)\sqrt[4]{\beta_2}}\eta \frac{||\bom_{t-1}||^2}{\lambda + \sqrt{\bnu_{t-1}}}$, which draws inspiration from the current analysis of Gradient Descent with Momentum (GDM) under the $L$-smooth condition \citep{sun2019non}. However, we have  introduced significant modifications to accommodate the integration of an adaptive learning rate. This carefully crafted Lyapunov function enables us to effectively control the deviation between the momentum term and the current gradient, even under $(L_0,L_1)$-smooth condition. Through this approach, we successfully establish the final optimal bound.

\begin{remark}[On the comparison with AdaGrad] Our result also suffices to separate Adam from AdaGrad. It is important to note that the convergence rate of AdaGrad under the $(L_0,L_1)$-smooth condition in a deterministic setting, as reported in\cite{wang2023convergence}, is $\frac{(f(\bw_1)-f^*)^2}{\varepsilon^2}$. This rate is outperformed by that of Adam\footnote{The state-of-art rate of AdaGrad under $(L_0,L_1)$-smooth condition and stochastic setting is $\frac{(f(\bw_1)-f^*)^2}{\varepsilon^4}$, which is also worse than the rate of Adam established latter in Theorem \ref{thm: stochastic  Adam}.}. In Appendix \ref{appen: adagrad}, we  show that the rate in \cite{wang2023convergence} is tight by providing a counterexample. The comparatively slower convergence rate of AdaGrad can be attributed to that $(L_0,L_1)$-smooth condition demands the update norm to be bounded by $\mathcal{O}(1)$ to prevent the local smoothness from exponentially increasing. This, in turn, necessitates a learning rate of $\mathcal{O}(1)$. However, the adaptive conditioner in AdaGrad, which accumulates over time, causes the adaptive learning rate to become excessively small during later training stages, resulting in reduced convergence speed. Conversely, Adam utilizes an exponential moving average for its adaptive learning rate, which prevents the conditioner from accumulating excessively. Consequently, Adam does not suffer from the aforementioned issue.
\end{remark}

\textbf{A lower bound for the convergence rate of GDM}

With Adam's upper bound, we then move on to a lower bound for the convergence rate of GDM. In fact, there has already been such lower bounds for GD in the existing literature \cite{zhang2019gradient,crawshaw2022robustness}, which we restate as follows: 
\begin{proposition}[Theorem 2, \cite{crawshaw2022robustness}]
\label{prop: crawshaw}
Fix $\varepsilon, L_0, L_1 $, and $\Delta_1$, {\color{blue}with learning rate $\eta$}, there exists objective function $f$ satisfying $(L_0,L_1)$-smooth condition and $f(\bw_1)-f^* =\Delta_1$, such that the minimum step $T$ of GD to achieve final error $\varepsilon$ (i.e., let $\{\bw_t\}_{t=1}^\infty$ be the iterates of GD, and  $T \triangleq \min \{t: \Vert \nabla f(\bw_t) \Vert < \varepsilon \}$) satisfies 
\begin{equation*}
    T = \tilde{\Omega}\left( \frac{L_1^2 \Delta_1^2 +L_0 \Delta_1}{\varepsilon^2} \right).
\end{equation*}
\end{proposition}
\vspace{-2mm}
However,  the proposition presents a limitation: the counter-example is chosen after the learning rate has been determined. This approach is inconsistent with standard practices, where hyperparameters are usually adjusted based on the specific task, and deviates from conventional lower bounds \cite{carmon2017lower,arjevani2022lower} that offer assurances for optimally-tuned hyperparameters. This type of result  does not eliminate the possibility that, if the learning rate were adjusted after selecting the objective function—as is common practice—Gradient Descent (GD) could potentially achieve a markedly faster convergence rate. This misalignment raises concerns about the appropriateness of the proposition’s methodology.  Moreover, this proposition does not take momentum into account, a technique that is commonly employed in conjunction with GD in practice.

To address these shortcomings, we introduce a new lower bound for GDM. This lower bound is applicable under the standard practice of adjusting hyperparameters after the objective function has been selected. Moreover, it encompasses scenarios where momentum is incorporated.
\begin{theorem}[Informal]
    \label{thm: gd}
    Fixing $\varepsilon, L_0, L_1 $, and $\Delta_1$,  there exists an objective function $f$ satisfying $(L_0,L_1)$-smooth condition and $f(\bw_1)-f^* =\Delta_1$, such that {\color{blue}for any learning rate $\eta >0$ and $\beta \in [0, 1]$}, the minimum step $T$ of GDM to achieve final error $\varepsilon$ satisfies 
\begin{equation*}
    T = \tilde{\Omega}\left( \frac{L_1^2 \Delta_1^2+L_0\Delta_1 }{\varepsilon^2} \right).
\end{equation*}
\end{theorem}
\begin{proof}
Please see Appendix \ref{appen: gd} for the formal statement of theorem and the proof.
\end{proof}
\vspace{-2mm}
It should be noticed in the above theorem, the hyperparameters (i.e., the learning rate and the momentum coefficient) are chosen after the objective function  is determined, which agrees with practice and the settings of common lower bounds, and overcomes the shortcoming of Proposition \ref{prop: crawshaw}. Moreover, as shown in \citet{zhang2019gradient}, it is easy to prove that the upper bound of GD's convergence rate is also $\mathcal{O}\left( \frac{L_1^2 \Delta_1^2+L_0\Delta_1 }{\varepsilon^2} \right)$, which indicates such a lower bound is optimal.

The proof addresses two primary challenges outlined above. The first challenge involves handling momentum. To tackle this, we extend the counterexample provided in Proposition \ref{prop: crawshaw} for cases where the momentum coefficient $\beta$ is small. Additionally, we introduce a new counterexample for situations with a large  $\beta$, demonstrating how large momentum can bias the optimization process and decelerate convergence. The second challenge is how to derive a universal counterexample such that every hyperparameter setting will lead to slow convergence. We overcome this by a simple but effective trick: we independently put counterexamples for different hyperparameters in Proposition \ref{prop: crawshaw} over different coordinates and make it a whole counterexample. Therefore, for different hyperparameters, there will be at least one coordinate converge slowly, which leads to the final result.

\textbf{Separating deterministic Adam and GDM.}  Upon careful examination of Theorem \ref{thm: deterministic Adam} and Theorem \ref{thm: gd}, it becomes apparent that the convergence rate of GDM is inferior to that of Adam since $\frac{\sum_{t=1}^T \Vert \bG_t \Vert }{T} \ge \min_{t\in [T]} \Vert \bG_t \Vert$. Notably, GDM exhibits a more pronounced dependency on the initial function value gap in comparison to Adam. This implies that with a sufficiently poor initial point, the convergence of GDM can be significantly slower than that of Adam. The underlying reason for this disparity can be attributed to GDM's inability to adeptly manage varying degrees of sharpness within the optimization landscape. Consequently, GDM necessitates a learning rate selection that is conservative, tailored to the most adverse sharpness encountered—often present during the initial optimization stages.

\subsection{Analysis for the stochastic setting}
\label{sec: separability_result_for_stochastic}
Transitioning to the more complex stochastic setting, we extend our analysis beyond the deterministic framework. As with our previous approach, we start by reviewing the literature to determine if the existing convergence rates for Adam under the $(L_0, L_1)$-smooth condition can delineate a clear distinction between the convergence behaviors of Adam and Stochastic Gradient Descent with Momentum (SGDM). In fact, the only two studies that delve into this problem are the ones we discussed in Section \ref{sec: gd}, i.e., \cite{wang2022provable,li2023convergence}.  
However, these results pertaining to Adam  are contingent upon rather stringent assumptions. \citet{wang2022provable} postulates that stochastic gradients not only conform to the $(L_0, L_1)$-smooth condition but are also limited to a finite set of possibilities. These assumptions are more restrictive than merely assuming that the true gradients satisfy the $(L_0, L_1)$-smooth condition, and such strong prerequisites are seldom employed outside of the analysis of variance-reduction algorithms. Meanwhile, \citet{li2023convergence} aligns its findings on stochastic Adam with those on deterministic Adam, leading to a polynomial dependency on $1/\lambda$, which deviates from practical scenarios as discussed in Section \ref{sec: gd}. Furthermore, it presumes an a.s. bounded difference between stochastic gradients and true gradients, an assumption that closely resembles the boundedness of stochastic gradients and is more limiting than the standard assumption of bounded variance for stochastic gradients.  
   
These more restricted and non-standard assumptions cast challenges in establishing a lower bound for the convergence of SGDM in the relevant contexts, let alone attempting a comparison between SGDM and Adam. In addition to the fact that these upper bounds fail to facilitate a clear comparison between Adam and SGDM, there are also concerns regarding their convergence rates. \citet{wang2022provable} reports a convergence rate of $\frac{(f(\bw_1)-f^*)^2}{\varepsilon^8}$, which has a higher-order dependence on the initial function value gap and the final error than the $\frac{(f(\bw_1)-f^*)}{\varepsilon^4}$ rate established for Clipped SGDM under the $(L_0, L_1)$-smooth condition \cite{zhang2020improved}\footnote{While \citet{zhang2020improved} also assumes an a.s. bounded gap between stochastic gradients and true gradients.}.
Furthermore, \citet{li2023convergence} indicates a convergence rate of $\mathcal{O}(\frac{(f(\bw_1)-f^*)^4\operatorname{poly}(1/\lambda)}{\varepsilon^4 })$, which, aside from the previously mentioned dependency issues on $1/\lambda$, shows a significantly stronger dependence over the initial function value gap compared to the analysis of Clipped SGDM. This naturally leads to the question of whether such  rates for Adam can be improved to match Clipped SGDM.

To tackle these obstacles, we present the following upper bound for Adam.

\textbf{An upper bound for the convergence rate of Adam.}
\begin{theorem}[Informal]
\label{thm: stochastic  Adam}
    Let Assumptions \ref{assum: objective} and \ref{assum: noise} hold. Then, $\forall 1>\beta_1 \ge 0$ and $\lambda =0$, if $\varepsilon\le \frac{1}{\operatorname{poly}(f(\bw_1)-f^*, L_0,L_1, \sigma_0,\sigma_1)}$, with a proper choice of learning rate $\eta$ and momentum hyperparameter $\btwo$, we have if $T \ge \Theta\left( \frac{(L_0+L_1)\sigma_0^3\sigma_1^2 (f(\bw_1)-f^*) }{\varepsilon^4}\right)$,
    \small
    \begin{equation*}
        \frac{1}{T}\E\sum_{t=1}^T \Vert \nabla f(\bw_t) \Vert \le \varepsilon.
    \end{equation*}
    \normalsize
\end{theorem}
\begin{proof}
    Please see Appendix \ref{appen: adam} for the formal statement of
theorem and the proof.
\end{proof}
\vspace{-2mm}
Below we include several discussions regarding Theorem \ref{thm: stochastic  Adam}. To begin with, one can immediately observe that Theorem \ref{thm: stochastic  Adam} only requires Assumptions \ref{assum: objective} and \ref{assum: noise}, and the convergence rate with respect to the initial function value gap and the final error $\frac{f(\bw_1)-f^*}{\varepsilon^4}$ matches that of Clipped SGDM \cite{zhang2020improved} even with a weaker noise assumption. Therefore, our result successfully mitigate these barriers raised above. Indeed, to the best of our knowledge, it is for the first time that an algorithm is shown to converge with rate $\mathcal{O}\left(\frac{f(\bw_1)-f^*}{\varepsilon^4}\right)$ only requiring Assumptions \ref{assum: objective} and \ref{assum: noise}, showcasing the advantage of Adam.

We briefly sketch the proof here before moving on to the result of SGDM. Specifically, the proof is inspired by recent analysis of Adam under $L$-smooth condition \cite{wang2023closing}, but several challenges arise during the proof:
\vspace{-2mm}
\begin{itemize}
    \item The first challenge lies in the additional error introduced by the $(L_0,L_1)$-smooth condition. We address this by demonstrating that the telescoping sum involving the auxiliary function $\frac{\Vert \bG_t \Vert^2}{\sqrt{\bnu_{t-1}}}$, as employed in \cite{wang2023closing}, can bound this additional error when the adaptive learning rate is upper bounded. Although the adaptive learning rate in the Adam algorithm is not inherently bounded, we establish that the deviation incurred by employing a bounded surrogate adaptive learning rate is manageable;
    \item The second challenge involves deriving the desired dependence on the initial function value gap. \citet{wang2023closing} introduces two distinct proof strategies for bounding the conditioner $\bnu_t$ and determining the final convergence rate. However, one strategy introduces an additional logarithmic dependence on $\varepsilon$, while the other exhibits sub-optimal dependence on the initial function value gap. We propose a novel two-stage divide-and-conquer approach to surmount this issue. In the first stage, we bound $\bnu_t$ effectively. Subsequently, we leverage this bound within the original descent lemma to achieve the optimal dependence on $f(\bw_1)-f^*$.
\end{itemize}

\begin{remark}[On the limitations]
\label{remark: limitation}
Although Theorem \ref{thm: stochastic Adam} addresses certain deficiencies identified in prior studies \cite{wang2022provable,li2023convergence}, it is not without its limitations. As noted by \citet{arjevani2022lower}, the established lower bound for the convergence rate of first-order optimization algorithms under the $L_0$-smooth condition with bounded noise variance (specifically, $\sigma_0 = \sigma_0$ and $\sigma_1 = 1$ as stated in Assumption \ref{assum: noise}) is $\mathcal{O}(\frac{(f(\bw_1)-f^*)L_0 \sigma_0^2}{\varepsilon^4})$. This sets a benchmark for the performance under Assumptions \ref{assum: objective} and \ref{assum: noise}. The upper bound of Adam's convergence rate as presented in Theorem \ref{thm: stochastic Adam} falls short when compared to this benchmark, exhibiting a weaker noise scale dependency ($\sigma_0^3$ as opposed to $\sigma_0^2$) and additional dependencies on $L_1$ and $\sigma_1$. 

% Moreover, the convergence rate implicitly depends on $\frac{1}{1-\beta_1}$, as indicated by the big-O notation, suggesting that the introduction of momentum is detrimental to convergence—contradicting empirical observations. 

To address these issues, we demonstrate in the subsequent section that by focusing on the convergence of the minimum gradient norm, $\E \min_{t\in [T]} \Vert \nabla f(\bw_t) \Vert$, we can attain an improved convergence rate of $\mathcal{O}(\frac{(f(\bw_1)-f^*)L_0 \sigma_0^2}{\varepsilon^4})$. This rate aligns with the aforementioned lower bound across all the problem hyperparameters. 
\end{remark}

We now establish the lower bound of SGDM. This is, however, more challenging than the deterministic case as to the best of our knowledge, there is no such a lower bound in existing literature (despite that the lower bounds of GD \citep{zhang2019gradient,crawshaw2022robustness} naturally offer a lower bound of SGD, which is considerably loose given the factor of $1/\varepsilon^2$). Intuitively, stochasticity can make the convergence of GDM even worse, as random fluctuations can inadvertently propel the iterations towards regions characterized by high smoothness even with a good initialization. We formulate this insight into the following theorem.

\textbf{A lower bound for the convergence rate of SGDM.}
\begin{theorem}[Informal]
    \label{thm: sgd}
    Fix $ L_0, L_1 $, and $\Delta_1$,  there exists objective function $f$ satisfying $(L_0,L_1)$-smooth condition and $f(\bw_1)-f^* =\Delta_1$, and a gradient noise oracle satisfying Assumption \ref{assum: noise}, such that {\color{blue}for any learning rate $\eta >0$ and $\beta \in [0, 1]$}, for all $T >0$, 
    \begin{equation*}
        \min_{t\in [T]} \E \Vert \nabla f(\bw_t) \Vert = \Vert \nabla f(\bw_1) \Vert \ge L_1\Delta_1.
    \end{equation*}
\end{theorem}
\begin{proof}
Please see Appendix \ref{appen: sgd} for the formal statement of theorem and the proof.
\end{proof}
\vspace{-2mm}
Theorem \ref{thm: sgd} provides concrete evidence for the challenges inherent in the convergence of SGDM. It shows that there are instances that comply with Assumption \ref{assum: objective} and Assumption \ref{assum: noise} for which SGDM fails to converge, regardless of the chosen learning rate and momentum coefficient. This outcome confirms our earlier hypothesis: the stochastic elements within SGDM can indeed adversely affect its convergence properties under non-uniform smoothness.

Our proof is founded upon a pivotal observation: an objective function that escalates rapidly can effectively convert non-heavy-tailed noise into a "heavy-tailed" one. In particular, under the $(L_0,L_1)$-smooth condition, the magnitude of the gradient is capable of exponential growth. As a result, even if the density diminishes exponentially, the expected value of the gradient norm may still become unbounded. This situation mirrors what occurs under the $L$-smooth condition when faced with heavy-tailed noise. Such a dynamic can lead to the non-convergence of SGDM.

\textbf{Separating Adam and SGDM.}  Considering that Adam can achieve convergence under Assumptions \ref{assum: objective} and \ref{assum: noise}, while SGD cannot, the superiority of Adam over SGDM becomes evident. It is important to note, however, a recent study by \cite{li2023convex}, which demonstrates that SGD can converge with high probability under the same assumptions, provided the noise variance is bounded. We would like to contextualize this finding in relation to our work as follows: First, this result does not conflict with our Theorem \ref{thm: sgd}, since our theorem pertains to bounds in expectation rather than with high probability. Second, our comparison of Adam and SGDM within an in-expectation framework is reasonable and aligns with the convention of most existing lower bounds in the literature \cite{carmon2017lower, drori2020complexity, arjevani2022lower}. Moreover, establishing high-probability lower bounds is technically challenging, and there are few references to such bounds in the existing literature. Lastly, while we have not derived a corresponding high-probability lower bound for SGD, the upper bound provided by \citet{li2023convex} is $\mathcal{O}(\frac{(f(\bw_1)-f^*)^4}{\varepsilon^4})$, which indicates a less favorable dependency on the initial function value gap compared to the bound for Adam.

\section{Can Adam reach the lower bound of the convergence rate under $(L_0,L_1)$-smooth condition?}
\label{sec: reach_lower_bound}
As we mentioned in Remark \ref{remark: limitation}, although Theorem \ref{thm: stochastic Adam} matches the lower bound established by \citet{arjevani2022lower} with respect to the initial function value gap $f(\bw_1)-f^*$, the final error $\varepsilon$, and the smoothness coefficient $L_0$, it exhibits sub-optimal dependence on the noise scale $\sigma_0$ and additional dependence on $L_1$ and $\sigma_1$. One may wonder whether these dependencies are inherently unavoidable or if they stem from technical limitations in our analysis.

Upon revisiting the proof, we identified that the sub-optimal dependencies arise from our strategy of substituting the original adaptive learning rate with a bounded surrogate. For example, the correlation between stochastic gradient and adaptive learning rate will introduce an error term $\eta\frac{\sigma_0^2 (1-\btwo)\Vert \bg_t \Vert^2}{\sqrt{\btwo \bnu_{t-1}}\bnu_t}$, detailed in Eq. (\ref{eq: approximation_error}). To bound this term, we add a constant $\lambda$ to $\btwo \bnu_{t-1}$, allowing us to upper bound $\frac{1}{\sqrt{\btwo \bnu_{t-1}+\lambda}}$. Consequently, the term $\eta\frac{\sigma_0^2 (1-\btwo)\Vert \bg_t \Vert2}{\sqrt{\btwo \bnu_{t-1}+\lambda}\bnu_t}$ can be bounded by $\eta\frac{\sigma_0^2 (1-\btwo)\Vert \bg_t \Vert^2}{\sqrt{\lambda}\bnu_t}$, which has the same order as a second-order Taylor expansion. To control the error introduced by adding $\lambda$, we cannot choose a value for $\lambda$ that is too large. The optimal choice of $\lambda$ for balancing the new error against the original error is $(1-\btwo)\sigma_0^2$. This selection results in the original error term $\eta\frac{\sigma_0 \sqrt{1-\btwo}\Vert \bg_t \Vert^2}{\bnu_t}$, which induces an additional $\sigma_0$ factor, ultimately leading to the sub-optimal dependence on $\sigma_0$. Therefore, we need to explore alternative methods to handle the error term to eliminate the sub-optimal dependence on $\sigma_0$.

We begin our analysis by observing that the term $\frac{(1-\btwo)\Vert \bg_t \Vert^2}{\sqrt{\btwo \bnu_{t-1}}\bnu_t}$ can in fact be bounded by an "approximate telescoping" series of $\frac{1}{\sqrt{\bnu_t}}$ (noting an additional coefficient $\frac{1}{\sqrt{\btwo}}$ in comparison to standard telescoping):

\small
\begin{equation*}
\frac{(1-\btwo)\Vert \bg_t \Vert^2}{\sqrt{\btwo \bnu_{t-1}}\bnu_t} \le \mathcal{O}\left(\frac{1}{ \sqrt{\btwo\bnu_{t-1}}} - \frac{1}{ \sqrt{\bnu_{t}}}\right).
\end{equation*}
\normalsize

Accordingly, summing $\eta\frac{\sigma_0^2 (1-\btwo)\Vert \bg_t \Vert2}{\sqrt{\btwo \bnu_{t-1}}\bnu_t}$ over $t$ yields a bound of $\mathcal{O}(\eta\sigma_0^2 \sum_t (1-\btwo) \frac{1}{ \sqrt{\bnu_{t}}} )$. However, this term could potentially be unbounded since $\sqrt{\bnu_{t}}$ is not lower bounded. To circumvent this issue, we consider the first-order Taylor's expansion of the descent lemma, which, gives $-\sum_t \eta \frac{\Vert \nabla f(\bw_t) \Vert2 }{\sqrt{\bnu_t}}$. Intuitively, if any $\Vert \nabla f(\bw_t) \Vert^2$ is of the order $\mathcal{O}(\sigma_0^2 (1-\btwo))$, our proof would be completed since we choose $1-\btwo= \Theta (\varepsilon^4)$. In the other case, the term $\mathcal{O}(\eta\sigma_0^2 \sum_t (1-\btwo) \frac{1}{ \sqrt{\bnu_{t}}} )$ can be offset by the negative term $-\sum_t \eta \frac{\Vert \nabla f(\bw_t) \Vert^2 }{\sqrt{\bnu_t}}$. However, formalizing this intuition into a proof is challenging in the context of stochastic analysis, where the randomness across iterations complicates the analysis. Specifically, if we condition on the event that "no gradient norm is as small as $\sigma_0^2 (1-\btwo)$," which is supported over the randomness of all iterations, it becomes difficult to express many expected values (such as those from the first-order Taylor expansion) in closed form.

We address this difficulty by introducing a stopping time $\tau\triangleq \min\{t: \Vert \nabla f(\bw_{t+1}) \Vert^2 \le \mathcal{O}(\sigma_0^2 (1-\btwo))\}$. By applying the optimal stopping theorem \cite{durrett2019probability}, we can maintain closed-form expressions for the expected values up to the stopping time, allowing the problematic error term to be absorbed within this interval. Building on this methodology, we formulate the following theorem.

\begin{theorem}[Informal]
\label{thm: attain_lower_bound}
    Let Assumptions \ref{assum: objective} and \ref{assum: noise} hold. Then, $\forall 1>\beta_1 \ge 0$, if $\varepsilon\le \frac{1}{\Poly}$, with a proper choice of learning rate $\eta$ and momentum hyperparameter $\btwo$, we have that if $T \ge \Theta(\frac{ L_0\sigma_0^2 (f(\bw_1)-f^*)}{\varepsilon^4})$
    \begin{equation*}
        \E \min_{t\in [1,T]}\Vert \nabla f(\bw_t) \Vert \le \varepsilon.
    \end{equation*}
\end{theorem}
\begin{proof}
Please see Appendix \ref{appen: tight_bound} for the formal statement of theorem and the proof.
\end{proof}
\vspace{-2mm}
One can easily see that the convergence rate of Theorem \ref{thm: attain_lower_bound} matches the lower bound in \citet{arjevani2022lower} with respect to all problem hyperparameters up to numerical constants even under the weaker $(L_0,L_1)$-smooth condition. Therefore, such a rate is optimal and provides an affirmative answer to the question raised in the beginning of this section.

One may notice that in the construction of the stopping time, we set the threshold for the squared gradient norm to be $\mathcal{O}(1-\beta_2)$. As we set $1-\beta_2 =\Theta(\varepsilon^4)$, the threshold is actually much smaller than what we aim for, since our goal is to have $\Vert \nabla f(\mathbf{w}_t) \Vert^2 \le \varepsilon^2$. Therefore, based on the stopping-time technique, we can actually show that Adam can converge with an optimal rate of $\mathcal{O}(\varepsilon^{-4})$ when $1-\beta_2 = \varepsilon^2$, or $1/\sqrt{T}$ if expressed in terms of the iteration number $T$. To the best of our knowledge, this is the first time that Adam has been shown to converge with an optimal rate under the condition that $1-\beta_2 = \Omega(1/T)$, which greatly enlarges the hyperparameter range. Moreover, as we select $\eta = 1/\sqrt{T}$, choosing $1-\beta_2 = \Omega(1/T)$ has the advantage that the update norm decreases with respect to $T$. This makes Adam parameter-agnostic under the $(L_0,L_1)$-smooth condition, as the update norm will eventually become smaller than $\frac{1}{L_1}$ as $T$ increases.
\begin{theorem}
\label{thm: parameter_agnostic}
     Let Assumptions \ref{assum: objective} and \ref{assum: noise} hold. Then, at the $t$-th iteration, setting $\eta= \frac{1}{\sqrt{t}}$, $\btwo = 1- \frac{1}{\sqrt[4]{t^3}}$, we have that Algorithm \ref{alg: adam} satisfies
    \begin{equation*}
        \E \min_{t\in [1,T]}\Vert \nabla f(\bw_t) \Vert \le \tilde{\mathcal{O}} \left(\frac{1}{\sqrt[4]{T}}\right).
    \end{equation*}
\end{theorem}
It is shown in \cite{hubler2023parameter} that Normed-SGDM is parameter-agnostic. Here we show that Adam with a specific scheduler can achieve the same goal.

\section{Conclusion}
In this paper, we have conducted a mathematical examination of the performance of the Adam optimizer and SGDM within the context of non-uniform smoothness. Our convergence analysis reveals that Adam exhibits a faster rate of convergence compared to SGDM under these conditions. Moreover, we introduce a novel stopping time technique that demonstrates Adam's capability to achieve the existing lower bounds for convergence rates. This finding underscores the robustness of Adam in complex optimization landscapes and contributes to a deeper understanding of its theoretical properties.

\section*{Impact Statement}
	This paper investigates convergence of Adam and SGDM under non-uniform smoothness. The main contributions of this paper are theoretical. Thus, in our opinion, the paper has no potential ethical and societal problems.

\nocite{langley00}

\bibliography{related}
\bibliographystyle{icml2024}

%%%%%%%%%%%%%%%%%%%%%%%%%%%%%%%%%%%%%%%%%%%%%%%%%%%%%%%%%%%%%%%%%%%%%%%%%%%%%%%
%%%%%%%%%%%%%%%%%%%%%%%%%%%%%%%%%%%%%%%%%%%%%%%%%%%%%%%%%%%%%%%%%%%%%%%%%%%%%%%
% APPENDIX
%%%%%%%%%%%%%%%%%%%%%%%%%%%%%%%%%%%%%%%%%%%%%%%%%%%%%%%%%%%%%%%%%%%%%%%%%%%%%%%
%%%%%%%%%%%%%%%%%%%%%%%%%%%%%%%%%%%%%%%%%%%%%%%%%%%%%%%%%%%%%%%%%%%%%%%%%%%%%%%
\newpage
\appendix
\onecolumn

\section{Auxiliary Lemmas}
\label{appen: auxiliary}
In this section, we provide auxiliary results which will be used in subsequent results.
\begin{lemma}
\label{lem: bounded_update}
    We have $\forall t\ge 1$, $\Vert  \bw_{t+1}-\bw_{t}\Vert \le \eta \frac{1-\bone}{\sqrt{1-\btwo}\sqrt{1-\frac{\bone^2}{\btwo}}}$.
\end{lemma}
\begin{proof}
    We have that 
    \begin{align*}
       & \Vert  \bw_{t+1}-\bw_{t}\Vert =\eta \left \vert \frac{\bom_{t}}{\sqrt{\bnu_{t}}} \right\vert \le\eta\frac{\sum_{i=0}^{t-1} (1-\bone) \bone^{i} \Vert \bg_{t-i}\Vert }{\sqrt{\sum_{i=0}^{t-1} (1-\btwo) \btwo^{i} \Vert \bg_{t-i}\Vert^2+\btwo^t \bnu_{0}}}
        \\
        \le &\eta\frac{1-\bone}{\sqrt{1-\btwo}}\frac{\sqrt{\sum_{i=0}^{t-1} \btwo^{i} \Vert \bg_{t-i}\Vert^2}\sqrt{\sum_{i=0}^{t-1} \frac{\bone^{2i}}{\btwo^{i}} } }{\sqrt{\sum_{i=0}^{t-1} \btwo^{i} \Vert \bg_{t-i}\Vert^2}}\le \eta\frac{1-\bone}{\sqrt{1-\btwo}\sqrt{1-\frac{\bone^2}{\btwo}}}.
    \end{align*}
    Here the second inequality is due to Cauchy's inequality. The proof is completed.
\end{proof}

The following lemma provides a novel descent lemma under $(L_0,L_1)$-smooth condition.
\begin{lemma}
\label{lem: descent}
Let Assumption \ref{assum: objective} hold.
Then, for any three points $\bw^1, \bw^2,\bw^3\in \mathcal{X}$  satisfying $\Vert \bw^1-\bw^2\Vert \le \frac{1}{2 L_1}$ and $\Vert \bw^1-\bw^3\Vert \le \frac{1}{2 L_1}$, we have
\begin{equation*}
    f(\bw^2)\le  f(\bw^3)+\langle \nabla f(\bw^3), \bw^2-\bw^3\rangle + \frac{1}{2}(L_0+L_1 \Vert \nabla f(\bw^1)\Vert) \Vert\bw^2-\bw^3\Vert (\Vert \bw^1-\bw^3 \Vert + \Vert \bw^1-\bw^2 \Vert )
    . 
\end{equation*}
\end{lemma}
\begin{proof}
By the Fundamental Theorem of Calculus, we have
\begin{align*}
   f(\bw^2)=& f(\bw^3)+\int_{0}^1 \langle \nabla f(\bw^3+a(\bw^2-\bw^3)), \bw^2-\bw^3\rangle
    \mathrm{d}a
    \\
    =& f(\bw^3)+\langle \nabla f(\bw^1), \bw^2-\bw^3\rangle +\int_{0}^1 \langle \nabla f(\bw^3+a(\bw^2-\bw^3))-\nabla f(\bw^1), \bw^2-\bw^3\rangle
    \mathrm{d}a
    \\
    \le & f(\bw^3)+\langle \nabla f(\bw^1), \bw^2-\bw^3\rangle +\int_{0}^1 \Vert \nabla f(\bw^3+a(\bw^2-\bw^3))-\nabla f(\bw^1)\Vert \Vert\bw^2-\bw^3\Vert
    \mathrm{d}a
    \\
    \overset{(\star)}{\le} &  f(\bw^3)+\langle \nabla f(\bw^1), \bw^2-\bw^3\rangle +\int_{0}^1 (L_0+L_1 \Vert \nabla f(\bw^1)\Vert )\Vert a(\bw^2-\bw^1)+(1-a) (\bw^3-\bw^1)\Vert \Vert\bw^2-\bw^3\Vert
    \mathrm{d}a
    \\
    \le & f(\bw^3)+\langle \nabla f(\bw^1), \bw^2-\bw^3\rangle + \frac{1}{2}(L_0+L_1 \Vert \nabla f(\bw^2)\Vert) \Vert\bw^2-\bw^3\Vert (\Vert \bw^1-\bw^3 \Vert + \Vert \bw^1-\bw^2 \Vert )
   ,
\end{align*}
where Inequality $(\star)$ is because due to 
\begin{equation*}
    \Vert \bw^3 +a (\bw^2-\bw^3) -\bw^1\Vert =\Vert a (\bw^2-\bw^1)+(1-a )(\bw^3-\bw^1)\Vert\le \frac{1}{L_1},
\end{equation*}
the definition of $(L_0, L_1)$-smooth condition can be applied.

The proof is completed.
\end{proof}

The following lemma is helpful when bounding the second-order term.

\begin{lemma}
\label{lem: sum_momentum}
Assume we have $0<\beta_1^2<\beta_2< 1$ and a sequence of real numbers $(a_n)_{n=1}^{\infty}$. Let $b_0>0$, $b_n= \btwo b_{n-1}+(1-\btwo) a_n^2$, $c_0=0$, and $c_n= \bone c_{n-1}+(1-\bone) a_n$. Then, we have
\begin{equation*}
     \sum_{n=1}^T \frac{\vert c_n \vert^2}{b_n}\le \frac{(1-\bone)^2}{(1-\frac{\bone}{\sqrt{\btwo}})^2(1-\btwo)} \left(\ln \left(\frac{b_T}{b_0}\right) - T \ln \btwo\right).
\end{equation*}
\end{lemma}

\begin{lemma} 
\label{lem: momentum_sum_1}

If $\btwo\ge \bone$, then
    we have
    \begin{equation*}
        \frac{\Vert \bom_t \Vert^2}{ (\sqrt{\bnu_t})^3} \le 4 (1-\bone) \left( \sum_{s=1}^t\sqrt[4]{\bone^{t-s}} \frac{2}{1-\btwo} \left(\frac{1}{\sqrt{\btwo \bnu_{s-1}}}- \frac{1}{\sqrt{ \bnu_{s}}}\right) \right).
    \end{equation*}
\end{lemma}
\begin{proof}
    To begin with, we have 
    \begin{align*}
        \frac{\Vert \bom_t \Vert }{\sqrt[4]{\bnu_t^3}} \le (1-\bone) \sum_{s=1}^t\frac{ \btwo^{t-s} \Vert \bg_s\Vert}{\sqrt[4]{\bnu_t^3}} \le (1-\bone) \sum_{s=1}^t\frac{ \bone^{t-s} \Vert \bg_s\Vert}{\sqrt[4]{\btwo^{3(t-s)}}\sqrt[4]{\bnu_s^3}}.
    \end{align*}
    Here in the last inequality we use $\bnu_t \ge \btwo^{t-s} \bnu_s$.

    By further applying Cauchy-Schwartz inequality, we obtain
    \begin{align*}
        \frac{\Vert \bom_t \Vert^2 }{\sqrt{\bnu_t^3}} \le&  (1-\bone)^2 \left( \sum_{s=1}^t\frac{ \bone^{t-s} \Vert \bg_s\Vert^2}{\sqrt[4]{\btwo^{3(t-s)}}\sqrt{\bnu_s^3}} \right) \left(\sum_{s=1}^t\frac{ \bone^{t-s} }{\sqrt[4]{\btwo^{3(t-s)}}}\right)
        \\
        \le & \frac{ (1-\bone)^2}{1-\frac{\bone}{\sqrt[4]{\btwo^3}}} \left( \sum_{s=1}^t\frac{ \bone^{t-s} \Vert \bg_s\Vert^2}{\sqrt[4]{\btwo^{3(t-s)}}\sqrt{\bnu_s^3}} \right)
        \\
        \le &4 (1-\bone) \left( \sum_{s=1}^t\frac{ \bone^{t-s} \Vert \bg_s\Vert^2}{\sqrt[4]{\btwo^{3(t-s)}}\sqrt{\bnu_s^3}} \right).
    \end{align*}
    As
$
    \frac{\Vert \bg_s \Vert^2}{\sqrt{\bnu_s^3}} \le \frac{2 \Vert \bg_s \Vert^2}{\sqrt{\bnu_s}\sqrt{\btwo\bnu_{s-1}}(\sqrt{\bnu_s}+\sqrt{\btwo\bnu_{s-1}})} =\frac{2}{1-\btwo} \left(\frac{1}{\sqrt{\btwo \bnu_{s-1}}}- \frac{1}{\sqrt{ \bnu_{s}}}\right)$,
the proof is completed.
\end{proof}

\begin{lemma} 
\label{lem: momentum_sum_2}

If $\btwo\ge \bone$, then
    we have
    \begin{equation*}
        \frac{\Vert \bom_t \Vert^2 \Vert \bG_t \Vert^2}{ \bnu_t \sqrt{\btwo \bnu_{t-1}}} \le 4 (1-\bone) \left( \sum_{s=1}^t\frac{ \sqrt[8]{\bone^{t-s}} \Vert \bg_s\Vert^2\Vert \bG_s \Vert^2}{\bnu_s\sqrt{\btwo\bnu_{s-1}}}\right) + 8 \frac{1-\bone}{1-\btwo} \frac{L_1^2}{L_0^2}\left(\sum_{s=1}^t\sqrt[8]{\bone^{t-s}}\left(\frac{1}{\sqrt{\btwo \bnu_{s-1}}} -\frac{1}{\sqrt{\bnu_s}}\right) \right).
    \end{equation*}
\end{lemma}
\begin{proof}
    Similar to the proof of Lemma \ref{lem: momentum_sum_1}, we have
    \begin{align}
    \label{eq: lemma_5_mid_1}
        \frac{\Vert \bom_t \Vert^2 }{\sqrt{\btwo \bnu_{t-1}}\bnu_t} \le& 4 (1-\bone) \left( \sum_{s=1}^t\frac{ \bone^{t-s} \Vert \bg_s\Vert^2}{\sqrt[4]{\btwo^{3(t-s)}}\sqrt{\btwo \bnu_{s-1}}\bnu_s} \right).
    \end{align}
     Meanwhile, according to Assumption \ref{assum: objective}, we have 
     \begin{align*}
         \Vert \bG_t \Vert^2 \le & \Vert \bG_{t-1} \Vert^2 +2\Vert \bG_{t-1} \Vert\Vert \bG_t -\bG_{t-1} \Vert+ \Vert \bG_t -\bG_{t-1} \Vert^2
         \\
         \le & \Vert \bG_{t-1} \Vert^2 +2\Vert \bG_{t-1} \Vert(L_0+L_1 \Vert\bG_{t-1} \Vert) \Vert \bw_{t+1}-\bw_t\Vert + 2(L_0^2+L_1^2 \Vert\bG_{t-1} \Vert^2) \Vert \bw_{t+1}-\bw_t\Vert^2
         \\
         \le &  \Vert \bG_{t-1} \Vert^2 +\frac{1-\sqrt[8]{\bone}}{3\sqrt[8]{\bone}}\Vert \bG_{t-1} \Vert^2+ \frac{3\sqrt[8]{\bone}L_0^2}{1-\sqrt[8]{\bone}}\Vert \bw_{t+1}-\bw_t\Vert^2+2L_1 \Vert\bG_{t-1} \Vert^2 \Vert \bw_{t+1}-\bw_t \Vert 
         \\
         &+ 2(L_0^2+L_1^2 \Vert\bG_{t-1} \Vert^2) \Vert \bw_{t+1}-\bw_t\Vert^2
         \\
         \overset{(\star)}{\le } & \Vert \bG_{t-1} \Vert^2 +\frac{1-\sqrt[8]{\bone}}{3\sqrt[8]{\bone}}\Vert \bG_{t-1} \Vert^2+ \frac{1-\sqrt[8]{\bone}}{2} \frac{L_0^2}{L_1^2}+\frac{1-\sqrt[8]{\bone}}{3\sqrt[8]{\bone}} \Vert\bG_{t-1} \Vert^2
         \\
         &+ \frac{1-\sqrt[8]{\bone}}{2} \frac{L_0^2}{L_1^2}+ \frac{1-\sqrt[8]{\bone}}{3\sqrt[8]{\bone}}\Vert\bG_{t-1} \Vert^2
         \\
         \le & \frac{1}{\sqrt[8]{\bone}} \Vert \bG_{t-1} \Vert^2+ (1-\sqrt[8]{\bone}) \frac{L_1^2}{L_0^2} .
     \end{align*}
     Here inequality $(\star)$ is because $\Vert \bw_{t+1}-\bw_t \Vert\le \frac{1-\sqrt[8]{\bone}}{6L_1}$. Recursively applying the above inequality, we obtain that
     \begin{align*}
         \Vert \bG_t \Vert^2 \le \frac{1}{\sqrt[8]{\bone^{t-s}}} \Vert \bG_{s} \Vert^2 + \left(\left(\frac{1}{\sqrt[8]{\bone}}\right)^{t-s}-1\right) \frac{L_1^2}{L_0^2},
     \end{align*}
     which by Eq. (\ref{eq: lemma_5_mid_1}) further gives
     \begin{align*}
         \frac{\Vert \bom_t \Vert^2 \Vert \bG_t \Vert^2 }{\sqrt{\btwo \bnu_{t-1}}\bnu_t} \le& 4 (1-\bone) \left( \sum_{s=1}^t\frac{ \bone^{t-s} \Vert \bg_s\Vert^2\Vert \bG_t \Vert^2}{\sqrt[4]{\btwo^{3(t-s)}}\bnu_s\sqrt{\btwo\bnu_{s-1}}} \right)
         \\
         \le & 4 (1-\bone) \left( \sum_{s=1}^t\frac{ \sqrt[8]{\bone^{t-s}} \Vert \bg_s\Vert^2\Vert \bG_s \Vert^2}{\bnu_s\sqrt{\btwo\bnu_{s-1}}} +  \sum_{s=1}^t\frac{ \sqrt[8]{\bone^{t-s}} \Vert \bg_s\Vert^2}{\bnu_s\sqrt{\btwo\bnu_{s-1}}} \frac{L_1^2}{L_0^2}\right)
         \\
         \le & 4 (1-\bone) \left( \sum_{s=1}^t\frac{ \sqrt[8]{\bone^{t-s}} \Vert \bg_s\Vert^2\Vert \bG_s \Vert^2}{\bnu_s\sqrt{\btwo\bnu_{s-1}}}\right) + 8 \frac{1-\bone}{1-\btwo} \frac{L_1^2}{L_0^2}\left(\sum_{s=1}^t\sqrt[8]{\bone^{t-s}}\left(\frac{1}{\sqrt{\btwo \bnu_{s-1}}} -\frac{1}{\sqrt{\bnu_s}}\right) \right).
     \end{align*}
     Here the last inequality is based on the similar reasoning of Lemma \ref{lem: momentum_sum_2}. 
     
     The proof is completed.
\end{proof}

\section{Proofs for deterministic algorithms}

\subsection{Proof for deterministic Adam}
\label{appen: deter_adam}
We  will first provide the formal statement of Theorem \ref{thm: deterministic Adam}, and then show the corresponding proof.

\begin{theorem}[Theorem \ref{thm: deterministic Adam}, restated]
    Let Assumption \ref{assum: objective} hold. Then, $\forall \beta_1, \beta_2$ satisfying $0\le \bone^2 < \btwo <1$, if $T> \frac{L_1^2(f(\bw_1)-f^*)(1-\frac{\bone^2}{\btwo})}{L_0(1-\bone)^2}$, picking $\eta = \frac{\sqrt{f(\bw_1)-f^*}\sqrt{1-\frac{\bone^2}{\btwo}}}{\sqrt{T L_0} (1-\bone)}$, we have
    \begin{equation*}
        \frac{1}{T}\sum_{t=1}^T \Vert \nabla f(\bw_t) \Vert \le \frac{64}{(1-\btwo)(1-\frac{\bone^2}{\btwo})\left(1-\frac{\bone}{\sqrt[4]{\btwo}}\right)^2} \left(\frac{ \sqrt{L_0 (f(\bw_1)-f^*)}}{\sqrt{T}}\right).
    \end{equation*}
\end{theorem}

\begin{proof}
    To begin with, according to Lemma \ref{lem: bounded_update} and restriction on the value of $T$, we obtain that
    \begin{equation*}
        \forall t \in \mathbb{N} ~\& t ~\ge 1, \Vert \bw_{t+1}-\bw_t \Vert \le \frac{1}{4L_1}.
    \end{equation*}

    Therefore, the descent lemma can then be applied and thus $\forall t \in \mathbb{N} \& t \ge 1$,
    \begin{equation*}
        f(\bw_{t+1}) \le f(\bw_t) \underbrace{- \eta \left\langle \bG_t, \frac{\bom_t}{\lambda + \sqrt{\bnu_t}}\right\rangle}_{\text{First Order}} +\underbrace{ \eta^2\frac{L_0+L_1 \Vert \bG_t \Vert }{ 2} \frac{\Vert \bom_t \Vert^2}{(\lambda+\sqrt{\bnu_t})^2}}_{\text{Second Order}}.
    \end{equation*}

    To begin with, as for the "First Order" term, acording to $\bom_t = \bone \bom_{t-1}+(1-\bone) \bG_t$ we have that 
    \begin{align*}
        - \eta \left\langle \bG_t, \frac{\bom_t}{\lambda + \sqrt{\bnu_t}}\right\rangle =& - \eta \frac{1}{1-\bone}\left\langle \bom_t, \frac{\bom_t}{\lambda + \sqrt{\bnu_t}}\right\rangle +\eta \frac{\bone}{1-\bone }\left\langle \bom_{t-1}, \frac{\bom_t}{\lambda + \sqrt{\bnu_t}}\right\rangle
        \\
        \overset{(\star)}{\le } & - \eta \frac{1}{1-\bone}\frac{\Vert \bom_t \Vert^2}{\lambda + \sqrt{\bnu_t} } +\eta \frac{\bone}{(1-\bone)\sqrt[4]{\btwo}}\left\langle \bom_{t-1}, \frac{\bom_t}{\sqrt{\lambda + \sqrt{\bnu_t}}\sqrt{\lambda + \sqrt{\bnu_{t-1}}}}\right\rangle
        \\
        \overset{(\ast)}{\le } &  - \eta \frac{1}{1-\bone}\frac{\Vert \bom_t \Vert^2}{\lambda + \sqrt{\bnu_t} } +{\frac{\bone}{2(1-\bone)\sqrt[4]{\btwo}}}\eta \frac{\Vert \bom_t \Vert^2}{\lambda + \sqrt{\bnu_t} }+{\frac{\bone}{2(1-\bone)\sqrt[4]{\btwo}}}\eta \frac{\Vert \bom_{t-1} \Vert^2}{\lambda + \sqrt{\bnu_{t-1}} }
        \\
        =&  - \eta \frac{1-\frac{\bone}{\sqrt[4]{\btwo}}}{1-\bone}\frac{\Vert \bom_t \Vert^2}{\lambda + \sqrt{\bnu_t} } -{\frac{\bone}{2(1-\bone)\sqrt[4]{\btwo}}}\eta \frac{\Vert \bom_t \Vert^2}{\lambda + \sqrt{\bnu_t} }+{\frac{\bone}{2(1-\bone)\sqrt[4]{\btwo}}}\eta \frac{\Vert \bom_{t-1} \Vert^2}{\lambda + \sqrt{\bnu_{t-1}} }.
    \end{align*}
    where inequality $(\star)$ is due to that $\sqrt{\bnu_t}\ge \sqrt{\btwo \bnu_{t-1}}$ and inequality $(\ast)$ is due to Young's inequality.

    Meanwhile, as for the "Second Order" term, we have 
    \begin{align*}
        \eta^2\frac{L_0+L_1 \Vert \bG_t \Vert }{ 2} \frac{\Vert \bom_t \Vert^2}{(\lambda+\sqrt{\bnu_t})^2}
       \overset{(\bullet)}{\le } & L_0 \eta^2 \frac{(1-\bone)^2}{(1-\btwo)(1-\frac{\bone^2}{\btwo})} + \frac{L_1\eta^2}{\sqrt{1-\btwo}} \frac{\Vert \bom_t \Vert^2}{\lambda+\sqrt{\bnu_t}}
       \\
       \overset{(\circ)}{\le} &  L_0 \eta^2 \frac{(1-\bone)^2}{(1-\btwo)(1-\frac{\bone^2}{\btwo})} +  \frac{\eta}{2} \frac{1-\frac{\bone}{\sqrt[4]{\btwo}}}{1-\bone}\frac{\Vert \bom_t \Vert^2}{\lambda + \sqrt{\bnu_t} }.
    \end{align*}
    Here inequality $(\bullet)$  is due to Lemma \ref{lem: bounded_update} and
    \begin{equation*}
        \bnu_t \ge (1-\btwo) \Vert   \bG_t\Vert ^2,
    \end{equation*}
    and inequality $(\circ)$ is due to the requirement over $T$.

    Applying the estimations of both the  "First Order" and the "Second Order" terms, we obtain that
    \begin{align*}
        f(\bw_{t+1}) -f (\bw_t)\le& - \frac{\eta}{2} \frac{1-\frac{\bone}{\sqrt[4]{\btwo}}}{1-\bone}\frac{\Vert \bom_t \Vert^2}{\lambda + \sqrt{\bnu_t} } -{\frac{\bone}{2(1-\bone)\sqrt[4]{\btwo}}}\eta \frac{\Vert \bom_t \Vert^2}{\lambda + \sqrt{\bnu_t} }+{\frac{\bone}{2(1-\bone)\sqrt[4]{\btwo}}}\eta \frac{\Vert \bom_{t-1} \Vert^2}{\lambda + \sqrt{\bnu_{t-1}} }\\
        & +L_0 \eta^2 \frac{(1-\bone)^2}{(1-\btwo)(1-\frac{\bone^2}{\btwo})} .
        \end{align*}

    Summing the above inequality over $t\in \{1,\cdots,T\}$ then gives
    \begin{equation}
    \label{eq: descent_sum}
    \begin{aligned}
        &\sum_{t=1}^T \frac{\eta}{2} \frac{1-\frac{\bone}{\sqrt[4]{\btwo}}}{1-\bone}\frac{\Vert \bom_t \Vert^2}{\lambda + \sqrt{\bnu_t} }
        \\
        \le& f(\bw_1) -f(\bw_{T+1})-{\frac{\bone}{2(1-\bone)\sqrt[4]{\btwo}}}\eta \frac{\Vert \bom_{T} \Vert^2}{\lambda + \sqrt{\bnu_{T}} }+TL_0 \eta^2 \frac{(1-\bone)^2}{(1-\btwo)(1-\frac{\bone^2}{\btwo})}
        \\
        \le& f(\bw_1) -f(\bw_{T+1})+TL_0 \eta^2 \frac{(1-\bone)^2}{(1-\btwo)(1-\frac{\bone^2}{\btwo})}.
        \end{aligned}
        \end{equation}

    Furthermore, as $(1-\bone) \bG_t = \bom_t -\bone \bom_{t-1} $, we have that
    \begin{equation*}
        \Vert \bG_t \Vert^2 \le \frac{1}{(1-\bone)^2} \Vert \bom_t\Vert^2 +\frac{1}{(1-\bone)^2} \Vert \bom_{t-1}\Vert^2.
    \end{equation*}
    Applying the above inequality and $\lambda = 0$ to Eq. (\ref{eq: descent_sum}), we obtain that
    \begin{equation*}
        \sum_{t=1}^T \frac{\eta}{4} \left(1-\frac{\bone}{\sqrt[4]{\btwo}}\right)(1-\bone)\frac{\Vert \bG_t \Vert^2}{\sqrt{\bnu_t} }\le f(\bw_1) -f(\bw_{T+1})+TL_0 \eta^2 \frac{(1-\bone)^2}{(1-\btwo)(1-\frac{\bone^2}{\btwo})}.
    \end{equation*}

    Meanwhile, we have 
    \begin{equation*}
       \sqrt{\bnu_t}- \sqrt{\btwo \bnu_{t-1}} =\frac{(1-\btwo) \Vert \bG_t\Vert^2 }{\sqrt{\bnu_t}+\sqrt{\btwo \bnu_{t-1}}} \le (1-\btwo)\frac{\Vert \bG_t \Vert^2}{\sqrt{\bnu_t} }.
    \end{equation*}

    Therefore, applying the above inequality and dividing both sides by $\eta$, we have
    \begin{equation*}
       \frac{1}{4} \left(1-\frac{\bone}{\sqrt[4]{\btwo}}\right)(1-\bone) \sum_{t=1}^T (\sqrt{\bnu_t}- \sqrt{\btwo \bnu_{t-1}})\le \frac{f(\bw_1) -f(\bw_{T+1})}{\eta}+TL_0 \eta \frac{(1-\bone)^2}{(1-\btwo)(1-\frac{\bone^2}{\btwo})},
    \end{equation*}
    which by telescoping further leads to
    \begin{equation*}
        \frac{1}{4} \left(1-\frac{\bone}{\sqrt[4]{\btwo}}\right)(1-\bone) \sum_{t=1}^T (1-\btwo)\sqrt{\bnu_t}\le \frac{f(\bw_1) -f(\bw_{T+1})}{\eta}+TL_0 \eta \frac{(1-\bone)^2}{(1-\btwo)(1-\frac{\bone^2}{\btwo})}.
    \end{equation*}

    According to Cauchy-Schwartz's inequality, we then obtain
    \begin{align*}
        \left(\sum_{t=1}^T \Vert \bG_t \Vert\right)^2 \le & \left(\sum_{t=1}^T \sqrt{\bnu_t}\right) \left(\sum_{t=1}^T\frac{\Vert \bG_t \Vert^2}{\sqrt{\bnu_t} }\right)
        \\
        \le &\frac{1}{1-\btwo}\left(\frac{4(f(\bw_1) -f(\bw_{T+1}))}{\eta \left(1-\frac{\bone}{\sqrt[4]{\btwo}}\right)(1-\bone)}+TL_0 \eta \frac{(1-\bone)}{(1-\btwo)\left(1-\frac{\bone}{\sqrt[4]{\btwo}}\right)(1-\frac{\bone^2}{\btwo})}\right)^2
        \\
        =& \frac{1}{1-\btwo}\left(\frac{4(f(\bw_1) -f(\bw_{T+1}))}{\eta \left(1-\frac{\bone}{\sqrt[4]{\btwo}}\right)(1-\bone)}+4TL_0 \eta \frac{(1-\bone)}{(1-\btwo)\left(1-\frac{\bone}{\sqrt[4]{\btwo}}\right)(1-\frac{\bone^2}{\btwo})}\right)^2.
    \end{align*}

    The proof is completed by applying the value of $\eta$.
\end{proof}

\subsection{Proof for GDM}
\label{appen: gd}
This section collects the proof of Theorem \ref{thm: gd}. To begin with, given problem hyperparameters $\Delta_1$, $\varepsilon$, $L_0$, and $L_1$. We first construct three 1D functions as follows:

\begin{equation}
		f_1(x)=\left\{
		\begin{aligned}
		&\frac{L_0 e^{L_1x -1 }}{L_1^2}&,  x\in \left[\frac{1}{L_1},\infty\right), \\
		&\frac{L_0x^2}{2} + \frac{L_0}{2L_1^2}&,  x\in [-\frac{1}{L_1},\frac{1}{L_1}], \\
		  &\frac{L_0 e^{-L_1 x-1}}{L_1^2}&,  x\in \left(-\infty,-\frac{1}{L_1}\right].
		\end{aligned}
		\right.
   \label{lowerbound_f1}
\end{equation}
\begin{equation}
		f_2(y)=\left\{
		\begin{aligned}
		&\varepsilon(y-1)+\frac{\varepsilon}{2}&,  y\in [1,\infty), \\
		&\frac{\varepsilon }{2} y^2&,  y\in [-1,1], \\
		  &-\varepsilon(y+1)+\frac{\varepsilon}{2}&,  y\in (-\infty,-1].
		\end{aligned}
		\right.
  \label{lowerbound_f2}
\end{equation}
\begin{equation}
		f_3(z)=\left\{
		\begin{aligned}
		&\varepsilon(z-1)+\frac{\varepsilon}{2L_1}+ \frac{L_0}{2L_1^2}&,  z\in [1,\infty), \\
		&\frac{\varepsilon L_1 }{2} z^2+ \frac{L_0}{2L_1^2}&,  z\in [0,\frac{1}{L_1}], \\
		  &\frac{L_0z^2}{2} + \frac{L_0}{2L_1^2}&,  z\in [-\frac{1}{L_1},0], \\
		  &\frac{L_0 e^{-L_1 z-1}}{L_1^2}&,  z\in \left(-\infty,-\frac{1}{L_1}\right].
		\end{aligned}
		\right.
  \label{lowerbound_f3}
\end{equation}

% \begin{equation}
% 		f_4(u)=\frac{L_0^2}{2} u^2.
%   \label{lowerbound_f4}
% \end{equation}

It is easy to verify that these functions satisfy $(L_0,L_1)$-smooth condition as long as $\varepsilon \le L_0$. We then respectively the convergence of GDM over these three examples with different learning rate and momentum coefficient.

\begin{lemma}[Convergence over $f_1$]
    Assume $\Delta_1 \ge \frac{L_0}{L_1^2} (e-\frac{1}{2})$, $\varepsilon \le 1$ ,and let $x_1=\frac{1+\log (\frac{1}{2}+\frac{L_1^2}{L_0}\Delta_1)}{L_1}$. Then, we have $f_1(x_1)-f_1^* =\Delta_1$, and if $\eta \geq   \frac{(5+8\log \frac{1}{\varepsilon})(1+\log (\frac{1}{2}+\frac{L_1^2}{L_0}\Delta_1))}{L_1^2(\Delta_1+\frac{L_0}{2L_1^2})}$ and $\beta \le 1- 2\left(\frac{L_1^2}{L_0}e\right)^{-4\log \frac{1}{\varepsilon}-2} (\Delta_1+\frac{L_0}{2L_1^2})^{-4\log \frac{1}{\varepsilon}-2}$, we have that GDM satisfies that $\forall
    t \in [1, \infty)$, $ \vert f_1'(x_t) \vert \ge L_1 \Delta_1$.
\end{lemma}
\begin{proof}

We prove this lemma by proving that $\forall k \ge 1$, $\vert x_{k+1} \vert \ge (4+8 \log \frac{1}{\varepsilon}) \vert x_k \vert$ and $\operatorname{Sign} (x_{k+1}) = (-1)^{k+1}$ by induction.  When $k =1$, according to the update rule of GDM, we have 
\begin{equation*}
    x_{2} = x_1 -\eta f_1'(x_1).
\end{equation*}
As $\eta \geq   \frac{(5+8\log \frac{1}{\varepsilon})(1+\log (\frac{1}{2}+\frac{L_1^2}{L_0}\Delta_1))}{\Delta_1+\frac{L_0}{2L_1^2}} = -\frac{(5+8\log \frac{1}{\varepsilon}) x_1}{f_1'(x_1)}$, we have 
\begin{equation*}
    x_{2} \le -(4+8 \log \frac{1}{\varepsilon}) x_1,
\end{equation*}
which leads to the claim.

Now assuming that the claim has been proved for $k \le t-1$ ($t\ge 2$). Then, for $k = t$, with induction hypothesis we have 
\begin{equation*}
    x_{t+1} = x_t -\eta \bom_t =x_t - \eta \left(\beta^{t} f_1'(x_1) +(1-\beta)\sum_{s=1}^{t-1} \beta^{t-s} f_1'(x_s)+(1-\beta)f_1'(x_t)\right).
\end{equation*}
Without the loss of generality, we assume $t$ is even. By the induction hypothesis, we obtain that $f_1'(x_t) < 0$ and $f_1'(x_{t-1}) < 0$, and 
\begin{equation*}
    \vert f_1'(x_1) \vert \le \vert f_1'(x_2) \vert \le \cdots \le \vert f_1'(x_{t-1}) \vert.
\end{equation*}
Therefore, we have 
\begin{align*}
    x_{t+1} \ge&  x_t - \eta \left(\beta f_1'(x_{t-1})+(1-\beta)f_1'(x_t)\right)
    \\
    = &  x_t - \frac{L_0}{L_1}\eta \left(\beta e^{L_1 x_{t-1}-1}-(1-\beta)e^{-L_1 x_{t}-1}\right)
    \\
    \ge &  x_t - \frac{L_0}{L_1}\eta \left(\beta e^{-\frac{L_1 x_{t}}{8\log \frac{1}{\varepsilon}+4}-1}-(1-\beta)e^{-L_1 x_{t}-1}\right).
\end{align*}

Furthermore, according to the definition of $x_1$, we have 
\begin{equation*}
    1-\beta \ge  2e^{-L_1(4\log \frac{1}{\varepsilon}+2)x_1} \ge   2e^{\frac{L_1x_t}{2}},
\end{equation*}
which leads to 
\begin{align*}
    x_{t+1} \ge  x_t + \frac{L_0}{L_1}\eta e^{-\frac{L_1 x_{t}}{2}-1} \ge x_t+  \frac{(5+8\log \frac{1}{\varepsilon}) x_1}{e^{L_1 x_1}} e^{-\frac{L_1 x_{t}}{2}}\ge  x_t+  \frac{(5+8\log \frac{1}{\varepsilon}) x_1}{e^{L_1 x_1}} e^{L_1 x_{t}(2+4\log \frac{1}{\varepsilon})}.
\end{align*}

Then, as $\frac{e^{\frac{L_1x}{ 2}}}{x}$ is monotonously increasing for $x \in [\frac{2}{L_1}, \infty)$, and $x_1\ge \frac{2}{L_1}$, we have 
\begin{equation*}
     x_{t+1} \ge x_t+  \frac{(5+8\log \frac{1}{\varepsilon}) x_1}{e^{L_1 x_1}} e^{L_1 x_{t}(1+2\log \frac{1}{\varepsilon})}  \ge x_t - (5+8\log \frac{1}{\varepsilon}) x_t \ge -(4+8\log \frac{1}{\varepsilon}) x_t.
\end{equation*}

The proof is completed.
\end{proof}

\begin{lemma}[Convergence over $f_2$]
    Assume that $\Delta_1 \ge \frac{\varepsilon}{2}+ \frac{L_1}{L_0}$, and let $y_1 \triangleq \frac{\Delta_1}{\varepsilon} +\frac{1}{2}$. Then, if $ \eta \le\frac{(5+8\log \frac{1}{\varepsilon})(1+\log (\frac{1}{2}+\frac{L_1^2}{L_0}\Delta_1))}{L_1^2(\Delta_1+\frac{L_0}{2L_1^2})}$, we have that GDM satisfies that $\Vert \nabla  f_2(y_t) \Vert \ge \varepsilon$ if $T \le \tilde{\Theta} (\frac{L_1^2\Delta_1^2+L_0\Delta_1}{\varepsilon^2}) $.
\end{lemma}
\begin{proof}
    We have that $\bom_t  =\varepsilon$ before $y_t$ enters the region $(-\infty, 1]$. As the movement of each step before $y_t$ enters the region $(-\infty, 1]$ is $\eta\varepsilon$ and the total length to enter  $(-\infty, 1]$ is $y_1-1$, the proof is completed.
\end{proof}

\begin{lemma}[Convergence over $f_3$]
     Assume $\Delta_1 \ge \frac{L_0}{L_1^2} e+ 4e +\frac{L_0^2}{e^2L_1^2}$, $L_1 \ge 1$, $\varepsilon\le \frac{1}{2}$, and let $z_1=-\frac{1+\log (\frac{1}{2}+\frac{L_1^2}{L_0}\Delta_1)}{L_1}$. Then, we have $f_3(z_1)-f_3^* =\Delta_1$, and if $\eta \geq   \frac{(5+8\log \frac{1}{\varepsilon})(1+\log (\frac{1}{2}+\frac{L_1^2}{L_0}\Delta_1))}{L_1^2(\Delta_1+\frac{L_0}{2L_1^2})}$ and $\beta \ge 1- 2\left(\frac{L_1^2}{L_0}e\right)^{-4\log \frac{1}{\varepsilon}-2} (\Delta_1+\frac{L_0}{2L_1^2})^{-4\log \frac{1}{\varepsilon}-2}$, we have that GDM satisfies that $\forall
    t \in [1, \Theta(\frac{L_1^2\Delta_1^2}{\varepsilon^3}))$, $ \vert f_3'(x_t) \vert \ge \varepsilon$.
\end{lemma}

\begin{proof}
    To begin with, according to the definition of $z_1$, we have $\eta \ge   -\frac{(5+8\log \frac{1}{\varepsilon}) z_1}{f_3'(x_1)}$ and $ 1-\beta \ge  2e^{L_1(4\log \frac{1}{\varepsilon}+2)z_1}\ge \frac{1}{2}$. Also. as $\Delta_1 \ge \frac{L_0}{L_1^2} (e-\frac{1}{2})$, we have $z_1\le -\frac{2}{L_1}$, and thus
    \begin{equation*}
        f'_3(z_1) = -\frac{L_0}{L_1} e^{-L_1z_1-1} \le - L_1 \left(\Delta_1+\frac{L_0}{2L_1^2}\right) \le -4.
    \end{equation*}
    We will first prove the following claim by induction: for $k \in [2,  \lfloor\frac{1}{1-\beta}\rfloor]$, we have $z_k \ge \frac{1}{L_1}$, and  $\bom_k \le \frac{\beta^{k-1}  f_3'(z_1)}{2}.$

As for $k=2$, we have 
\begin{equation*}
    z_{2} =z_1 -\eta f_3'(z_1) \ge- \left(4+8\log \frac{1}{\varepsilon}\right) z_1. 
\end{equation*}
According to $\Delta_1 \ge \frac{L_0}{L_1^2} (e-\frac{1}{2})$, we have $z_1\le -\frac{2}{L_1}$, and thus $z_2\ge \frac{1}{L_1}$. Since $ \bom_2 = \beta f'(z_1)+(1-\beta) \varepsilon < \frac{  f_3'(z_1)}{2}$, the claim is proved for $k=2$.

Now assuming that we have prove the claim for $k\le t-1$.  According to the induction hypothesis, we have
\begin{equation*}
    f_3'(z_2) = \cdots = f_3'(z_{t-1}) =\varepsilon,
\end{equation*}
and thus
\begin{equation*}
    \bom_t = \beta^{t-1}  f_3'(z_1) +(1-\beta^{t-1}) \varepsilon \overset{(\star)}{\le}  \beta^{t-1}  f_3'(z_1) -\frac{\beta^{t-1}  f_3'(z_1)}{2} \le \frac{\beta^{t-1}  f_3'(z_1)}{2}. 
\end{equation*}
Here inequality $(\star)$ is due to $\beta^{\lfloor\frac{1}{1-\beta}\rfloor} \ge \frac{1}{4}$ as  $\beta \ge \frac{1}{2}$. Therefore, as $z_{t} = z_{t-1} -\eta \bom_t \ge z_{t-1} \ge \frac{1}{L_1}$, we prove the claim.

It should be noticed that $\forall t \in [1, \lfloor \frac{1}{1-\beta}\rfloor]$, $\Vert f_3'(z_t)\vert \ge \varepsilon$. Furthermore, according to the claim, $z_{\lfloor \frac{1}{1-\beta}\rfloor+1}$ can now be bounded as 
\begin{align*} 
    z_{\lfloor \frac{1}{1-\beta}\rfloor+1} =&z_1 -\eta \sum_{k=1}^{\lfloor \frac{1}{1-\beta}\rfloor} \bom_t
    \ge  \frac{\eta}{5+8\log \frac{1}{\varepsilon}}f_3'(z_1)  -\eta \sum_{k=1}^{\lfloor \frac{1}{1-\beta}\rfloor}  \frac{\beta^{k-1}  f_3'(z_1)}{2} \ge \frac{\eta}{5+8\log \frac{1}{\varepsilon}}f_3'(z_1)  -\eta \frac{1-\frac{1}{e}}{(1-\beta)}  \frac{  f_3'(z_1)}{2}
    \\
    \ge & \frac{1}{L_1} -\eta \frac{1-\frac{1}{e}}{(1-\beta)}  \frac{  f_3'(z_1)}{4} \ge  \frac{1}{L_1} -\eta \left(1-\frac{1}{e}\right)  \frac{  f_3'(z_1)}{8} \left(\frac{L_1^2}{L_0}e\right)^{4\log \frac{1}{\varepsilon}+2} \left(\Delta_1+\frac{L_0}{2L_1^2}\right)^{4\log \frac{1}{\varepsilon}+2}
    \\
    \ge & \frac{1}{L_1} +\frac{\eta}{16} \frac{L_1^2 \Delta_1^2+L_0 \Delta_1}{\varepsilon^2} .
\end{align*}

As $f_3'(z) = \varepsilon$ for all $z \ge \frac{1}{L_1}$, the iterates needs additional $\frac{\frac{\eta}{16} \frac{L_1^2 \Delta_1^2}{\varepsilon^2}}{\eta \varepsilon} = \frac{1}{16} \frac{L_1^2 \Delta_1^2}{\varepsilon^3}$ steps to make $f_3'(z_t) <  \varepsilon$. The proof is completed.
\end{proof}

\begin{theorem}[Theorem \ref{thm: gd}, restated]
\label{theorem_gd_appendix}
     Assume that $\Delta_1 \ge 4\frac{L_0}{L_1} e+ 16e +4\frac{L_0^2}{e^2L_1^2}$, $L_1 \ge 1$ and $\varepsilon \le 1$, then there exists objective function $f$ satisfying $(L_0,L_1)$-smooth condition and $f(\bw_1)-f^* =\Delta_1$, such that \textbf{for any learning rate $\eta >0$ and $\beta \in [0, 1]$}, the minimum step $T$ of GDM to achieve final error $\varepsilon$ satisfies 
\begin{equation*}
    T = \tilde{\Omega}\left( \frac{L_1^2 \Delta_1^2+L_0\Delta_1 }{\varepsilon^2} \right).
\end{equation*}
    
\end{theorem}

\begin{proof}
  Construct the objective function as $f(x,y,z,u) =f_1(x) +f_2(y)+ f_3 (z)+f_4(u)$. Then, let $x_1$, $y_1$, $z_1$, $u_1$ be chosen as $f_1(x_1)-f_1^*=f_2(y_1)-f_2^*= f_3(z_1)-f_3^*  =\frac{\Delta_1}{3}
$ and $z_1\le 0$.  Then, for each learning rate and momentum coefficient, they will always be cover by one of the above lemmas, and applying the corresponding lemma gives the desired result.

  The proof is completed.
\end{proof}

\subsection{Proof for Deterministic AdaGrad} 
\label{appen: adagrad}
To begin with, we recall the following result from \citet{wang2023convergence}:
\begin{proposition}
    For every learning rate $\eta \ge \Theta (\frac{1}{L_1})$ and $\Delta_1$, there exist a lower-bounded objective function $g_1$
obeying Assumption \ref{assum: objective} and a corresponding initialization point $\bw_0$ with $g_1(\bw_1)-g_1^* = \Delta_1$, such that AdaGrad with learning
rate $\eta$ and initialized at $\bw_0$ diverges over $g_1$.
\end{proposition}

We then define $g_2$ as the $f_2$ in the proof of Theorem \ref{thm: gd}, i.e., 
\begin{equation}
		g_2(y)=\left\{
		\begin{aligned}
		&\varepsilon(y-1)+\frac{\varepsilon}{2}&,  y\in [1,\infty), \\
		&\frac{\varepsilon }{2} y^2&,  y\in [-1,1], \\
		  &-\varepsilon(y+1)+\frac{\varepsilon}{2}&,  y\in (-\infty,-1].
		\end{aligned}
		\right.
  \label{lowerbound_g2}
\end{equation}
We then  have the following lemma characterizing the convergence of AdaGrad over $g_2$.

\begin{lemma}[Convergence over $g_2$]
    Assume that $\Delta_1 \ge \frac{\varepsilon}{2}+ \frac{L_1}{L_0}$, and let $y_1 \triangleq \frac{\Delta_1}{\varepsilon} +\frac{1}{2}$. Then, if $ \eta \le \Theta (\frac{1}{L_1})$, we have that AdaGrad satisfies that $\Vert \nabla  g_2(y_t) \Vert \ge \varepsilon$ if $T \le \tilde{\Theta} (\frac{L_1^2\Delta_1^2}{\varepsilon^2}) $.
\end{lemma}
\begin{proof}
    We have that $\bg_t  =\varepsilon$ before $y_t$ enters the region $(-\infty, 1]$. Therefore, the sum of movement of each step before $y_t$ enters the region $(-\infty, 1]$ is 
    \begin{equation*}
        \eta \sum_{s=1}^t \frac{\varepsilon}{\sqrt{s} \varepsilon} =\eta \Theta(\sqrt{t}).
    \end{equation*} Solving $\eta \Theta(\sqrt{t}) = \frac{\Delta_1}{\varepsilon} +\frac{1}{2} -1$  gives $t = \frac{L_1^2 \Delta_1^2}{\varepsilon^2}$, and the proof is completed.
\end{proof}

We then have the following lower bound for deterministic AdaGrad.

\begin{theorem}
     Assume that $\Delta_1 \ge \frac{\varepsilon}{2}+ \frac{L_1}{L_0}$. Then, there exists objective function $f$ satisfying $(L_0,L_1)$-smooth condition and $f(\bw_1)-f^* =\Delta_1$, such that \textbf{for any learning rate $\eta >0$ and $\beta \in [0, 1]$}, the minimum step $T$ of AdaGrad to achieve final error $\varepsilon$ satisfies 
     \begin{equation*}
         T = \Omega(\frac{L_1^2\Delta_1^2}{\varepsilon^2}).
     \end{equation*}
\end{theorem}
\begin{proof}
    The proof is completed by letting $f(x,y)=g_1(x)+g_2(y)$ following the same routine as Theorem \ref{theorem_gd_appendix}.
\end{proof}
\section{Proof for stochastic algorithms}

\subsection{Proof for Adam}
\label{appen: adam}
To begin with, we restate the theorem as follows:
\begin{theorem}[Theorem \ref{thm: stochastic  Adam}, restated]
\label{thm: stochastic  Adam_appen}
    Let Assumptions \ref{assum: objective} and \ref{assum: noise} hold. Then, $\forall \beta_1 \ge 0$ and $\lambda =0$, if $\varepsilon\le \frac{1}{\operatorname{poly}(f(\bw_1)-f^*, L_0,L_1, \sigma_0,\sigma_1)}$, with $\eta = \frac{\sqrt{f(\bw_1)-f^*}}{\sqrt{L_0+L_1} \sqrt{T \sigma_0\sigma_1^2}}$ and momentum hyperparameter $\btwo= 1- \eta ^2\left(\frac{1024\sigma_1^2(L_1+L_0)(1-\bone)}{\sqrt{1-\frac{\bone^2}{\btwo}}(1- \frac{\bone}{\sqrt{\btwo}})}\right)^2$, we have if $T \ge  \Theta\left(\frac{(L_0+L_1)\sigma_0^3\sigma_1^2 (f(\bw_1)-f^*) }{\varepsilon^4}\right)$, then Algorithm \ref{alg: adam} satisfies
     \small
    \begin{equation*}
        \frac{1}{T}\E\sum_{t=1}^T \Vert \nabla f(\bw_t) \Vert \le \varepsilon.
    \end{equation*}
    \normalsize
\end{theorem}

\begin{proof}
    Let the approximate iterative sequence be defined as $\bu_t \triangleq \frac{\bw_{t}-\frac{\bone}{\sqrt{\btwo}}\bw_{t-1}}{1-\frac{\bone}{\sqrt{\btwo}}}$ and the surrogate second-order momentum be defined as $\bnut_t \triangleq \btwo \bnu_{t-1}+(1-\btwo) \sigma_0^2$. Then, as $\frac{\eta}{\sqrt{1-\btwo}} =  \frac{\sqrt{1-\frac{\bone^2}{\btwo}}(1- \frac{\bone}{\sqrt{\btwo}})}{1024\sigma_1^2(L_1+L_0)(1-\bone)}$, we have
    \begin{equation*}
        \Vert \bu_{t}-\bw_t \Vert  = \frac{\frac{\bone}{\sqrt{\btwo}} }{1- \frac{\bone}{\sqrt{\btwo}} }\Vert \bw_t-\bw_{t-1} \Vert  \overset{(*)}{\le } \eta \frac{\frac{\bone}{\sqrt{\btwo}} }{1- \frac{\bone}{\sqrt{\btwo}} } \frac{1-\bone}{\sqrt{1-\btwo} \sqrt{1-\frac{\bone^2}{\btwo}}}\le \frac{1}{4L_1},
    \end{equation*}
    and
    \begin{equation*}
        \Vert \bu_{t+1}-\bw_t \Vert  = \frac{1 }{1- \frac{\bone}{\sqrt{\btwo}} } \Vert \bw_{t+1}-\bw_{t} \Vert  \overset{(*)}{\le } \eta  \frac{1 }{1- \frac{\bone}{\sqrt{\btwo}} } \frac{1-\bone}{\sqrt{1-\btwo} \sqrt{1-\frac{\bone^2}{\btwo}}}\le \frac{1}{4L_1}.
    \end{equation*}

    Therefore, if choosing $\bw^1= \bw_t$, $\bw^2 = \bu_{t+1}$, and $\bw^3 = \bu_t$  in Lemma \ref{lem: descent}, we see the conditions of Lemma \ref{lem: descent} is satisfied, which after taking expectation gives
    \begin{align*}
        \E^{|\gF_t} f(\bu_{t+1})\le  f(\bu_t)+ \E^{|\gF_t}\langle \nabla f(\bw_t), \bu_{t+1}-\bu_t\rangle + \frac{1}{2}(L_0+L_1 \Vert \nabla f(\bw_t)\Vert) \E^{|\gF_t}(\Vert \bu_{t+1} -\bw_t\Vert+  \Vert \bu_t-\bw_t\Vert) \Vert\bu_{t+1}-\bu_t\Vert. 
    \end{align*}

    We call $\langle \nabla f(\bw_t), \bu_{t+1}-\bu_t\rangle$ the  first-order term and $ \frac{1}{2}(L_0+L_1 \Vert \nabla f(\bw_t)\Vert)(\Vert \bu_{t+1} -\bw_t\Vert+  \Vert \bu_t-\bw_t\Vert) \Vert\bu_{t+1}-\bu_t\Vert$ the second-order term, as they respectively correspond to the first-order and second-order Taylor's expansion. We then respectively bound these two terms as follows.

    \textbf{Analysis for the first-order term.} Before we start, denote $\bnut_t \triangleq \btwo \bnu_{t-1} +(1-\btwo) \sigma_0^2$ \begin{align*}
    \bu_{t+1}-\bu_t=&\frac{\bw_{t+1}-\bw_t}{1-\frac{\bone}{\sqrt{\btwo}}}-\frac{\bone}{\sqrt{\btwo}} \frac{\bw_{t}-\bw_{t-1}}{1-\frac{\bone}{\sqrt{\btwo}}}
    \\
    =& -\frac{\eta}{1-\frac{\bone}{\sqrt{\btwo}}} \frac{1}{\sqrt{\bnu_t}}  \bom_t+\bone \frac{\eta}{1-\frac{\bone}{\sqrt{\btwo}}} \frac{1}{\sqrt{\btwo\bnu_{t-1}}}  \bom_{t-1}
    \\
    =& -\frac{\eta}{1-\frac{\bone}{\sqrt{\btwo}}} \frac{1}{\sqrt{\tilde{\bnu}_t}}  \bom_t+\bone \frac{\eta}{1-\frac{\bone}{\sqrt{\btwo}}} \frac{1}{\sqrt{\bnut_{t}}}  \bom_{t-1}
 -\frac{\eta}{1-\frac{\bone}{\sqrt{\btwo}}} \left(\frac{1}{\sqrt{\bnu_t}}-\frac{1}{\sqrt{\tilde{\bnu}_t}}\right)  \bom_t
 \\
 &+\bone \frac{\eta}{1-\frac{\bone}{\sqrt{\btwo}}} \left(\frac{1}{\sqrt{\btwo\bnu_{t-1}}}-\frac{1}{\sqrt{\bnut_{t}}}\right)  \bom_{t-1}
    \\
 =& -\eta \frac{1-\bone}{1-\frac{\bone}{\sqrt{\btwo}}} \frac{1}{\sqrt{\tilde{\bnu}_t}}  \bg_t
     -\frac{\eta}{1-\frac{\bone}{\sqrt{\btwo}}} \left(\frac{1}{\sqrt{\bnu_t}}-\frac{1}{\sqrt{\tilde{\bnu}_t}}\right)  \bom_t+\bone \frac{\eta}{1-\frac{\bone}{\sqrt{\btwo}}} \left(\frac{1}{\sqrt{\btwo\bnu_{t-1}}}-\frac{1}{\sqrt{\bnut_{t}}}\right)  \bom_{t-1}.
\end{align*}

According to the above decomposition, we have the first-order term can also be decomposed into
\small
\begin{align}
\nonumber
    &\dE^{|\gF_t} \left[ \left\langle \nabla f(\bw_t) , \bu_{t+1}-\bu_t \right\rangle \right]
    \\
\nonumber
    =& \frac{1-\bone}{1-\frac{\bone}{\sqrt{\btwo}}}\dE^{|\gF_t} \left[ \left\langle\bG_t, -\eta  \frac{1}{\sqrt{\tilde{\bnu}_t}}  \bg_t \right\rangle \right]+\dE^{|\gF_t} \left[ \left\langle\bG_t,-\frac{\eta}{1-\frac{\bone}{\sqrt{\btwo}}} \left(\frac{1}{\sqrt{\bnu_t}}-\frac{1}{\sqrt{\tilde{\bnu}_t}}\right)  \bom_t\right\rangle \right]
    \\
\label{eq: first_order_decomp}
    &+\dE^{|\gF_t} \left[ \left\langle\bG_t,\bone \frac{\eta}{1-\frac{\bone}{\sqrt{\btwo}}} \left(\frac{1}{\sqrt{\btwo\bnu_{t-1}}}-\frac{1}{\sqrt{\bnut_{t}}}\right)  \bom_{t-1}\right\rangle \right].
\end{align}
\normalsize

As $\dE^{|\gF_t} \left[ \left\langle\bG_t, -\eta  \frac{1}{\sqrt{\tilde{\bnu}_t}}  \bg_t \right\rangle \right] = -\eta \frac{ \Vert \bG_t \Vert^2}{\sqrt{\bnut_t}}$, we have
\begin{equation*}
   \frac{1-\bone}{1-\frac{\bone}{\sqrt{\btwo}}} \dE^{|\gF_t} \left[ \left\langle\bG_t, -\eta  \frac{1}{\sqrt{\tilde{\bnu}_t}}  \bg_t \right\rangle \right] \le - \frac{ \Vert \bG_t \Vert^2}{\sqrt{\bnut_t}}.
\end{equation*}We then respectively bound the rest of the two terms in Eq. (\ref{eq: first_order_decomp}). To begin with,
\small
\begin{align}
\nonumber
    &\dE^{|\fil_t} \left[ \left\langle\bG_t,-\frac{\eta}{1-\frac{\bone}{\sqrt{\btwo}}} \left(\frac{1}{\sqrt{\bnu_t}}-\frac{1}{\sqrt{\tilde{\bnu}_t}}\right)    \bom_t\right\rangle \right] 
    \\
\nonumber
    =&  \dE^{|\fil_t} \left[ \left\langle\bG_t,-\frac{\eta}{1-\frac{\bone}{\sqrt{\btwo}}} \left(\frac{(1-\beta_2)(\sigma_0^2-\Vert \bg_t \Vert ^{    2})}{\sqrt{\bnu_t}\sqrt{\tilde{\bnu}_t}(\sqrt{\bnu_t}+\sqrt{\tilde{\bnu}_t})}\right)    \bom_t\right\rangle \right]
    \\
\nonumber
    \le &   \frac{\eta}{1-\frac{\bone}{\sqrt{\btwo}}}\dE^{|\fil_t} \left[ \Vert \bG_{t}\Vert  \left(\frac{(1-\beta_2)(\sigma_0^2+\Vert \bg_{t} \Vert^{ 2})}{\sqrt{\bnu_{t}}\sqrt{\tilde{\bnu}_{t}}(\sqrt{\bnu_{t}}+\sqrt{\tilde{\bnu}_{t}})}\right)\Vert  \bom_{t}\Vert  \right]
    \\
\label{eq: approximation_error}
    =&{  \frac{\eta}{1-\frac{\bone}{\sqrt{\btwo}}}\dE^{|\fil_t} \left[ \Vert \bG_{t}\Vert  \left(\frac{(1-\beta_2)\Vert \bg_{t}\Vert ^{ 2}}{\sqrt{\bnu_{t}}\sqrt{\tilde{\bnu}_{t}}(\sqrt{\bnu_{t}}+\sqrt{\tilde{\bnu}_{t}})}\right)\Vert  \bom_{t}\Vert  \right]}
+{  \frac{\eta}{1-\frac{\bone}{\sqrt{\btwo}}}\dE^{|\fil_t} \left[ \Vert \bG_{t}\Vert  \left(\frac{(1-\beta_2)\sigma_0^2}{\sqrt{\bnu_{t}}\sqrt{\tilde{\bnu}_{t}}(\sqrt{\bnu_{t}}+\sqrt{\tilde{\bnu}_{t}})}\right)\Vert  \bom_{t}\Vert  \right]}.
\end{align}
\normalsize

The first term in the right-hand-side of Eq. (\ref{eq: approximation_error}) can be bounded as 
\small
\begin{align}
\nonumber
&  \frac{\eta}{1-\frac{\beta_1}{\sqrt{\beta_2}}}\dE^{|\fil_t} \left[ \Vert \bG_{t}\Vert  \left(\frac{(1-\beta_2)\Vert \bg_{t} \Vert ^{ 2}}{\sqrt{\bnu_{t}}\sqrt{\tilde{\bnu}_{t}}(\sqrt{\bnu_{t}}+\sqrt{\tilde{\bnu}_{t}})}\right)\Vert  \bom_{t}\Vert  \right]
    \overset{(*)}{\le }   \frac{\eta(1-\bone)}{\left(\sqrt{1-\frac{\bone}{\sqrt{\btwo}}}\right)^3}\dE^{|\fil_t} \left[ \Vert \bG_{t}\Vert  \left(\frac{\sqrt{1-\beta_2} \Vert \bg_{t} \Vert^2 }{\sqrt{\tilde{\bnu}_{t}}(\sqrt{\bnu_{t}}+\sqrt{\tilde{\bnu}_{t}})}\right)  \right]
    \\
    \nonumber
     \overset{(\circ)}{\le } &    \frac{\eta(1-\bone)}{\left(\sqrt{1-\frac{\bone}{\sqrt{\btwo}}}\right)^3} \frac{\Vert \bG_{t}\Vert}{\sqrt{\tilde{\bnu}_{t}} }\sqrt{\dE^{|\fil_t}  \Vert \bg_{t} \Vert ^2  }\sqrt{\dE^{|\fil_t} \frac{  \Vert \bg_{t} \Vert ^2}{(\sqrt{\bnu_{t}}+\sqrt{\tilde{\bnu}_{t}})^2}  }
     \overset{(\bullet)}{\le }  \frac{\eta(1-\bone)\sqrt{1-\btwo}}{\left(\sqrt{1-\frac{\bone}{\sqrt{\btwo}}}\right)^3} \frac{\Vert \bG_{t}\Vert}{\sqrt{\tilde{\bnu}_{t}} }\sqrt{\sigma_0^2+\sigma_1^2 \Vert\bG_{t} \Vert^2 }\sqrt{\dE^{|\fil_t} \frac{  \Vert \bg_{t} \Vert ^2}{(\sqrt{\bnu_{t}}+\sqrt{\tilde{\bnu}_{t}})^2}  }
     \\
     \nonumber
     \le &   \frac{\eta(1-\bone)\sqrt{1-\btwo}}{\left(\sqrt{1-\frac{\bone}{\sqrt{\btwo}}}\right)^3} \frac{\Vert \bG_{t}\Vert}{\sqrt{\tilde{\bnu}_{t}} }(\sigma_0+\sigma_1 \Vert \bG_{t}\Vert)\sqrt{\dE^{|\fil_t} \frac{  \Vert \bg_{t} \Vert ^2}{(\sqrt{\bnu_{t}}+\sqrt{\tilde{\bnu}_{t}})^2}  },
\end{align}
\normalsize
where inequality $(*)$ uses Lemma \ref{lem: bounded_update}, inequality $(\circ)$ is due to Holder's inequality, and inequality $(\bullet)$ is due to Assumption \ref{assum: noise}. Applying mean-value inequality respectively to  $  \frac{\eta(1-\bone)\sqrt{1-\btwo}}{ \left(\sqrt{1-\frac{\bone}{\sqrt{\btwo}}}\right)^3}\dE^{|\fil_t} \frac{\Vert \bG_t\Vert}{\sqrt{\tilde{\bnu}_{t}} }\sigma_0\sqrt{\dE^{|\fil_t} \frac{  \Vert \bg_{t} \Vert^2}{(\sqrt{\bnu_{t}}+\sqrt{\tilde{\bnu}_{t}})^2}  }$ and $  \frac{\eta(1-\bone)\sqrt{1-\btwo}}{\left(\sqrt{1-\frac{\bone}{\sqrt{\btwo}}}\right)^3}\dE^{|\fil_t} \frac{\Vert \bG_t\Vert}{\sqrt{\tilde{\bnu}_{t}} }\sigma_1 \Vert \bG_{t}\Vert\sqrt{\dE^{|\fil_t} \frac{  \Vert \bg_t \Vert^2}{(\sqrt{\bnu_{t}}+\sqrt{\tilde{\bnu}_{t}})^2}  }$ and due to $\bone\le \btwo$, we obtain that the right-hand-side of the above inequality can be bounded by 
\begin{align}
\nonumber
    &\frac{1}{16}  \eta \frac{1-\bone}{1-\frac{\bone}{\sqrt{\btwo}}}\sqrt{1-\btwo}\sigma_0\frac{\Vert \bG_t\Vert^2}{\tilde{\bnu}_{t} }+\frac{4\eta\sqrt{1-\btwo}\sigma_0}{\left(1-\frac{\bone}{\sqrt{\btwo}}\right)^2} \dE^{|\fil_t} \frac{  \Vert \bg_t \Vert^2}{(\sqrt{\bnu_{t}}+\sqrt{\tilde{\bnu}_{t}})^2}
    \\
\nonumber
    &+ \frac{1}{16}    \eta\frac{1-\bone}{1-\frac{\bone}{\sqrt{\btwo}}}\frac{\Vert \bG_t\Vert^2}{\sqrt{\tilde{\bnu}_{t}} }+4\eta \frac{(1-\btwo)(1-\bone)}{(1-\frac{\bone}{\sqrt{\btwo}})^2} \sigma_1^2\frac{\Vert \bG_t\Vert^2}{\sqrt{\tilde{\bnu}_{t}} }\dE^{|\fil_t}  \frac{  \Vert \bg_t \Vert^2}{(\sqrt{\bnu_{t}}+\sqrt{\tilde{\bnu}_{t}})^2}
    \\
    \le & \frac{1}{8}  \eta \frac{\Vert \bG_t\Vert^2}{\sqrt{\tilde{\bnu}_{t} }}+\frac{4\eta \sqrt{1-\btwo}\sigma_0}{\left(1-\frac{\bone}{\sqrt{\btwo}}\right)^2} \dE^{|\fil_t} \frac{  \Vert \bg_t \Vert^2}{\bnu_{t}}
    \label{eq: mid_result}
    + \frac{1}{8}    \eta\frac{\Vert \bG_t\Vert^2}{\sqrt{\tilde{\bnu}_{t}} }+16\eta \frac{(1-\btwo)}{(1-\bone)} \sigma_1^2\frac{\Vert \bG_t\Vert^2}{\sqrt{\tilde{\bnu}_{t}} }\dE^{|\fil_t}  \frac{  \Vert \bg_t \Vert^2}{(\sqrt{\bnu_{t}}+\sqrt{\tilde{\bnu}_{t}})^2}.
\end{align}
Here the inequality is due to $\bnut_{t}=(1-\btwo) \sigma_0^2+\btwo \bnu_{t-1} \ge (1-\btwo) \sigma_0^2$. Meanwhile, we have
\begin{align*}
   &\left( \frac{1}{\sqrt{\btwo \bnut_{t}}}- \frac{1}{\sqrt{ \bnut_{t+1}}}\right) \Vert \bG_{t} \Vert^2
   \\
   =& \frac{\Vert \bG_{t} \Vert^2((1-\btwo)^2\sigma_0^2+\btwo(1-\btwo)  \Vert \bg_t \Vert^2)}{\sqrt{\btwo \bnut_{t}}\sqrt{ \bnut_{t+1}} (\sqrt{\btwo \bnut_{t}}+\sqrt{ \bnut_{t+1}})} \ge \frac{\Vert \bG_{t} \Vert^2\btwo(1-\btwo)  \Vert \bg_t \Vert^2}{\sqrt{\btwo \bnut_{t}}\sqrt{ \bnut_{t+1}} (\sqrt{\btwo \bnut_{t}}+\sqrt{ \bnut_{t+1}})}
   \\
   \ge & \frac{1}{4}\frac{\Vert \bG_{t} \Vert^2(1-\btwo)  \Vert \bg_t \Vert^2}{\sqrt{ \bnut_{t}} (\sqrt{ \bnu_{t}}+\sqrt{ \bnut_{t}})^2},
\end{align*}
where in the last inequality, we use $\sqrt{\btwo}\ge \frac{1}{2}$.
Applying the above inequality back to Eq. (\ref{eq: mid_result}), we obtain that
\begin{align}
\nonumber
    &  \frac{\eta}{1-\beta_1}\dE^{|\fil_t} \left[ \Vert \bG_{t}\Vert  \left(\frac{(1-\beta_2)\bg_{t}^{ 2}}{\sqrt{\bnu_{t}}\sqrt{\tilde{\bnu}_{t}}(\sqrt{\bnu_{t}}+\sqrt{\tilde{\bnu}_{t}})}\right)\Vert  \bom_{t}\Vert  \right]
    \\
    \le & \frac{1}{4}    \eta\frac{\Vert \bG_t\Vert^2}{\sqrt{\tilde{\bnu}_{t}} }+\frac{4\eta \sqrt{1-\btwo}\sigma_0}{\left(1-\frac{\bone^2}{\btwo}\right)^2} \dE^{|\fil_t} \frac{  \Vert \bg_t \Vert^2}{\bnu_{t}}
\label{eq: mid_result_2}
    + \eta \frac{64}{(1-\bone) } \sigma_1^2 \dE^{|\fil_t}\left( \frac{1}{\sqrt{\btwo \bnut_{t}}}- \frac{1}{\sqrt{ \bnut_{t+1}}}\right) \Vert \bG_{t} \Vert^2.
\end{align}

Furthermore, due to Assumption \ref{assum: objective}, we have (we define $G_0\triangleq G_1$)
\begin{align*}
    \Vert \bG_{t+1} \Vert^2\le & \Vert \bG_{t}\Vert^2+2\Vert \bG_t\Vert \Vert \bG_{t+1}-\bG_{t}\Vert  + \Vert\bG_{t+1}-\bG_{t}\Vert ^2
    \\
    \le & \Vert \bG_{t}\Vert^2+2(L_0+L_1 \Vert \bG_t \Vert)\Vert \bG_t\Vert \Vert \bw_{t+1}-\bw_{t}\Vert  + 2(L_0^2+L_1^2\Vert \bG_t \Vert^2) \Vert \bw_{t+1}-\bw_{t}\Vert^2,
\end{align*}
which by  $\frac{\eta}{\sqrt{1-\btwo}} =\frac{\sqrt{1-\frac{\bone^2}{\btwo}}(1- \frac{\bone}{\sqrt{\btwo}})^2}{1024\sigma_1^2(L_1+L_0)(1-\bone)}$ further leads to
\begin{align*}
    &\frac{1}{\sqrt{\btwo \bnut_{t+1}}} \Vert \bG_{t} \Vert^2
    \\
    \ge& \frac{1}{\sqrt{\btwo \bnut_{t+1}}} \left(\Vert \bG_{t+1}\Vert^2-2(L_0+L_1 \Vert \bG_t \Vert)\Vert \bG_t\Vert \Vert \bw_{t+1}-\bw_{t}\Vert  -2 (L_0^2+L_1^2\Vert \bG_t \Vert^2) \Vert \bw_{t+1}-\bw_{t}\Vert^2\right)
    \\
    \ge &  \left(\frac{1}{\sqrt{\btwo \bnut_{t+1}}}\Vert \bG_{t+1}\Vert^2-\frac{2L_0}{\sigma_0}\frac{(1-\bone)}{64\sigma_1^2} \Vert \bw_{t+1} -\bw_{t}\Vert^2 - \frac{3}{8} \frac{(1-\bone)}{64\sigma_1^2}  \frac{\Vert \bG_t\Vert^2}{\sqrt{ \bnut_{t}}} \right) .
\end{align*}

Applying the above inequality back to Eq. (\ref{eq: mid_result_2}) leads to that
\begin{align}
\nonumber
  &   \frac{\eta}{1-\frac{\beta_1}{\sqrt{\btwo}}}\dE^{|\fil_t} \left[ \Vert \bG_{t}\Vert  \left(\frac{(1-\beta_2)\bg_{t}^{ 2}}{\sqrt{\bnu_{t}}\sqrt{\tilde{\bnu}_{t}}(\sqrt{\bnu_{t}}+\sqrt{\tilde{\bnu}_{t}})}\right)\Vert  \bom_{t}\Vert  \right]
   \\
\nonumber
\le & \frac{5}{8}    \eta\frac{\Vert \bG_t\Vert^2}{\sqrt{\tilde{\bnu}_{t}} }+\frac{4\eta \sqrt{1-\btwo}\sigma_0}{\left(1-{\bone}\right)^2} \dE^{|\fil_t} \frac{  \Vert \bg_t \Vert^2}{\bnu_{t}}
    + \eta \frac{64}{(1-\bone)} \sigma_1^2 \dE^{|\fil_t}\left( \frac{\Vert \bG_{t}\Vert^2}{\sqrt{\btwo \bnut_{t}}}- \frac{\Vert \bG_{t+1} \Vert^2}{\sqrt{ \bnut_{t+1}}}\right)   
    \\
    \label{eq: estimation_I11}
    &+ 2 \frac{L_0}{\sigma_0} \E^{|\gF_t} \Vert \bw_{t+1} -\bw_t \Vert^2.
\end{align}

As for the second term in the right-hand-side of Eq. (\ref{eq: approximation_error}), we have
\begin{align}
\nonumber
&   \frac{\eta}{1-\frac{\bone}{\sqrt{\btwo}}}\dE^{|\fil_t} \left[ \Vert \bG_{t}\Vert  \left(\frac{(1-\beta_2)\sigma_0^2}{\sqrt{\bnu_{t}}\sqrt{\tilde{\bnu}_{t}}(\sqrt{\bnu_{t}}+\sqrt{\tilde{\bnu}_{t}})}\right)\Vert  \bom_{t}\Vert  \right]
     \\
\nonumber
     \le &  \frac{\eta}{1-\frac{\bone}{\sqrt{\btwo}}}\dE^{|\fil_t} \left[ \Vert \bG_{t}\Vert  \left(\frac{\sqrt[4]{1-\beta_2}\sqrt {\sigma_0}}{\sqrt[4]{\tilde{\bnu}_{t}}\sqrt{\bnu_{t}}}\right)\Vert  \bom_{t}\Vert  \right]
     \\
\label{eq: estimation_I12}
     \le & \frac{1}{8}  \eta \frac{\Vert \bG_t\Vert^2}{\sqrt{\tilde{\bnu}_{t}} }+  \frac{8\eta\sqrt{1-\beta_2}\sigma_0}{(1-\beta_1)^2}\dE^{|\fil_t} \left[   \left(\frac{\Vert  \bom_{t}\Vert^2}{\bnu_{t}}\right)  \right].
\end{align}
In the last inequality we use again $\btwo \ge \bone$. With Inequalities (\ref{eq: estimation_I11}) and (\ref{eq: estimation_I12}), we conclude that the first-order term can be bounded by
\begin{align}
\nonumber
   \dE^{|\gF_t} \left[ \left\langle \nabla f(\bw_t) , \bu_{t+1}-\bu_t \right\rangle \right]\le  & - \frac{1}{4}    \eta\dE\frac{\Vert \bG_t\Vert^2}{\sqrt{\tilde{\bnu}_{t}} }+\frac{4\eta \sqrt{1-\btwo}\sigma_0}{\left(1-\bone\right)^2} \dE^{|\fil_t} \frac{  \Vert \bg_t \Vert^2}{\bnu_{t}}
    + \eta \frac{64}{(1-{\bone})} \sigma_1^2 \dE^{|\fil_t}\left( \frac{\Vert \bG_{t}\Vert^2}{\sqrt{\btwo \bnut_{t}}}- \frac{\Vert \bG_{t+1} \Vert^2}{\sqrt{ \bnut_{t+1}}}\right)   
    \\
    \label{eq: estimate_first}
    &+ 2 \frac{L_0}{\sigma_0} \E^{|\gF_t} \Vert \bw_{t+1} -\bw_t \Vert^2 +  \frac{8\eta\sqrt{1-\beta_2}\sigma_0}{(1-\beta_1)^2}\dE^{|\fil_t} \left[   \left(\frac{\Vert  \bom_{t}\Vert^2}{\bnu_{t}}\right)  \right].
\end{align}

    \textbf{Analysis for the second-order term.} To recall, the second order term is $ \frac{1}{2}(L_0+L_1 \Vert \nabla f(\bw_t)\Vert)(\Vert \bu_{t+1} -\bw_t\Vert+  \Vert \bu_t-\bw_t\Vert) \Vert\bu_{t+1}-\bu_t\Vert$. Before we start, we have the following expansion for $\bu_{t+1}-\bu_t$: (here the operations are all coordinate-wisely)
    \begin{align}
    \nonumber
         \bu_{t+1} -\bu_t  = &  \frac{\bw_{t+1}-\bw_t-\frac{\bone}{\sqrt{\btwo}}(\bw_t-\bw_{t-1})}{1-\frac{\bone}{\sqrt{\btwo}}} 
        \\
    \nonumber
        =& \frac{-\eta \frac{\bom_t}{\sqrt{\bnu_{t}}}+\eta\frac{\bone}{\sqrt{\btwo}}\frac{\bom_{t-1}}{\sqrt{\bnu_{t-1}}}}{1-\frac{\bone}{\sqrt{\btwo}}} = \frac{-\eta \frac{\bom_t}{\sqrt{\bnu_{t}}}+\eta\bone\frac{\bom_{t-1}}{\sqrt{\bnu_{t}}}-\eta\bone\frac{\bom_{t-1}}{\sqrt{\bnu_{t}}}+\eta\frac{\bone}{\sqrt{\btwo}}\frac{\bom_{t-1}}{\sqrt{\bnu_{t-1}}}}{1-\frac{\bone}{\sqrt{\btwo}}}
        \\
    \label{eq: expansion_u_t}
        =&  \frac{-\eta \frac{(1-\bone)\bg_t}{\sqrt{\bnu_{t}}}+\eta\frac{\bone(1-\btwo) \Vert  \bg_t\Vert^2 }{\sqrt{\btwo}}\frac{\bom_{t-1}}{\sqrt{\bnu_{t-1}}\sqrt{\bnu_{t}}(\sqrt{\bnu_{t}}+\sqrt{\btwo\bnu_{t-1}})}}{1-\frac{\bone}{\sqrt{\btwo}}}
    \end{align} Then firstly, we have
    \small
    \begin{align*}
       &\frac{1}{2}L_0(\Vert \bu_{t+1} -\bw_t\Vert+  \Vert \bu_t-\bw_t\Vert) \Vert\bu_{t+1}-\bu_t\Vert
       \\
       \le & \frac{1}{2} L_0 \left(\Vert \bu_{t+1} -\bu_t\Vert^2+ \frac{1}{2} \Vert \bu_{t+1}-\bw_t \Vert^2 +\frac{1}{2}\Vert \bu_t -\bw_t \Vert^2\right)
       \\
       = & \frac{1}{2}L_0 \left(\left\Vert  \frac{-\eta \frac{(1-\bone)\bg_t}{\sqrt{\bnu_{t}}}+\eta\frac{\bone(1-\btwo) \Vert  \bg_t\Vert^2 }{\sqrt{\btwo}}\frac{\bom_{t-1}}{\sqrt{\bnu_{t-1}}\sqrt{\bnu_{t}}(\sqrt{\bnu_{t}}+\sqrt{\btwo\bnu_{t-1}})}}{1-\frac{\bone}{\sqrt{\btwo}}} \right\Vert^2 + \frac{1}{2} \left\Vert  \frac{\frac{\bone}{\sqrt{\btwo}} }{1- \frac{\bone}{\sqrt{\btwo}} }( \bw_t-\bw_{t-1} )\right\Vert^2+ \frac{1}{2} \left\Vert \frac{1 }{1- \frac{\bone}{\sqrt{\btwo}} } ( \bw_{t+1}-\bw_{t} ) \right\Vert^2 \right)
       \\
       \le & \frac{L_0\eta^2 }{2} \left(\left(\frac{1-\bone}{1-\frac{\bone}{\sqrt{\btwo}}}+ \frac{\bone(1-\bone)}{(\sqrt{\btwo}-\bone)\sqrt{1-\frac{\bone^2}{\btwo}}}\right)^2 \left\Vert \frac{\bg_t}{\sqrt{\bnu_t}} \right\Vert^2 +  \frac{1}{2}\left(\frac{\frac{\bone}{\sqrt{\btwo}} }{1- \frac{\bone}{\sqrt{\btwo}} }\right)^2\left\Vert \frac{\bom_{t-1}}{\sqrt{\bnu_{t-1}}} \right\Vert^2+\frac{1}{2}\left(\frac{1 }{1- \frac{\bone}{\sqrt{\btwo}} }\right)^2\left\Vert \frac{\bom_{t}}{\sqrt{\bnu_{t}}} \right\Vert^2 \right)
       \\
       \overset{(\bullet)}{\le } &\frac{L_0\eta^2 }{2} \left(2\left(\frac{1-\bone}{1-\frac{\bone}{\sqrt{\btwo}}}+ \frac{\bone(1-\bone)}{(\sqrt{\btwo}-\bone)\sqrt{1-\frac{\bone^2}{\btwo}}}\right)^2 \left\Vert \frac{\bg_t}{\sqrt{\bnu_t}} \right\Vert^2 +  \left(\frac{\frac{\bone}{\sqrt{\btwo}} }{1- \frac{\bone}{\sqrt{\btwo}} }\right)^2\left\Vert \frac{\bom_{t-1}}{\sqrt{\bnu_{t-1}}} \right\Vert^2\right).
    \end{align*}
    \normalsize
    Secondly, we have 
    \begin{align*}
         &\frac{1}{2}L_1\Vert \nabla f(\bw_t)\Vert (\Vert \bu_{t+1} -\bw_t\Vert+  \Vert \bu_t-\bw_t\Vert) \Vert\bu_{t+1}-\bu_t\Vert
         \\
         \le & \frac{1}{2}L_1\Vert \nabla f(\bw_t)\Vert (2\Vert \bu_{t+1} -\bw_t\Vert+  \Vert \bu_{t+1}-\bu_t\Vert)\left( \frac{\left\Vert\eta \frac{(1-\bone)\bg_t}{\sqrt{\bnu_{t}}}\right\Vert}{1-\frac{\bone}{\sqrt{\btwo}}}+\frac{\eta\frac{\bone(1-\btwo) \Vert  \bg_t\Vert^2 }{\sqrt{\btwo}}\frac{\Vert \bom_{t-1} \Vert}{\sqrt{\bnu_{t-1}}\sqrt{\bnu_{t}}(\sqrt{\bnu_{t}}+\sqrt{\btwo\bnu_{t-1}})}}{1-\frac{\bone}{\sqrt{\btwo}}}\right)
         \\
         \overset{(*)}{\le} & \frac{1}{2}L_1\Vert \nabla f(\bw_t)\Vert (2\Vert \bu_{t+1} -\bw_t\Vert+  \Vert \bu_{t+1}-\bu_t\Vert)\left( \frac{\left\Vert\eta \frac{(1-\bone)\bg_t}{\sqrt{\bnu_{t}}}\right\Vert}{1-\frac{\bone}{\sqrt{\btwo}}}+
         \frac{\eta\frac{\bone(1-\bone) }{\sqrt{\btwo}}\frac{\Vert \bg_{t} \Vert}{\sqrt{\bnu_{t}}}}{(1-\frac{\bone}{\sqrt{\btwo}})\sqrt{1-\frac{\bone^2}{\btwo}}}\right)
         \\
         =& \frac{L_1}{2}\eta \left(\frac{1-\bone}{1-\frac{\bone}{\sqrt{\btwo}}}+ \frac{\bone(1-\bone)}{(\sqrt{\btwo}-\bone)\sqrt{1-\frac{\bone^2}{\btwo}}}\right) \Vert \nabla f(\bw_t ) \Vert (2\Vert \bu_{t+1} -\bw_t\Vert+  \Vert \bu_t-\bu_{t+1}\Vert)  \frac{\Vert \bg_t \Vert}{\sqrt{\bnu_t}} 
         \\
         \overset{(\circ)}{=} &\frac{L_1}{2}\eta \left(\frac{1-\bone}{1-\frac{\bone}{\sqrt{\btwo}}}+ \frac{\bone(1-\bone)}{(\sqrt{\btwo}-\bone)\sqrt{1-\frac{\bone^2}{\btwo}}}\right) \Vert \bG_t \Vert \left(\Vert \bu_{t+1} -\bu_t \Vert +  2\frac{1 }{1- \frac{\bone}{\sqrt{\btwo}} } \eta \left\Vert \frac{\bom_t}{\sqrt{\bnu_t}}\right\Vert\right)  \frac{\Vert \bg_t \Vert}{\sqrt{\bnu_t}}.
    \end{align*}
    where inequality $(*)$ is due to that $\frac{\Vert \bom_{t-1} \Vert}{\sqrt{\bnu_{t-1}}} \le \frac{1-\bone}{\sqrt{1-\btwo} \sqrt{1-\frac{\bone^2}{\btwo}}}$, $\frac{\Vert \bg_t \Vert }{\sqrt{\bnu_t}} \le \frac{1}{\sqrt{1-\btwo}}$, and equation $(\circ)$ is due to $  \bu_{t}-\bw_t   = \frac{\frac{\bone}{\sqrt{\btwo}} }{1- \frac{\bone}{\sqrt{\btwo}} }( \bw_t-\bw_{t-1} ) $ and  $ \bu_{t+1}-\bw_t   = \frac{1 }{1- \frac{\bone}{\sqrt{\btwo}} } ( \bw_{t+1}-\bw_{t} )$. As for the term $\Vert \bG_t \Vert \frac{ \Vert\bom_t\Vert}{\sqrt{\bnu_t}}  \frac{\Vert \bg_t \Vert }{\sqrt{\bnu_t}} $, we first add additional denominator for it. Specifically, we have 
    \begin{align*}
        \Vert \bG_t \Vert \frac{ \Vert\bom_t\Vert}{\sqrt{\bnu_t}}  \frac{\Vert \bg_t \Vert }{\sqrt{\bnu_t}} = &\frac{ \Vert \bG_t \Vert  \Vert\bom_t\Vert \Vert \bg_t \Vert }{\bnu_t+(1-\btwo) \sigma_0^2}+\frac{ \Vert \bG_t \Vert  \Vert\bom_t\Vert \Vert \bg_t \Vert (1-\btwo) \sigma_0^2}{(\bnu_t+(1-\btwo) \sigma_0^2)\bnu_t}
        \\
        \le &\frac{ \Vert \bG_t \Vert  \Vert\bom_t\Vert \Vert \bg_t \Vert }{\bnu_t+(1-\btwo) \sigma_0^2}+\frac{ \Vert \bG_t \Vert  \Vert\bom_t\Vert   \sigma_0}{\sqrt{\bnu_t+(1-\btwo) \sigma_0^2}\sqrt{\bnu_t}}
        \\
        \le & \frac{ \Vert \bG_t \Vert  \Vert\bom_t\Vert \Vert \bg_t \Vert }{\bnu_t+(1-\btwo) \sigma_0^2}+\frac{1}{2}\frac{ \Vert \bG_t \Vert ^2   \sigma_0}{\bnu_t+(1-\btwo) \sigma_0^2} +\frac{1}{2} \sigma_0 \frac{\Vert \bom_t \Vert^2}{\bnu_t}
        \\
        \le & \frac{ \Vert \bG_t \Vert  \Vert\bom_t\Vert \Vert \bg_t \Vert }{\bnu_t+(1-\btwo) \sigma_0^2}+\frac{1}{2\sqrt{1-\btwo}}\frac{ \Vert \bG_t \Vert ^2  }{\sqrt{\bnu_t+(1-\btwo) \sigma_0^2}} +\frac{1}{2} \sigma_0 \frac{\Vert \bom_t \Vert^2}{\bnu_t}.
    \end{align*}
    
    We analyze the first term in the right-hand-side of above inequality more carefully. Specifically, this term with expectation can be bounded as 
    \begin{align*}
         &\E^{|\gF_t}\frac{ \Vert \bG_t \Vert  \Vert\bom_t\Vert \Vert \bg_t \Vert }{\bnu_t+(1-\btwo) \sigma_0^2}
         \\
         \le &  \E^{|\gF_t}\frac{ \Vert \bG_t \Vert  \Vert\bom_t\Vert \Vert \bg_t \Vert }{\sqrt{\bnu_t+(1-\btwo) \sigma_0^2}\sqrt{\btwo \bnu_{t-1}+(1-\btwo) \sigma_0^2}
         }
         \\
         \le & \frac{\Vert \bG_t\Vert }{\sqrt{\btwo \bnu_{t-1}+(1-\btwo) \sigma_0^2}}\sqrt{ \Vert  \bg_t \Vert^2}\sqrt{\E^{|\gF_t} \frac{   \Vert\bom_t\Vert^2  }{\bnu_t+(1-\btwo) \sigma_0^2}}
        \\
        \overset{(\star)}{\le}  & \frac{\Vert \bG_t\Vert }{\sqrt{\btwo \bnu_{t-1}+(1-\btwo) \sigma_0^2}}\sqrt{\sigma_1^2 \Vert  \bG_t \Vert^2+\sigma_0^2}\sqrt{\E^{|\gF_t} \frac{   \Vert\bom_t\Vert^2  }{\bnu_t+(1-\btwo) \sigma_0^2}} 
        \\
        \le & \frac{\Vert \bG_t\Vert }{\sqrt{\btwo \bnu_{t-1}+(1-\btwo) \sigma_0^2}}(\sigma_1 \Vert  \bG_t \Vert+\sigma_0)\sqrt{\E^{|\gF_t} \frac{   \Vert\bom_t\Vert^2  }{\bnu_t+(1-\btwo) \sigma_0^2}} 
        \\
        \le & \frac{1-\bone}{\sqrt{1-\btwo}\sqrt{1-\frac{\bone^2}{\btwo}}}\sigma_1\frac{\Vert \bG_t\Vert^2 }{\sqrt{\btwo \bnu_{t-1}+(1-\btwo) \sigma_0^2}} + \frac{1}{2\sqrt{1-\btwo}} \frac{\Vert \bG_t\Vert^2 }{\sqrt{\btwo \bnu_{t-1}+(1-\btwo) \sigma_0^2}} +\frac{\sigma_0}{2}  \E^{|\gF_t} \frac{   \Vert\bom_t\Vert^2  }{\bnu_t+(1-\btwo) \sigma_0^2},
    \end{align*} 
    where Eq. ($\star$) is due to Holder's inequality.

    Meanwhile, due to Eq. (\ref{eq: expansion_u_t}), we have that the term $\vert \bG_t \Vert \Vert \bu_{t+1} -\bu_t \Vert \frac{\Vert \bg_t \Vert }{\sqrt{\bnu_t}}$ can be be bounded as 
    \begin{align*}
        \vert \bG_t \Vert \Vert \bu_{t+1} -\bu_t \Vert \frac{\Vert \bg_t \Vert }{\sqrt{\bnu_t}} \le \eta \left(\frac{1-\bone}{1-\frac{\bone}{\sqrt{\btwo}}}+ \frac{\bone(1-\bone)}{(\sqrt{\btwo}-\bone)\sqrt{1-\frac{\bone^2}{\btwo}}}\right) \Vert \bG_t \Vert  \frac{ \Vert \bg_t\Vert}{\sqrt{\bnu_t}}\frac{\Vert \bg_t \Vert }{\sqrt{\bnu_t}}.
    \end{align*}

    Then, following the similar reasoning above, we have $\vert \bG_t \Vert \Vert \bu_{t+1} -\bu_t \Vert \frac{\Vert \bg_t \Vert }{\sqrt{\bnu_t}}$ can be bounded as
    \begin{align*}
        &\E^{|\gF_t}\Vert \bG_t \Vert  \frac{ \Vert \bg_t\Vert}{\sqrt{\bnu_t}}\frac{\Vert \bg_t \Vert }{\sqrt{\bnu_t}}
        \\
        \le &  
 \frac{1}{\sqrt{1-\btwo}}\sigma_1\frac{\Vert \bG_t\Vert^2 }{\sqrt{\btwo \bnu_{t-1}+(1-\btwo) \sigma_0^2}} + \frac{1}{2\sqrt{1-\btwo}} \frac{\Vert \bG_t\Vert^2 }{\sqrt{\btwo \bnu_{t-1}+(1-\btwo) \sigma_0^2}} +\sigma_0  \E^{|\gF_t} \frac{   \Vert\bg_t\Vert^2  }{\bnu_t+(1-\btwo) \sigma_0^2}
 \\
 &+\frac{1}{2\sqrt{1-\btwo}}\frac{ \Vert \bG_t \Vert ^2  }{\sqrt{\bnu_t+(1-\btwo) \sigma_0^2}} +\frac{1}{2} \sigma_0 \E^{|\gF_t} \frac{\Vert \bg_t \Vert^2}{\bnu_t}.
    \end{align*}

    Putting all the estimations together, we have that the second-order term can be bounded by

    \begin{align}
    \nonumber
        &\E^{|\gF_t}\frac{1}{2}(L_0+L_1 \Vert \nabla f(\bw_t)\Vert)(\Vert \bu_{t+1} -\bw_t\Vert+  \Vert \bu_t-\bw_t\Vert) \Vert\bu_{t+1}-\bu_t\Vert
    \\
    \nonumber
    \le & \frac{ L_1\eta^2}{1-\frac{\bone}{\sqrt{\btwo}}}\left( \frac{1-\bone}{1-\frac{\bone}{\sqrt{\btwo}}}+\frac{\bone(1-\bone)}{(\sqrt{\btwo}-\bone)\sqrt{1-\frac{\bone^2}{\btwo}}}\right) \left(\frac{2}{\sqrt{1-\btwo}}\frac{\Vert \bG_t\Vert^2 }{\sqrt{\btwo \bnu_{t-1}+(1-\btwo) \sigma_0^2}}+\frac{\sigma_0}{2}\E^{|\gF_t} \frac{\Vert \bg_t \Vert^2}{\bnu_t}  \right)
    \\
    \nonumber
    & +\frac{L_0\eta^2 }{2} \left(2\left(\frac{1-\bone}{1-\frac{\bone}{\sqrt{\btwo}}}+ \frac{\bone(1-\bone)}{(\sqrt{\btwo}-\bone)\sqrt{1-\frac{\bone^2}{\btwo}}}\right)^2 \E^{|\gF_t}\left\Vert \frac{\bg_t}{\sqrt{\bnu_t}} \right\Vert^2 +  \left(\frac{\frac{\bone}{\sqrt{\btwo}} }{1- \frac{\bone}{\sqrt{\btwo}} }\right)^2\left\Vert \frac{\bom_{t-1}}{\sqrt{\bnu_{t-1}}} \right\Vert^2\right)
    \\
    \nonumber
    \le & 4\frac{ L_1\eta^2}{1-\bone}\left( 1 +\frac{1}{\sqrt{1-\bone}}\right) \left(\frac{2}{\sqrt{1-\btwo}}\frac{\Vert \bG_t\Vert^2 }{\sqrt{\btwo \bnu_{t-1}+(1-\btwo) \sigma_0^2}}+\frac{\sigma_0}{2}\E^{|\gF_t} \frac{\Vert \bg_t \Vert^2}{\bnu_t}  \right)
    \\    \nonumber
    & +2L_0\eta^2  \left(2\left(1 +\frac{1}{\sqrt{1-\bone}}\right)^2 \E^{|\gF_t}\left\Vert \frac{\bg_t}{\sqrt{\bnu_t}} \right\Vert^2 +  \left(\frac{1 }{1- \bone }\right)^2\left\Vert \frac{\bom_{t-1}}{\sqrt{\bnu_{t-1}}} \right\Vert^2\right)
     \\
    \le & \frac{1}{8 } \eta \frac{\Vert \bG_t \Vert^2}{\sqrt{\bnut_t}}+4\frac{ L_1\eta^2\sigma_0}{(1-\bone)^{\frac{3}{2}}} \E^{|\gF_t} \frac{\Vert \bg_t \Vert^2}{\bnu_t}  
       \label{eq: estimate_second_order}
     +2L_0\eta^2  \left(8\frac{1}{1-\bone} \E^{|\gF_t}\left\Vert \frac{\bg_t}{\sqrt{\bnu_t}} \right\Vert^2 +  \left(\frac{1 }{1- \bone }\right)^2\left\Vert \frac{\bom_{t-1}}{\sqrt{\bnu_{t-1}}} \right\Vert^2\right).
    \end{align}

    Here in the second inequality we use $\btwo\ge \bone$, and in the last inequality we use $\frac{\eta}{\sqrt{1-\btwo}} =\frac{\sqrt{1-\frac{\bone^2}{\btwo}}(1- \frac{\bone}{\sqrt{\btwo}})^2}{1024\sigma_1^2(L_1+L_0)(1-\bone)}$.

Applying the estimations of the first-order term (Eq. (\ref{eq: estimate_first})) and the  second-order term (Eq. (\ref{eq: estimate_second_order})) back into the descent lemma, we derive that 
\begin{align*}
        \E^{|\gF_t} f(\bu_{t+1})\le & f(\bu_t) - \frac{1}{8}    \eta\frac{\Vert \bG_t\Vert^2}{\sqrt{\tilde{\bnu}_{t}} }+\frac{4\eta \sqrt{1-\btwo}\sigma_0}{\left(1-\bone\right)^2} \dE^{|\fil_t} \frac{  \Vert \bg_t \Vert^2}{\bnu_{t}}
    + \eta \frac{64}{(1-{\bone})} \sigma_1^2 \dE^{|\fil_t}\left( \frac{\Vert \bG_{t}\Vert^2}{\sqrt{\btwo \bnut_{t}}}- \frac{\Vert \bG_{t+1} \Vert^2}{\sqrt{ \bnut_{t+1}}}\right)   
    \\
    &+ 2  \frac{L_0}{\sigma_0}\E^{|\gF_t} \Vert \bw_{t+1} -\bw_t \Vert^2 +  \frac{8\eta\sqrt{1-\beta_2}\sigma_0}{(1-\beta_1)^2}\dE^{|\fil_t} \left[   \left(\frac{\Vert  \bom_{t}\Vert^2}{\bnu_{t}}\right)  \right]
    \\
    &+4\frac{ L_1\eta^2\sigma_0}{(1-\bone)^{\frac{3}{2}}} \E^{|\gF_t} \frac{\Vert \bg_t \Vert^2}{\bnu_t}  
     +2L_0\eta^2  \left(8\frac{1}{1-\bone} \E^{|\gF_t}\left\Vert \frac{\bg_t}{\sqrt{\bnu_t}} \right\Vert^2 +  \left(\frac{1 }{1- \bone }\right)^2\left\Vert \frac{\bom_{t-1}}{\sqrt{\bnu_{t-1}}} \right\Vert^2\right). 
    \end{align*}

    Taking expectation to the above inequality and summing it over $t \in [1,T]$ then gives 
    \begin{align*}
        \frac{1}{8} \eta\sum_{t=1}^T \E \frac{\Vert \bG_t \Vert^2}{\sqrt{\bnut_t}} \le & f(\bu_1) -f^* + \eta \frac{64}{(1-{\bone})} \sigma_1^2 \frac{\Vert \bG_{1}\Vert^2}{\sqrt{\btwo \bnut_{1}}} +\eta \frac{64}{(1-{\bone})} \sigma_1^2\left(\frac{1}{\sqrt{\btwo}}-1\right)\sum_{t=1}^T\dE\frac{\Vert \bG_{t}\Vert^2}{\sqrt{ \bnut_{t}}}
        \\
         &+ \left(\frac{4\eta \sqrt{1-\btwo}\sigma_0}{\left(1-\bone\right)^2}+4\frac{ L_1\eta^2\sigma_0}{(1-\bone)^{\frac{3}{2}}}+\frac{16L_0 \eta^2}{1-\bone}\right)\sum_{t=1}^T \E \frac{\Vert \bg_t \Vert^2}{\bnu_t}
         \\
         &+ \left(2\frac{L_0}{\sigma_0}\eta^2+\frac{8\eta\sqrt{1-\beta_2}\sigma_0}{(1-\beta_1)^2}+\frac{2L_0\eta^2}{(1-\bone)^2}\right)\sum_{t=1}^T \E\left\Vert \frac{\bom_{t}}{\sqrt{\bnu_{t}}} \right\Vert^2.
    \end{align*}
    Since $\btwo \ge \frac{1}{2} $ and $1-\btwo \le \frac{1-\bone}{1024 \sigma_1^2}$, we have 
    \begin{equation*}
        \eta \frac{64}{(1-{\bone})} \sigma_1^2\left(\frac{1}{\sqrt{\btwo}}-1\right)\sum_{t=1}^T\dE\frac{\Vert \bG_{t}\Vert^2}{\sqrt{ \bnut_{t}}} \le \frac{1}{16}    \eta\sum_{t=1}^T \E \frac{\Vert \bG_t \Vert^2}{\sqrt{\bnut_t}}.
    \end{equation*}
    By further applying Lemma \ref{lem: sum_momentum} and $\btwo \ge \bone$, we obtain
     \begin{align}
     \nonumber
        &\frac{1}{16} \eta\sum_{t=1}^T \E \frac{\Vert \bG_t \Vert^2}{\sqrt{\bnut_t}}
        \\
        \le & f(\bu_1) -f^* + \eta \frac{64}{(1-{\bone})} \sigma_1^2 \frac{\Vert \bG_{1}\Vert^2}{\sqrt{\btwo \bnut_{1}}} 
        \\
        \nonumber
         &+\frac{1}{1-\btwo} \left(\frac{36\eta \sqrt{1-\btwo}\sigma_0}{\left(1-\bone\right)^2}+4\frac{ L_1\eta^2\sigma_0}{(1-\bone)^{\frac{3}{2}}}+\frac{24L_0 \eta^2}{1-\bone}+8\frac{L_0}{\sigma_0}\eta^2 \right)\left( \E \ln \bnu_T -T\ln\btwo \right)
         \\
         \nonumber
         \le & f(\bw_1) -f^* + \eta \frac{64}{(1-{\bone})} \sigma_1^2 \frac{\Vert \bG_{1}\Vert^2}{\sqrt{\btwo \bnut_{1}}} 
        \\
        \label{eq: first_stage}
         &+ \frac{1}{1-\btwo}\left(\frac{147456\eta^2 (L_0+L_1)\sigma_1^2\sigma_0}{\left(1-\bone\right)^\frac{5}{2}}+4\frac{ L_1\eta^2\sigma_0}{(1-\bone)^{\frac{3}{2}}}+\frac{24L_0 \eta^2}{1-\bone}+8\frac{L_0}{\sigma_0}\eta^2 \right)\left( \E \ln \bnu_T -T\ln\btwo \right). 
    \end{align}

Here last inequality we apply $\frac{\eta}{\sqrt{1-\btwo}} =\frac{\sqrt{1-\frac{\bone^2}{\btwo}}(1- \frac{\bone}{\sqrt{\btwo}})^2}{1024\sigma_1^2(L_1+L_0)(1-\bone)}$. 

Below we transfer the above bound to the bound of $\sum_{t=1}^T \Vert \bG_t \Vert$ by two rounds of divide-and-conquer. In the first round, we will bound $\E \ln \bnu_T$. To start with, we have that
\begin{align*}
   & \frac{\Vert \bG_{t}\Vert^2}{\sqrt{\bnut_{t}}}\mathds{1}_{\Vert G_{t}\Vert \ge \frac{\sigma_0}{\sigma_1}}\ge \frac{\frac{1}{2\sigma_1^2}\mathbb{E}^{|\gF_t}\Vert \bg_{t}\Vert^2}{\sqrt{\bnut_{t}}}\mathds{1}_{\Vert G_{t}\Vert \ge \frac{\sigma_0}{\sigma_1}}
    \\
    =& \frac{\frac{1}{2\sigma_1^2}\mathbb{E}^{|\gF_t}\Vert \bg_{t}\Vert^2}{\sqrt{\btwo\bnu_{t-1}+(1-\btwo)\sigma_0^2}}\mathds{1}_{\Vert G_{t}\Vert \ge \frac{\sigma_0}{\sigma_1}}
    \\
    \ge &\frac{1}{2\sigma_1^2}\mathbb{E}^{|\gF_t} \frac{ \btwo^{T-t}\Vert \bg_{t}\Vert^2}{\sqrt{\bnu_T+(1-\btwo) \sigma_0^2}}\mathds{1}_{\Vert G_{t}\Vert \ge \frac{\sigma_0}{\sigma_1}},
\end{align*}
where the last inequality is due to that
\begin{align}
    \btwo\bnu_{t-1}+(1-\btwo) \sigma_0^2\le  \btwo^{t-T}\bnu_{T}+(1-\btwo) \sigma_0^2
   \label{eq: estimation_specific}
   \le (\bnu_T+(1-\btwo) \sigma_0^2) \btwo^{2(t-T)}.
\end{align}

Furthermore, we have 
\begin{align}
\nonumber
   & \frac{\sigma_0^2+\frac{\btwo^T \bnu_{0}}{1-\btwo}}{\sqrt{\bnu_T+(1-\btwo) \sigma_0^2}}+\sum_{t=1}^T\mathbb{E} \frac{ \btwo^{T-t}\Vert \bg_{t}\Vert^2}{\sqrt{\bnu_T+(1-\btwo) \sigma_0^2}}\mathds{1}_{\Vert \bG_{t}\Vert  < \frac{\sigma_0}{\sigma_1}}
    \\
\nonumber
    \le &\frac{\sigma_0^2+\frac{\btwo^T\bnu_{0}}{1-\btwo}}{\sqrt{\bnu_0 \btwo^T+\sum_{s=1}^T \btwo^{T-s} \Vert g_{s}\Vert^2\mathds{1}_{\Vert \bG_{s}\Vert  < \frac{\sigma_0}{\sigma_1}}+(1-\btwo)\sigma_0^2}}+\sum_{t=1}^T\mathbb{E} \frac{\btwo^{T-s}\Vert \bg_{t}\Vert^2}{\sqrt{\bnu_0 \btwo^T+\sum_{s=1}^T \btwo^{T-s}  \Vert g_{s}\Vert^2\mathds{1}_{\Vert \bG_{s}\Vert  < \frac{\sigma_0}{\sigma_1}}+(1-\btwo)\sigma_0^2}}\mathds{1}_{\Vert \bG_{t}\Vert  < \frac{\sigma_0}{\sigma_1}}
    \\
    =& \frac{1}{1-\btwo}\dE\sqrt{\bnu_0 \btwo^T+\sum_{s=1}^T \btwo^{T-s} \Vert g_{s}\Vert^2\mathds{1}_{\Vert \bG_{s}\Vert  < \frac{\sigma_0}{\sigma_1}}+(1-\btwo)\sigma_0^2}
    \label{eq: thm2_mid}
    \le \frac{1}{1-\btwo}\sqrt{\btwo^T \bnu_{0} +2\sigma_0^2} .
\end{align}

Conclusively, we obtain 
\begin{align*}
&\dE\sqrt{\bnu_T +(1-\btwo) \sigma_0^2}
\\
    =&(1-\btwo)\left(\frac{\sigma_0^2+\frac{\btwo^T \bnu_{0}}{1-\btwo}}{\sqrt{\bnu_T+(1-\btwo) \sigma_0^2}}+\sum_{t=1}^T\mathbb{E} \frac{ \btwo^{T-t}\Vert \bg_{t}\Vert^2}{\sqrt{\bnu_T+(1-\btwo) \sigma_0^2}}\mathds{1}_{\Vert \bG_{t}\Vert  < \frac{\sigma_0}{\sigma_1}}\right.
    \\
    &+\left.\sum_{t=1}^T\mathbb{E} \frac{ \btwo^{T-t}\Vert \bg_{t}\Vert^2}{\sqrt{\bnu_T+(1-\btwo) \sigma_0^2}}\mathds{1}_{\Vert \bG_{t}\Vert  \ge \frac{\sigma_0}{\sigma_1}}\right)
    \\
    \le & \sqrt{\btwo^T \bnu_{0} +2\sigma_0^2}+2(1-\btwo)\sigma_1^2\dE\sum_{t=1}^T\frac{\Vert \bG_{t}\Vert^2}{\sqrt{\bnut_{t}}}\mathds{1}_{\Vert \bG_{t}\Vert\ge \frac{\sigma_0}{\sigma_1}}
    \\
    \le & \sqrt{\btwo^T \bnu_{0} +2\sigma_0^2}+2(1-\btwo)\sigma_1^2\dE\sum_{t=1}^T\frac{\Vert \bG_{t}\Vert^2}{\sqrt{\bnut_{t}}}.
\end{align*}

Substituting $\dE\sum_{t=1}^T\frac{\Vert \bG_{t}\Vert^2}{\sqrt{\bnut_{t}}}$ according to Eq. (\ref{eq: first_stage}), we obtain that

\begin{align*}
   &\dE\sqrt{\bnu_T +(1-\btwo) \sigma_0^2}
\\
    \le & \sqrt{\btwo^T \bnu_{0} +2\sigma_0^2}+\frac{2(1-\btwo)\sigma_1^2}{\eta} \eta\dE\sum_{t=1}^T\frac{\Vert \bG_{t}\Vert^2}{\sqrt{\bnut_{t}}}
    \\
    \le &  \sqrt{\btwo^T \bnu_{0} +2\sigma_0^2}+\frac{2(1-\btwo)\sigma_1^2}{\eta} \left(f(\bw_1) -f^* + \eta \frac{64}{(1-{\bone})} \sigma_1^2 \frac{\Vert \bG_{1}\Vert^2}{\sqrt{\btwo \bnut_{1}}} \right.
        \\
         &+\left. \frac{1}{1-\btwo}\left(\frac{147456\eta^2 (L_0+L_1)\sigma_1^2\sigma_0}{\left(1-\bone\right)^\frac{5}{2}}+4\frac{ L_1\eta^2\sigma_0}{(1-\bone)^{\frac{3}{2}}}+\frac{24L_0 \eta^2}{1-\bone}+8\frac{L_0}{\sigma_0}\eta^2 \right)\left( \E \ln \bnu_T -T\ln\btwo \right)\right)
    \\
  \le &   \sqrt{\btwo^T \bnu_{0} +2\sigma_0^2} + \sigma_0 + \frac{1}{4} \E \ln \bnu_T 
  \\
  \le & \sqrt{\btwo^T \bnu_{0} +2\sigma_0^2} + \sigma_0 + \frac{1}{2} \E  \sqrt{\bnu_T +(1-\btwo) \sigma_0^2}.
\end{align*}
where the third inequality is due to
\begin{align*}
T \ge &\frac{36*2048^4(L_0+L_1)^3\sigma_1^{12}(f(\bw_1)-f^*)}{(1-\bone)^6\sigma_0^2}+\frac{768* 2048^2(f(\bw_1)-f^*) \sigma_1^{8}(8L_1^2(f(\bw_1)-f^*)^2+4L_0(f(\bw_1)-f^*))}{(1-\bone)^4 \sigma_0^2}
\\
 &+ \frac{24^2*147456(L_0+L_1)\sigma_1^8(f(\bw_1)-f^*)\sigma_0^2}{(1-\btwo)^5}+\frac{128^2(L_0+L_1)(f(\bw_1)-f^*)\sigma_1^4}{\sigma_0^2}
 \\
&+ \frac{24^2*147456*2048^2(L_0+L_1)^3\sigma_1^{16}(f(\bw_1)-f^*)^3}{(1-\btwo)^{11}}+\frac{128^2*2048^2(L_0+L_1)^3(f(\bw_1)-f^*)^3\sigma^{12}}{\sigma_0^4(1-\bone)^6},
\end{align*} and the last inequality is due to $ \ln x \le x$. Solving the above inequality with respect to $\dE\sqrt{\bnu_T +(1-\btwo) \sigma_0^2}$ and applying $\bnu_0 =\sigma_0^2$ then gives
\begin{align}
\E  \sqrt{\bnu_T }  \le  \E  \sqrt{\bnu_T +(1-\btwo) \sigma_0^2}
   \label{eq: d_3}
   \le & 6\sigma_0.
\end{align}

Therefore, Eq. (\ref{eq: first_stage}) can be rewritten as 
\begin{align}
     \nonumber
        &\frac{1}{16} \eta\sum_{t=1}^T \E \frac{\Vert \bG_t \Vert^2}{\sqrt{\bnut_t}}
        \\
         \nonumber
         \le & f(\bw_1) -f^* + \eta \frac{64}{(1-{\bone})} \sigma_1^2 \frac{\Vert \bG_{1}\Vert^2}{\sqrt{\btwo \bnut_{1}}} 
        \\
\label{eq: first_stage_renewed}
         &+ \frac{1}{1-\btwo}\left(\frac{147456\eta^2 (L_0+L_1)\sigma_1^2\sigma_0}{\left(1-\bone\right)^\frac{5}{2}}+4\frac{ L_1\eta^2\sigma_0}{(1-\bone)^{\frac{3}{2}}}+\frac{24L_0 \eta^2}{1-\bone}+8\frac{L_0}{\sigma_0}\eta^2 \right)\left( 2\ln 6\sigma_0 -T\ln\btwo \right). 
    \end{align}

We then execute the second round of divide-and-conquer. To begin with, we have that
\begin{equation}
\label{eq: core_2}
    \sum_{t=1}^T\dE\left[  \frac{\Vert \bG_{t}\Vert^2}{\sqrt{\bnut_{t}}}\mathds{1}_{\Vert G_{t}\Vert \ge \frac{\sigma_0}{\sigma_1}} \right]\le \sum_{t=1}^T\dE\left[  \frac{\Vert \bG_{t}\Vert^2}{\sqrt{\bnut_{t}}} \right].
\end{equation}

On the other hand, we have that
\begin{align*}
   & \frac{\Vert \bG_{t}\Vert^2}{\sqrt{\bnut_{t}}}\mathds{1}_{\Vert G_{t}\Vert \ge \frac{\sigma_0}{\sigma_1}}
   \ge \frac{\frac{2}{3}\Vert \bG_{t}\Vert^2+\frac{1}{3}\frac{\sigma^2_0}{\sigma_1^2}}{\sqrt{\bnut_{t}}}\mathds{1}_{\Vert G_{t}\Vert \ge \frac{\sigma_0}{\sigma_1}}
   \ge \frac{\frac{\btwo}{3\sigma_1^2}\mathbb{E}^{|\gF_t}\Vert \bg_{t}\Vert^2+\frac{1-\btwo}{3}\frac{\sigma^2_0}{\sigma_1^2}}{\sqrt{\bnut_{t}}}\mathds{1}_{\Vert G_{t}\Vert \ge \frac{\sigma_0}{\sigma_1}}
    \\
    =& \dE^{|\fil_t}\frac{\frac{\btwo}{3\sigma_1^2}\Vert \bg_{t}\Vert^2+\frac{1-\btwo}{3\sigma_1^2}\sigma_0^2}{\sqrt{\bnut_{t}}}\mathds{1}_{\Vert G_{t}\Vert \ge \frac{\sigma_0}{\sigma_1}}
    \ge \frac{1}{2} \dE^{|\fil_t}\frac{\frac{\btwo}{3\sigma_1^2}\Vert \bg_{t}\Vert^2+\frac{1-\btwo}{3\sigma_1^2}\sigma_0^2}{\sqrt{\bnut_{t+1}}+\sqrt{\btwo \bnut_{t}}}\mathds{1}_{\Vert G_{t}\Vert \ge \frac{\sigma_0}{\sigma_1}}.
\end{align*}
 As a conclusion, 
\begin{align*}
      &\sum_{t=1}^T\dE\left[  \frac{\Vert \bG_{t}\Vert^2}{\sqrt{\bnut_{t}}}\mathds{1}_{\Vert G_{t}\Vert \ge \frac{\sigma_0}{\sigma_1}} \right] \ge  \frac{1}{2}\sum_{t=1}^T\dE\left[\frac{\frac{\btwo}{3\sigma_1^2}\Vert \bg_{t}\Vert^2+\frac{1-\btwo}{3\sigma_1^2}\sigma_0^2}{\sqrt{\bnut_{t+1}}+\sqrt{\btwo \bnut_{t}}}\mathds{1}_{\Vert G_{t}\Vert \ge \frac{\sigma_0}{\sigma_1}}\right]
      \\
      \ge& \frac{1}{6(1-\btwo)\sigma_1^2}\sum_{t=1}^T \dE\left[\left(\sqrt{\bnut_{t+1}}-\sqrt{\btwo\bnut_{t}}\right)\mathds{1}_{\Vert G_{t}\Vert \ge \frac{\sigma_0}{\sigma_1}}\right].
\end{align*}

Meanwhile, for convenience, we define $\{\bar{\bnu}_{t}\}_{t=0}^{\infty}$ as $\bar{\bnu}_{0}= \bnu_{0}$, $\bar{\bnu}_{t}= \btwo\bar{\bnu}_{t-1}+(1-\btwo)\vert g_{t} \vert^2\mathds{1}_{\Vert \bG_{t}\Vert  < \frac{\sigma_0^2}{\sigma_1^2}}$. One can easily observe that $\bar{\bnu}_{t}\le \bnu_{t} $, and thus
\begin{align*}
    &\sum_{t=1}^T \dE\left[\left(\sqrt{\bnut_{t+1}}-\sqrt{\btwo\bnut_{t}}\right)\mathds{1}_{\Vert \bG_{t}\Vert  < \frac{\sigma_0^2}{\sigma_1^2}} \right]
    \\
    =& \sum_{t=1}^T \dE\left(\sqrt{\btwo^2 \bnu_{t-1}+ \btwo(1-\btwo)\Vert g_{t}\Vert^2 + (1-\btwo )\sigma_0^2}-\sqrt{\btwo(\btwo\bnu_{t-1}+ (1-\btwo )\sigma_0^2)}\right)\mathds{1}_{\Vert \bG_{t}\Vert  < \frac{\sigma_0^2}{\sigma_1^2}}
    \\
    \le &\sum_{t=1}^T \dE\left(\sqrt{\btwo^2 \bar{\bnu}_{t-1}+ \btwo(1-\btwo)\Vert g_{t}\Vert^2 + (1-\btwo )\sigma_0^2}-\sqrt{\btwo(\btwo\bar{\bnu}_{t-1}+ (1-\btwo )\sigma_0^2)}\right)\mathds{1}_{\Vert \bG_{t}\Vert  < \frac{\sigma_0^2}{\sigma_1^2}}
    \\
    \le & \sum_{t=1}^T \dE\left(\sqrt{\btwo^2 \bar{\bnu}_{t-1}+ \btwo(1-\btwo)\Vert g_{t}\Vert^2\mathds{1}_{\Vert \bG_{t}\Vert  < \frac{\sigma_0^2}{\sigma_1^2}} + (1-\btwo )\sigma_0^2}-\sqrt{\btwo(\btwo\bar{\bnu}_{t-1}+ (1-\btwo )\sigma_0^2)}\right)
    \\
    =& \sum_{t=1}^T \dE\left(\sqrt{\btwo\bar{\bnu}_{t} + (1-\btwo )\sigma_0^2}-\sqrt{\btwo(\btwo\bar{\bnu}_{t-1}+ (1-\btwo )\sigma_0^2)}\right)
    \\
    =& \dE \sqrt{\btwo\bar{\bnu}_{t} + (1-\btwo )\sigma_0^2}+(1-\sqrt{\btwo})\sum_{t=1}^{T-1} \dE\sqrt{\btwo\bar{\bnu}_{t} + (1-\btwo )\sigma_0^2} - \dE\sqrt{\btwo(\btwo\bar{\bnu}_{0}+ (1-\btwo )\sigma_0^2)}.
\end{align*}
All in all, summing the above two inequalities together, we obtain that
\small
\begin{align}
\nonumber
&\dE \sqrt{\bnut_{t+1} }+(1-\sqrt{\btwo})\sum_{t=2}^{T} \dE\sqrt{\bnut_{t} } - \sqrt{\btwo{\bnut}_{1}}
\\
\nonumber
    =&\sum_{t=1}^T \dE \left(\sqrt{\bnut_{t} }-\sqrt{\btwo{\bnut}_{t-1}}\right)
    \\
\nonumber
    \le &\sum_{t=1}^T \dE \left(\sqrt{\bnut_{t} }-\sqrt{\btwo{\bnut}_{t-1}}\right)\mathds{1}_{\Vert G_{t}\Vert \ge \frac{\sigma_0}{\sigma_1}}
    +\sum_{t=1}^T \dE \left(\sqrt{\bnut_{t} }-\sqrt{\btwo{\bnut}_{t-1}}\right)\mathds{1}_{\Vert \bG_{t}\Vert  < \frac{\sigma_0^2}{\sigma_1^2}}
    \\
\nonumber
\le & \frac{3(1-\btwo)\sigma_1^2}{\sqrt{\btwo}}\sum_{t=1}^T\dE\left[  \frac{\Vert \bG_{t}\Vert^2}{\sqrt{\bnut_{t}}} \right]+ \dE \sqrt{\btwo\bar{\bnu}_{t} + (1-\btwo )\sigma_0^2}+(1-\sqrt{\btwo})\sum_{t=1}^{T-1} \dE\sqrt{\btwo\bar{\bnu}_{t} + (1-\btwo )\sigma_0^2} - \sqrt{\btwo(\btwo\bar{\bnu}_{0}+ (1-\btwo )\sigma_0^2)}.
\end{align}
\normalsize
Since $\forall t \ge 1$,
\begin{align*}
    \dE\sqrt{\btwo \bar{\bnu}_{t} + (1-\btwo )\sigma_0^2} \le \sqrt{\btwo\dE\bar{\bnu}_{t} + (1-\btwo )\sigma_0^2} \le \sqrt{\sigma_0^2+\bnu_{0}} \le \sqrt{2}\sigma_0,
\end{align*}
combining with $\sqrt{\btwo{\bnut}_{1}}=\sqrt{\btwo(\btwo\bar{\bnu}_{0}+ (1-\btwo )\sigma_0^2)}$ and $\dE \sqrt{\bnut_{t+1} } =\dE \sqrt{\btwo{\bnu}_{t} + (1-\btwo )\sigma_0^2} \ge \dE \sqrt{\btwo\bar{\bnu}_{t} + (1-\btwo )\sigma_0^2}$, we obtain

\begin{align}
\nonumber
    (1-\sqrt{\btwo})\sum_{t=1}^{T} \dE\sqrt{\bnut_{t} }
\le&  \frac{3(1-\btwo)\sigma_1^2}{\sqrt{\btwo}}\sum_{t=2}^T\dE\left[  \frac{\Vert \bG_{t}\Vert^2}{\sqrt{\bnut_{t}}} \right]+ +(1-\sqrt{\btwo})\sum_{t=1}^{T} \dE\sqrt{\btwo\bar{\bnu}_{t} + (1-\btwo )\sigma_0^2} 
\\
\nonumber
\le &\frac{3(1-\btwo)\sigma_1^2}{\sqrt{\btwo}}\sum_{t=1}^T\dE\left[  \frac{\Vert \bG_{t}\Vert^2}{\sqrt{\bnut_{t}}} \right]+ \sqrt{2}(1-\sqrt{\btwo})T\sigma_0.
.
\end{align}
\normalsize

Dividing both sides of the above equation by $1-\sqrt{\btwo}$ then gives 
\begin{align}
\nonumber
    \sum_{t=1}^{T} \dE\sqrt{\bnut_{t} }
\le&  \frac{3(1-\btwo)\sigma_1^2}{\sqrt{\btwo}}\sum_{t=2}^T\dE\left[  \frac{\Vert \bG_{t}\Vert^2}{\sqrt{\bnut_{t}}} \right]+(1-\sqrt{\btwo})\sum_{t=1}^{T} \dE\sqrt{\btwo\bar{\bnu}_{t} + (1-\btwo )\sigma_0^2} 
\\
\label{eq: estimation_2}
\le &12\sigma_1^2\sum_{t=1}^T\dE\left[  \frac{\Vert \bG_{t}\Vert^2}{\sqrt{\bnut_{t}}} \right]+ \sqrt{2}T\sigma_0.
\end{align}
By applying Eq. (\ref{eq: first_stage_renewed}) and the constraint of $T$, we obtain that 
\begin{align*}
    \sum_{t=1}^{T} \dE\sqrt{\bnut_{t} } \le &\sqrt{2}T\sigma_0+12\frac{\sigma_1^2}{\eta} \left(f(\bw_1) -f^* + \eta \frac{64}{(1-{\bone})} \sigma_1^2 \frac{\Vert \bG_{1}\Vert^2}{\sqrt{\btwo \bnut_{1}}} \right.
        \\
         &+ \left.\frac{1}{1-\btwo}\left(\frac{147456\eta^2 (L_0+L_1)\sigma_1^2\sigma_0}{\left(1-\bone\right)^\frac{5}{2}}+4\frac{ L_1\eta^2\sigma_0}{(1-\bone)^{\frac{3}{2}}}+\frac{24L_0 \eta^2}{1-\bone}+8\frac{L_0}{\sigma_0}\eta^2 \right)\left( 2\ln 6\sigma_0 -T\ln\btwo \right)\right)
         \\
         \le & 4T\sigma_0.
\end{align*}
Combining the above inequality and Eq. (\ref{eq: first_stage_renewed}) and applying Cauchy's inequality, we obtain that
\begin{align*}
    \left(\E \sum_{t=1}^T \Vert \nabla f(\bw_t) \Vert\right)^2 \le & \left( \sum_{t=1}^{T} \dE\sqrt{\bnut_{t} } \right) \left(\sum_{t=1}^T\dE\left[  \frac{\Vert \bG_{t}\Vert^2}{\sqrt{\bnut_{t}}} \right]\right)
    \\
    \le & 4T\sigma_0\times \frac{1}{\eta} \left(f(\bw_1) -f^* + \eta \frac{64}{(1-{\bone})} \sigma_1^2 \frac{\Vert \bG_{1}\Vert^2}{\sqrt{\btwo \bnut_{1}}} \right.
        \\
         &+ \left.\frac{1}{1-\btwo}\left(\frac{147456\eta^2 (L_0+L_1)\sigma_1^2\sigma_0}{\left(1-\bone\right)^\frac{5}{2}}+4\frac{ L_1\eta^2\sigma_0}{(1-\bone)^{\frac{3}{2}}}+\frac{24L_0 \eta^2}{1-\bone}+8\frac{L_0}{\sigma_0}\eta^2 \right)\left( 2\ln 6\sigma_0 -T\ln\btwo \right)\right).
\end{align*}

By $\eta = \frac{\sqrt{f(\bw_1)-f^*}}{\sqrt{L_0+L_1} \sqrt{T \sigma_0\sigma_1^2}}$ and the constraint of $T$, the proof is completed.
 \end{proof}

\subsection{Proof for SGDM}
\label{appen: sgd}

\begin{theorem}[Informal]
    \label{thm: sgd_appen}
    Fix $ L_0\ge 0, L_1 >0$, and $\Delta_1 \ge 0$,  there exists objective function $f$ satisfying $(L_0,L_1)$-smooth condition and $f(\bw_1)-f^* =\Delta_1$, and a noise oracle $\mathcal{O}(\bw,z)$ generating stochastic gradient by $\bg_t = \nabla f(\bw_t) + \mathcal{O}(\bw_t,z_t)$ and satisfying Assumption \ref{assum: noise} ($z_t$ is i.i.d. sampled from some underlying distribution), such that \textbf{for any learning rate $\eta >0$ and $\beta \in [0, 1]$}, for all $T >0$, 
    \begin{equation*}
        \min_{t\in [T]} \E \Vert \nabla f(\bw_t) \Vert = \Vert \nabla f(\bw_1) \Vert \ge L_1\Delta_1.
    \end{equation*}
\end{theorem}

\begin{proof}
   Define the objective function $f$ as $f_1$ used in the the proof of Theorem \ref{thm: gd}  as
   \begin{equation}
		f_1(x)=\left\{
		\begin{aligned}
		&\frac{L_0 e^{L_1x -1 }}{L_1^2}&,  x\in \left[\frac{1}{L_1},\infty\right), \\
		&\frac{L_0x^2}{2} + \frac{L_0}{2L_1^2}&,  x\in [-\frac{1}{L_1},\frac{1}{L_1}], \\
		  &\frac{L_0 e^{-L_1 x-1}}{L_1^2}&,  x\in \left(-\infty,-\frac{1}{L_1}\right].
		\end{aligned}
		\right.
   \label{lowerbound_f1_restated}
\end{equation}

It is easy to verify that $f_1$ obeys Assumption \ref{assum: objective}.  Then, we set $w_1$ as the solution of $f_1(x)-\frac{L_0}{2L_1^2}=\Delta_1$, thus $f(\bw_1)-f^* = \Delta_1$ is satisfied. We then construct the noise oracle as $O_f(\bw,z) =z $, where $z\sim e^{-\frac{\sqrt{\vert z\vert }}{\sqrt[6]{\sigma_0^2/960}}}$. One can easily verify that $\operatorname{Var}(z) =\sigma_0^2$ and Assumption \ref{assum: noise} is meet.

 Now, we prove the following claim: starting any point $w_t$  and with any previous momentum $\bom_{t-1}$, one step of SGDM
\begin{equation*}
    \E[\Vert \nabla f(\bw_{t+1}) \Vert |\bw_t] =\infty.
\end{equation*}

Specifically, we have one step of SGDM gives 
\begin{align*}
    \bw_{t+1} =\bw_t -\eta(1-\beta) \nabla f(\bw_t) -\eta \beta \bom_{t-1} -\eta(1-\beta) z_t.
\end{align*}

    Therefore, we have 
    \begin{align*}
        \E [\vert \nabla f(w_{t+1}) \vert| \bw_t]  \ge &  \E \left[\vert \nabla f(w_{t+1}) \vert \mathds{1}_{w_{t+1} \ge \max\{\frac{1}{L_1}, \bw_t-\eta(1-\beta) \nabla f(\bw_t) -\eta \beta \bom_{t-1}\}}\right]
        \\
         \ge &  \E \left[\vert \nabla f(w_{t+1}) \vert \mathds{1}_{z_t \le \min\{\frac{\bw_t-\frac{1}{L_1}}{\eta(1-\beta)}- \nabla f(\bw_t) -\frac{\beta }{1-\beta}\bom_{t-1}, 0\}}\right]
        \\
        \ge & \frac{1}{2}\int^{\min\{\frac{\bw_t-\frac{1}{L_1}}{\eta(1-\beta)}- \nabla f(\bw_t) -\frac{\beta }{1-\beta}\bom_{t-1}, 0\}}_{-\infty}  \frac{L_0}{L_1} e^{L_1(\bw_t -\eta(1-\beta) \nabla f(\bw_t) -\eta \beta \bom_{t-1} -\eta(1-\beta) z)-1} e^{-\frac{\sqrt{-z }}{\sqrt[6]{\sigma_0^2/960}}} \mathrm{d} z.
    \end{align*}
    Since $\lim_{z\rightarrow -\infty} e^{L_1(\bw_t -\eta(1-\beta) \nabla f(\bw_t) -\eta \beta \bom_{t-1} -\eta (1-\beta)z)-1} e^{-\frac{\sqrt{-z }}{\sqrt[6]{\sigma_0^2/960}}} = \infty$ regardless of $\eta$, $\beta$, and $\bom_{t-1}$, we have $ \E [\vert \nabla f(w_{t+1}) \vert|\bw_t] = \infty$ based on the above inequalities. This means that an update from any point over this example will always lead to the divergence on expected gradient norm, thus we have $\forall t >1$,
    \begin{equation*}
       \min_{t\in [T]} \E \vert \nabla f(w_t) \vert =  \vert \nabla f(w_1) \vert.
    \end{equation*}
    The proof is completed.
\end{proof}

\section{Proofs for Section \ref{sec: reach_lower_bound}}
\subsection{Proof for Theorem \ref{thm: attain_lower_bound}}
\label{appen: tight_bound}
\begin{theorem}[Theorem \ref{thm: attain_lower_bound}, restated]
\label{thm: attain_lower_bound_appen}
    Let Assumption \ref{assum: objective} hold. Then, $\forall \beta_1 \ge 0$, if $\varepsilon\le \frac{1}{\Poly}$, with $\eta = (1-\bone)\frac{\sqrt{L_0(f(\bw_1)-f^*)}}{\sqrt{T}}$ and momentum hyperparameter $\btwo = 1-\eta^2\frac{(256\sigma_1^2L_1)^2}{1-\bone}$, we have that if $T \ge \Theta(\frac{ L_0\sigma_0^2 (f(\bw_1)-f^*)}{\varepsilon^4})$
    \begin{equation*}
        \E \min_{t\in [1,T]}\Vert \nabla f(\bw_t) \Vert \le \varepsilon.
    \end{equation*}
\end{theorem}

\begin{proof}
   Recall that $\bu_t \triangleq \frac{\bw_{t}-\frac{\bone}{\sqrt{\btwo}}\bw_{t-1}}{1-\frac{\bone}{\sqrt{\btwo}}}$. and the surrogate second-order momentum be defined as $\bnut_t \triangleq \btwo \bnu_{t-1}+(1-\btwo) \sigma_0^2$. Due to $\frac{\eta}{\sqrt{1-\btwo}} \le \frac{\sqrt{1-\bone}}{8L_1}$ and  following the similar routine as Theorem \ref{thm: stochastic  Adam}, one can easily verify that 
    \begin{equation*}
        \Vert \bu_{t}-\bw_t \Vert  \le \frac{1}{4L_1},
        \Vert \bu_{t+1}-\bw_t \Vert \le \frac{1}{4L_1}.
    \end{equation*}

    Therefore, if  Lemma \ref{lem: descent} can be applied with $\bw^1= \bw_t$, $\bw^2 = \bu_{t+1}$, and $\bw^3 = \bu_t$, we see the conditions of Lemma \ref{lem: descent} is satisfied, which after taking expectation gives
    \begin{align*}
        \E^{|\gF_t} f(\bu_{t+1})\le  f(\bu_t)+ \E^{|\gF_t}\langle \nabla f(\bw_t), \bu_{t+1}-\bu_t\rangle + \frac{1}{2}(L_0+L_1 \Vert \nabla f(\bw_t)\Vert) \E^{|\gF_t}(\Vert \bu_{t+1} -\bw_t\Vert+  \Vert \bu_t-\bw_t\Vert) \Vert\bu_{t+1}-\bu_t\Vert. 
    \end{align*}

    We call $\langle \nabla f(\bw_t), \bu_{t+1}-\bu_t\rangle$ the  first-order term and $ \frac{1}{2}(L_0+L_1 \Vert \nabla f(\bw_t)\Vert)(\Vert \bu_{t+1} -\bw_t\Vert+  \Vert \bu_t-\bw_t\Vert) \Vert\bu_{t+1}-\bu_t\Vert$ the second-order term, as they respectively correspond to the first-order and second-order Taylor's expansion. We then respectively bound these two terms as follows.

    \textbf{Analysis for the first-order term.} Similar to  bounding the first-order term in the proof of  Theorem \ref{thm: stochastic  Adam}, we have the following decomposition : \begin{align*}
    \bu_{t+1}-\bu_t=&\frac{\bw_{t+1}-\bw_t}{1-\frac{\bone}{\sqrt{\btwo}}}-\frac{\bone}{\sqrt{\btwo}} \frac{\bw_{t}-\bw_{t-1}}{1-\frac{\bone}{\sqrt{\btwo}}}
    \\
 =& -\eta \frac{1-\bone}{1-\frac{\bone}{\sqrt{\btwo}}} \frac{1}{\sqrt{\btwo\bnu_{t-1}}}  \bg_t
     -\frac{\eta}{1-\frac{\bone}{\sqrt{\btwo}}} \left(\frac{1}{\sqrt{\bnu_t}}-\frac{1}{\sqrt{\btwo\bnu_{t-1}}}\right)  \bom_t.
\end{align*}

According to the above decomposition, we have the first-order term can also be decomposed into
\small
\begin{align}
\nonumber
    &\dE^{|\gF_t} \left[ \left\langle \nabla f(\bw_t) , \bu_{t+1}-\bu_t \right\rangle \right]
    \\
\nonumber
    =& \frac{1-\bone}{1-\frac{\bone}{\sqrt{\btwo}}}\dE^{|\gF_t} \left[ \left\langle\bG_t, -\eta  \frac{1}{\sqrt{\btwo{\bnu}_{t-1}}}  \bg_t \right\rangle \right]+\dE^{|\gF_t} \left[ \left\langle\bG_t,-\frac{\eta}{1-\frac{\bone}{\sqrt{\btwo}}} \left(\frac{1}{\bnu_t}-\frac{1}{\sqrt{\btwo\bnu_{t-1}}}\right)  \bom_t\right\rangle \right].
\end{align}
\normalsize

As $\dE^{|\gF_t} \left[ \left\langle\bG_t, -\eta  \frac{1}{\sqrt{\btwo\bnu_{t-1}}}  \bg_t \right\rangle \right] = -\eta \frac{ \Vert \bG_t \Vert^2}{\sqrt{\btwo \bnu_{t-1}}}$, we have
\begin{equation*}
   \frac{1-\bone}{1-\frac{\bone}{\sqrt{\btwo}}} \dE^{|\gF_t} \left[ \left\langle\bG_t, -\eta  \frac{1}{\sqrt{\btwo\bnu_{t-1}}}  \bg_t \right\rangle \right] \le - \frac{ \Vert \bG_t \Vert^2}{\sqrt{\btwo \bnu_{t-1}}}.
\end{equation*}
We then  bound $\dE^{|\gF_t} \left[ \left\langle\bG_t,-\frac{\eta}{1-\frac{\bone}{\sqrt{\btwo}}} \left(\frac{1}{\bnu_t}-\frac{1}{\sqrt{\btwo\bnu_{t-1}}}\right)  \bom_t\right\rangle \right]$ as follows
\small
\begin{align}
\nonumber
    &\dE^{|\fil_t} \left[ \left\langle\bG_t,-\frac{\eta}{1-\frac{\bone}{\sqrt{\btwo}}} \left(\frac{1}{\sqrt{\bnu_t}}-\frac{1}{\sqrt{\btwo\bnu_{t-1}}}\right)    \bom_t\right\rangle \right] 
    \\
\nonumber
    =&  \dE^{|\fil_t} \left[ \left\langle\bG_t,-\frac{\eta}{1-\frac{\bone}{\sqrt{\btwo}}} \left(\frac{(1-\beta_2)\Vert \bg_t \Vert ^{    2}}{\sqrt{\bnu_t}\sqrt{\btwo\bnu_{t-1}}(\sqrt{\bnu_t}+\sqrt{\btwo\bnu_{t-1}})}\right)    \bom_t\right\rangle \right]
    \\
\nonumber
    \le &   \frac{\eta}{1-\frac{\bone}{\sqrt{\btwo}}}\dE^{|\fil_t} \left[ \Vert \bG_{t}\Vert  \left(\frac{(1-\beta_2)\Vert \bg_{t} \Vert^{ 2}}{\sqrt{\bnu_{t}}\sqrt{\btwo \bnu_{t-1}}(\sqrt{\bnu_{t}}+\sqrt{\btwo \bnu_{t-1}})}\right)\Vert  \bom_{t}\Vert  \right]
    \\
    \nonumber
    =&{  \frac{\eta}{1-\frac{\bone}{\sqrt{\btwo}}}\dE^{|\fil_t} \left[ \Vert \bG_{t}\Vert  \left(\frac{(1-\beta_2)\Vert \bg_{t}\Vert ^{ 2}}{\sqrt{\bnu_{t}}\sqrt{\btwo \bnu_{t-1}}(\sqrt{\bnu_{t}}+\sqrt{\btwo \bnu_{t-1}})}\right)\Vert  \bom_{t}\Vert  \right]}
.
\end{align}
\normalsize

 Due to Lemma \ref{lem: bounded_update}, the right-hand-side of the above inequality can be further bounded as 
\small
\begin{align}
\nonumber
&  \frac{\eta}{1-\frac{\beta_1}{\sqrt{\beta_2}}}\dE^{|\fil_t} \left[ \Vert \bG_{t}\Vert  \left(\frac{(1-\beta_2)\Vert \bg_{t} \Vert ^{ 2}}{\sqrt{\bnu_{t}}\sqrt{\btwo \bnu_{t-1}}(\sqrt{\bnu_{t}}+\sqrt{\btwo \bnu_{t-1}})}\right)\Vert  \bom_{t}\Vert  \right]
    \le
    \frac{\eta(1-\bone)}{\left(\sqrt{1-\frac{\bone}{\sqrt{\btwo}}}\right)^3}\dE^{|\fil_t} \left[ \Vert \bG_{t}\Vert  \left(\frac{\sqrt{1-\beta_2} \Vert \bg_{t} \Vert^2 }{\sqrt{\btwo \bnu_{t-1}}(\sqrt{\bnu_{t}}+\sqrt{\btwo \bnu_{t-1}})}\right)  \right]
    \\
    \nonumber
     \overset{(\circ)}{\le } &    \frac{\eta(1-\bone)}{\left(\sqrt{1-\frac{\bone}{\sqrt{\btwo}}}\right)^3} \frac{\Vert \bG_{t}\Vert}{\sqrt{\btwo \bnu_{t-1}} }\sqrt{\dE^{|\fil_t}  \Vert \bg_{t} \Vert ^2  }\sqrt{\dE^{|\fil_t} \frac{  \Vert \bg_{t} \Vert ^2}{(\sqrt{\bnu_{t}}+\sqrt{\btwo \bnu_{t-1}})^2}  }
     \overset{(\bullet)}{\le }  \frac{\eta(1-\bone)\sqrt{1-\btwo}}{\left(\sqrt{1-\frac{\bone}{\sqrt{\btwo}}}\right)^3} \frac{\Vert \bG_{t}\Vert}{\sqrt{\btwo \bnu_{t-1}} }\sqrt{\sigma_0^2+\sigma_1^2 \Vert\bG_{t} \Vert^2 }\sqrt{\dE^{|\fil_t} \frac{  \Vert \bg_{t} \Vert ^2}{(\sqrt{\bnu_{t}}+\sqrt{\btwo \bnu_{t-1}})^2}  }
     \\
     \nonumber
     \le &   \frac{\eta(1-\bone)\sqrt{1-\btwo}}{\left(\sqrt{1-\frac{\bone}{\sqrt{\btwo}}}\right)^3} \frac{\Vert \bG_{t}\Vert}{\sqrt{\btwo \bnu_{t-1}} }(\sigma_0+\sigma_1 \Vert \bG_{t}\Vert)\sqrt{\dE^{|\fil_t} \frac{  \Vert \bg_{t} \Vert ^2}{(\sqrt{\bnu_{t}}+\sqrt{\btwo \bnu_{t-1}})^2}  },
\end{align}
\normalsize
where  inequality $(\circ)$ is due to Holder's inequality, and inequality $(\bullet)$ is due to Assumption \ref{assum: noise}. Applying mean-value inequality respectively to  $  \frac{\eta(1-\bone)\sqrt{1-\btwo}}{ \left(\sqrt{1-\frac{\bone}{\sqrt{\btwo}}}\right)^3}\dE^{|\fil_t} \frac{\Vert \bG_t\Vert}{\sqrt{\btwo \bnu_{t-1}} }\sigma_0\sqrt{\dE^{|\fil_t} \frac{  \Vert \bg_{t} \Vert^2}{(\sqrt{\bnu_{t}}+\sqrt{\btwo \bnu_{t-1}})^2}  }$ and $  \frac{\eta(1-\bone)\sqrt{1-\btwo}}{\left(\sqrt{1-\frac{\bone}{\sqrt{\btwo}}}\right)^3}\dE^{|\fil_t} \frac{\Vert \bG_t\Vert}{\sqrt{\btwo \bnu_{t-1}} }\sigma_1 \Vert \bG_{t}\Vert\sqrt{\dE^{|\fil_t} \frac{  \Vert \bg_t \Vert^2}{(\sqrt{\bnu_{t}}+\sqrt{\btwo \bnu_{t-1}})^2}  }$ and due to $\bone\le \btwo$, we obtain that the right-hand-side of the above inequality can be bounded by 
\begin{align}
\nonumber
    &\frac{1}{16}  \eta \frac{1-\bone}{1-\frac{\bone}{\sqrt{\btwo}}}\frac{\Vert \bG_t\Vert^2}{\sqrt{\btwo \bnu_{t-1} }}+\frac{4\eta(1-\btwo)\sigma_0^2}{\left(1-\frac{\bone}{\sqrt{\btwo}}\right)^2\sqrt{\btwo \bnu_{t-1} } } \dE^{|\fil_t} \frac{  \Vert \bg_t \Vert^2}{(\sqrt{\bnu_{t}}+\sqrt{\btwo \bnu_{t-1}})^2}
    \\
\nonumber
    &+ \frac{1}{16}    \eta\frac{1-\bone}{1-\frac{\bone}{\sqrt{\btwo}}}\frac{\Vert \bG_t\Vert^2}{\sqrt{\btwo \bnu_{t-1}} }+4\eta \frac{(1-\btwo)(1-\bone)}{(1-\frac{\bone}{\sqrt{\btwo}})^2} \sigma_1^2\frac{\Vert \bG_t\Vert^2}{\sqrt{\btwo \bnu_{t-1}} }\dE^{|\fil_t}  \frac{  \Vert \bg_t \Vert^2}{(\sqrt{\bnu_{t}}+\sqrt{\btwo \bnu_{t-1}})^2}
    \\
    \nonumber
    \le & \frac{1}{8}  \eta \frac{\Vert \bG_t\Vert^2}{\sqrt{\btwo \bnu_{t-1} }}+\frac{8\eta (1-\btwo)\sigma_0^2 }{\left(1-{\bone}\right)^2} \dE^{|\fil_t} \frac{  \Vert \bg_t \Vert^2}{\sqrt{\btwo \bnu_{t-1}}(\sqrt{\bnu_{t}}+\sqrt{\btwo \bnu_{t-1}})^2}
    \\
    \nonumber
    &+ \frac{1}{8}    \eta\frac{\Vert \bG_t\Vert^2}{\sqrt{\btwo \bnu_{t-1}} }+16\eta \frac{(1-\btwo)}{(1-\bone)} \sigma_1^2\frac{\Vert \bG_t\Vert^2}{\sqrt{\btwo \bnu_{t-1}} }\dE^{|\fil_t}  \frac{  \Vert \bg_t \Vert^2}{(\sqrt{\bnu_{t}}+\sqrt{\btwo \bnu_{t-1}})^2}.
\end{align}
Meanwhile, we have
\begin{align*}
   &\left( \frac{1}{\sqrt{\btwo \bnu_{t-1}}}- \frac{1}{\sqrt{ \bnu_{t}}}\right) \Vert \bG_{t} \Vert^2
   \\
   =& \frac{\Vert \bG_{t} \Vert^2((1-\btwo)  \Vert \bg_t \Vert^2)}{\sqrt{\btwo \bnu_{t-1}}\sqrt{ \bnu_{t}} (\sqrt{\btwo \bnu_{t-1}}+\sqrt{ \bnu_{t}})} \ge \frac{\Vert \bG_{t} \Vert^2((1-\btwo)  \Vert \bg_t \Vert^2)}{\sqrt{\btwo \bnu_{t-1}} (\sqrt{\btwo \bnu_{t-1}}+\sqrt{ \bnu_{t}})^2},
\end{align*}
and
\begin{equation*}
    \frac{1}{\sqrt{\btwo \bnu_{t-1}}}- \frac{1}{\sqrt{ \bnu_{t}}} = \frac{(1-\btwo)  \Vert \bg_t \Vert^2}{\sqrt{\btwo \bnu_{t-1}}\sqrt{ \bnu_{t}} (\sqrt{\btwo \bnu_{t-1}}+\sqrt{ \bnu_{t}})} \ge \frac{(1-\btwo)  \Vert \bg_t \Vert^2}{\sqrt{\btwo \bnu_{t-1}}(\sqrt{\btwo \bnu_{t-1}}+\sqrt{ \bnu_{t}})^2}.
\end{equation*}
Combing the above two inequalities, we further obtain
\begin{align}
\nonumber
    &  \frac{\eta}{1-\beta_1}\dE^{|\fil_t} \left[ \Vert \bG_{t}\Vert  \left(\frac{(1-\beta_2)\bg_{t}^{ 2}}{\sqrt{\bnu_{t}}\sqrt{\btwo \bnu_{t-1}}(\sqrt{\bnu_{t}}+\sqrt{\btwo \bnu_{t-1}})}\right)\Vert  \bom_{t}\Vert  \right]
    \\
    \nonumber
    \le & \frac{1}{4}    \eta\frac{\Vert \bG_t\Vert^2}{\sqrt{\btwo \bnu_{t-1}} }+\frac{8\eta\sigma_0^2}{\left(1-\bone\right)^2} \dE^{|\fil_t} \left(\frac{1}{\sqrt{\btwo \bnu_{t-1}}}- \frac{1}{\sqrt{ \bnu_{t}}}\right)
    + \eta \frac{16}{(1-\bone) } \sigma_1^2 \dE^{|\fil_t}\left( \frac{1}{\sqrt{\btwo \bnu_{t-1}}}- \frac{1}{\sqrt{ \bnu_{t}}}\right) \Vert \bG_{t} \Vert^2.
\end{align}

Furthermore, due to Assumption \ref{assum: objective}, we have (we define $G_0\triangleq G_1$)
\begin{align*}
    \Vert \bG_{t+1} \Vert^2\le & \Vert \bG_{t}\Vert^2+2\Vert \bG_t\Vert \Vert \bG_{t+1}-\bG_{t}\Vert  + \Vert\bG_{t+1}-\bG_{t}\Vert ^2
    \\
    \le & \Vert \bG_{t}\Vert^2+2(L_0+L_1 \Vert \bG_t \Vert)\Vert \bG_t\Vert \Vert \bw_{t+1}-\bw_{t}\Vert  + 2(L_0^2+L_1^2\Vert \bG_t \Vert^2) \Vert \bw_{t+1}-\bw_{t}\Vert^2,
\end{align*}
which further leads to
\begin{align}
\nonumber
    &\frac{1}{\sqrt{ \bnu _{t}}} \Vert \bG_{t} \Vert^2
    \\
\nonumber
    \ge& \frac{1}{\sqrt{\bnu_{t}}} \left(\Vert \bG_{t+1}\Vert^2-2(L_0+L_1 \Vert \bG_t \Vert)\Vert \bG_t\Vert \Vert \bw_{t+1}-\bw_{t}\Vert  -2 (L_0^2+L_1^2\Vert \bG_t \Vert^2) \Vert \bw_{t+1}-\bw_{t}\Vert^2\right)
    \\
\nonumber
    \ge &  \frac{1}{\sqrt{\bnu_{t}}}\Vert \bG_{t+1}\Vert^2-\frac{1-\bone}{128\sigma_1^2} \frac{\Vert \bG_t\Vert^2}{\sqrt{\bnu_t}}-\frac{128L_0^2 \sigma_1^2}{(1-\bone)\sqrt{\bnu_t}} \Vert \bw_{t+1}-\bw_t \Vert^2 -2L_1\frac{\Vert \bG_t \Vert^2}{\sqrt{\bnu_t}} \Vert \bw_{t+1}-\bw_t \Vert -2\frac{L_0^2}{\sqrt{\bnu_t}} \Vert \bw_{t+1}-\bw_t \Vert^2
    \\
\nonumber
    &-\frac{2L_1^2 \Vert \bG_t \Vert^2 \Vert \bw_{t+1}-\bw_t \Vert^2}{\sqrt{\bnu_t}}
    \\
\nonumber
    \ge &  \frac{1}{\sqrt{\bnu_{t}}}\Vert \bG_{t+1}\Vert^2-\frac{1-\bone}{128\sigma_1^2} \frac{\Vert \bG_t\Vert^2}{\sqrt{\bnu_t}}-\frac{128L_0^2 \sigma_1^2}{(1-\bone)\sqrt{\bnu_t}} \Vert \bw_{t+1}-\bw_t \Vert^2 -\frac{1-\bone}{128\sigma_1^2} \frac{\Vert \bG_t\Vert^2}{\sqrt{\bnu_t}} -2\frac{L_0^2}{\sqrt{\bnu_t}} \Vert \bw_{t+1}-\bw_t \Vert^2
    \\
\label{eq: lower_bound_G_t}
    &-\frac{1-\bone}{128\sigma_1^2} \frac{\Vert \bG_t\Vert^2}{\sqrt{\bnu_t}},
\end{align}
where  the second inequality is due to Young's inequality, and the last inequality is due to $\Vert \bw_{t+1}-\bw_t \Vert \le \frac{\eta(1-\bone)}{\sqrt{1-\btwo}
\sqrt{1-\frac{\bone^2}{\btwo}}} \le \frac{\eta\sqrt{1-\bone}}{\sqrt{1-\btwo}
}\le \frac{1-\bone}{256\sigma_1^2L_1}$.

Applying the above inequality back to the estimation of $\frac{\eta}{1-\beta_1}\dE^{|\fil_t} \left[ \Vert \bG_{t}\Vert  \left(\frac{(1-\beta_2)\bg_{t}^{ 2}}{\sqrt{\bnu_{t}}\sqrt{\btwo \bnu_{t-1}}(\sqrt{\bnu_{t}}+\sqrt{\btwo \bnu_{t-1}})}\right)\Vert  \bom_{t}\Vert  \right]$ leads to that
\begin{align}
\nonumber
  &   \frac{\eta}{1-\bone}\dE^{|\fil_t} \left[ \Vert \bG_{t}\Vert  \left(\frac{(1-\beta_2)\bg_{t}^{ 2}}{\sqrt{\bnu_{t}}\sqrt{\btwo \bnu_{t-1}}(\sqrt{\bnu_{t}}+\sqrt{\btwo \bnu_{t-1}})}\right)\Vert  \bom_{t}\Vert  \right]
   \\
\nonumber
\le & \frac{5}{8}    \eta\frac{\Vert \bG_t\Vert^2}{\sqrt{\btwo \bnu_{t-1}} }+\frac{8\eta\sigma_0^2}{\left(1-\bone\right)^2} \dE^{|\fil_t} \left(\frac{1}{\sqrt{\btwo \bnu_{t-1}}}- \frac{1}{\sqrt{ \bnu_{t}}}\right)
    + \eta \frac{16}{(1-\bone) } \sigma_1^2 \dE^{|\fil_t}\left( \frac{\Vert \bG_{t} \Vert^2}{\sqrt{\btwo \bnu_{t-1}}}- \frac{\Vert \bG_{t+1} \Vert^2}{\sqrt{ \bnu_{t}}}\right)
    \\
\nonumber
    +& \frac{32768 L_0^2 \sigma_1^4 \eta^3 }{(1-\bone)(1-\btwo)} \E^{|\gF_t}   \left( \sum_{s=1}^t\frac{ \bone^{t-s} }{\sqrt[4]{\btwo^{3(t-s)}}}  \left(\frac{1}{\sqrt{\btwo \bnu_{s-1}}}- \frac{1}{\sqrt{ \bnu_{s}}}\right) \right).
\end{align}

All in all, we conclude that the first-order term can be bounded by
\small
\begin{align}
\nonumber
   \dE^{|\gF_t} \left[ \left\langle \nabla f(\bw_t) , \bu_{t+1}-\bu_t \right\rangle \right]\le  & - \frac{3}{8}    \eta\frac{\Vert \bG_t\Vert^2}{\sqrt{\btwo \bnu_{t-1}} }+\frac{8\eta\sigma_0^2}{\left(1-\bone\right)^2} \dE^{|\fil_t} \left(\frac{1}{\sqrt{\btwo \bnu_{t-1}}}- \frac{1}{\sqrt{ \bnu_{t}}}\right)
    + \eta \frac{16}{(1-\bone) } \sigma_1^2 \dE^{|\fil_t}\left( \frac{\Vert \bG_{t} \Vert^2}{\sqrt{\btwo \bnu_{t-1}}}- \frac{\Vert \bG_{t+1} \Vert^2}{\sqrt{ \bnu_{t}}}\right)
    \\
    \nonumber
    +& \frac{32768 L_0^2 \sigma_1^4 \eta^3 }{(1-\bone)(1-\btwo)} \E^{|\gF_t}   \left( \sum_{s=1}^t\frac{ \bone^{t-s} }{\sqrt[4]{\btwo^{3(t-s)}}}  \left(\frac{1}{\sqrt{\btwo \bnu_{s-1}}}- \frac{1}{\sqrt{ \bnu_{s}}}\right) \right).
\end{align}
\normalsize
    
    \textbf{Analysis for the second-order term.} To recall, the second order term is $ \frac{1}{2}(L_0+L_1 \Vert \nabla f(\bw_t)\Vert)(\Vert \bu_{t+1} -\bw_t\Vert+  \Vert \bu_t-\bw_t\Vert) \Vert\bu_{t+1}-\bu_t\Vert$. Before we start, we have the following expansion for $\bu_{t+1}-\bu_t$: (here the operations are all coordinate-wisely)
    \begin{align}
    \nonumber
         \bu_{t+1} -\bu_t  = &  \frac{\bw_{t+1}-\bw_t-\frac{\bone}{\sqrt{\btwo}}(\bw_t-\bw_{t-1})}{1-\frac{\bone}{\sqrt{\btwo}}} 
        \\
    \nonumber
        =& \frac{-\eta \frac{\bom_t}{\sqrt{\bnu_{t}}}+\eta\frac{\bone}{\sqrt{\btwo}}\frac{\bom_{t-1}}{\sqrt{\bnu_{t-1}}}}{1-\frac{\bone}{\sqrt{\btwo}}} = \frac{-\eta \frac{\bom_t}{\sqrt{\bnu_{t}}}+\eta\bone\frac{\bom_{t-1}}{\sqrt{\bnu_{t}}}-\eta\bone\frac{\bom_{t-1}}{\sqrt{\bnu_{t}}}+\eta\frac{\bone}{\sqrt{\btwo}}\frac{\bom_{t-1}}{\sqrt{\bnu_{t-1}}}}{1-\frac{\bone}{\sqrt{\btwo}}}
        \\
    \nonumber
        =&  \frac{-\eta \frac{(1-\bone)\bg_t}{\sqrt{\bnu_{t}}}+\eta\frac{\bone(1-\btwo) \Vert  \bg_t\Vert^2 }{\sqrt{\btwo}}\frac{\bom_{t-1}}{\sqrt{\bnu_{t-1}}\sqrt{\bnu_{t}}(\sqrt{\bnu_{t}}+\sqrt{\btwo\bnu_{t-1}})}}{1-\frac{\bone}{\sqrt{\btwo}}}
    \end{align} Then firstly, we have
    \small
    \begin{align*}
       &\frac{1}{2}L_0(\Vert \bu_{t+1} -\bw_t\Vert+  \Vert \bu_t-\bw_t\Vert) \Vert\bu_{t+1}-\bu_t\Vert
       \\
       \le & \frac{1}{2} L_0 \left(\Vert \bu_{t+1} -\bu_t\Vert^2+ \frac{1}{2} \Vert \bu_{t+1}-\bw_t \Vert^2 +\frac{1}{2}\Vert \bu_t -\bw_t \Vert^2\right)
       \\
       = & \frac{1}{2}L_0 \left(\left\Vert  \frac{-\eta \frac{(1-\bone)\bg_t}{\sqrt{\bnu_{t}}}+\eta\frac{\bone(1-\btwo) \Vert  \bg_t\Vert^2 }{\sqrt{\btwo}}\frac{\bom_{t-1}}{\sqrt{\bnu_{t-1}}\sqrt{\bnu_{t}}(\sqrt{\bnu_{t}}+\sqrt{\btwo\bnu_{t-1}})}}{1-\frac{\bone}{\sqrt{\btwo}}} \right\Vert^2 + \frac{1}{2} \left\Vert  \frac{\frac{\bone}{\sqrt{\btwo}} }{1- \frac{\bone}{\sqrt{\btwo}} }( \bw_t-\bw_{t-1} )\right\Vert^2+ \frac{1}{2} \left\Vert \frac{1 }{1- \frac{\bone}{\sqrt{\btwo}} } ( \bw_{t+1}-\bw_{t} ) \right\Vert^2 \right)
       \\
       \le & \frac{L_0\eta^2 }{2} \left(\left(\frac{1-\bone}{1-\frac{\bone}{\sqrt{\btwo}}}+ \frac{\bone(1-\bone)}{(\sqrt{\btwo}-\bone)\sqrt{1-\frac{\bone^2}{\btwo}}}\right)^2 \left\Vert \frac{\bg_t}{\sqrt{\bnu_t}} \right\Vert^2 +  \frac{1}{2}\left(\frac{\frac{\bone}{\sqrt{\btwo}} }{1- \frac{\bone}{\sqrt{\btwo}} }\right)^2\left\Vert \frac{\bom_{t-1}}{\sqrt{\bnu_{t-1}}} \right\Vert^2+\frac{1}{2}\left(\frac{1 }{1- \frac{\bone}{\sqrt{\btwo}} }\right)^2\left\Vert \frac{\bom_{t}}{\sqrt{\bnu_{t}}} \right\Vert^2 \right)
       \\
       \le  &\frac{L_0\eta^2 }{2} \left(2\left(\frac{1-\bone}{1-\frac{\bone}{\sqrt{\btwo}}}+ \frac{\bone(1-\bone)}{(\sqrt{\btwo}-\bone)\sqrt{1-\frac{\bone^2}{\btwo}}}\right)^2 \left\Vert \frac{\bg_t}{\sqrt{\bnu_t}} \right\Vert^2 +  \left(\frac{\frac{\bone}{\sqrt{\btwo}} }{1- \frac{\bone}{\sqrt{\btwo}} }\right)^2\left\Vert \frac{\bom_{t-1}}{\sqrt{\bnu_{t-1}}} \right\Vert^2\right)
       \\
       \le & \frac{L_0\eta^2 }{2} \left(\frac{32}{(1-\bone)^2}\left\Vert \frac{\bg_t}{\sqrt{\bnu_t}} \right\Vert^2 +  \frac{4}{(1-\bone)^2}\left\Vert \frac{\bom_{t-1}}{\sqrt{\bnu_{t-1}}} \right\Vert^2\right) .
    \end{align*}
    \normalsize
    Secondly, we have 
    \begin{align*}
         &\frac{1}{2}L_1\Vert \nabla f(\bw_t)\Vert (\Vert \bu_{t+1} -\bw_t\Vert+  \Vert \bu_t-\bw_t\Vert) \Vert\bu_{t+1}-\bu_t\Vert
         \\
         \le & \frac{1}{2}L_1\Vert \nabla f(\bw_t)\Vert (2\Vert \bu_{t+1} -\bw_t\Vert+  \Vert \bu_{t+1}-\bu_t\Vert)\left( \frac{\left\Vert\eta \frac{(1-\bone)\bg_t}{\sqrt{\bnu_{t}}}\right\Vert}{1-\frac{\bone}{\sqrt{\btwo}}}+\frac{\eta\frac{\bone(1-\btwo) \Vert  \bg_t\Vert^2 }{\sqrt{\btwo}}\frac{\Vert \bom_{t-1} \Vert}{\sqrt{\bnu_{t-1}}\sqrt{\bnu_{t}}(\sqrt{\bnu_{t}}+\sqrt{\btwo\bnu_{t-1}})}}{1-\frac{\bone}{\sqrt{\btwo}}}\right)
         \\
         \overset{(*)}{\le} & \frac{1}{2}L_1\Vert \nabla f(\bw_t)\Vert (2\Vert \bu_{t+1} -\bw_t\Vert+  \Vert \bu_{t+1}-\bu_t\Vert)\left( \frac{\left\Vert\eta \frac{(1-\bone)\bg_t}{\sqrt{\bnu_{t}}}\right\Vert}{1-\frac{\bone}{\sqrt{\btwo}}}+
         \frac{\eta\frac{\bone(1-\bone) }{\sqrt{\btwo}}\frac{\Vert \bg_{t} \Vert}{\sqrt{\bnu_{t}}}}{(1-\frac{\bone}{\sqrt{\btwo}})\sqrt{1-\frac{\bone^2}{\btwo}}}\right)
         \\
         =& \frac{L_1}{2}\eta \left(\frac{1-\bone}{1-\frac{\bone}{\sqrt{\btwo}}}+ \frac{\bone(1-\bone)}{(\sqrt{\btwo}-\bone)\sqrt{1-\frac{\bone^2}{\btwo}}}\right) \Vert \nabla f(\bw_t ) \Vert (2\Vert \bu_{t+1} -\bw_t\Vert+  \Vert \bu_t-\bu_{t+1}\Vert)  \frac{\Vert \bg_t \Vert}{\sqrt{\bnu_t}} 
         \\
         \overset{(\circ)}{=} &\frac{L_1}{2}\eta \left(\frac{1-\bone}{1-\frac{\bone}{\sqrt{\btwo}}}+ \frac{\bone(1-\bone)}{(\sqrt{\btwo}-\bone)\sqrt{1-\frac{\bone^2}{\btwo}}}\right) \Vert \bG_t \Vert \left(\Vert \bu_{t+1} -\bu_t \Vert +  2\frac{1 }{1- \frac{\bone}{\sqrt{\btwo}} } \eta \left\Vert \frac{\bom_t}{\sqrt{\bnu_t}}\right\Vert\right)  \frac{\Vert \bg_t \Vert}{\sqrt{\bnu_t}}
         \\
         \le &  \frac{2 L_1\eta}{\sqrt{1-\bone}} \Vert \bG_t \Vert \left(\Vert \bu_{t+1} -\bu_t \Vert +  4\frac{1 }{1- \bone } \eta \left\Vert \frac{\bom_t}{\sqrt{\bnu_t}}\right\Vert\right)  \frac{\Vert \bg_t \Vert}{\sqrt{\bnu_t}}.
    \end{align*}
    where inequality $(*)$ is due to that $\frac{\Vert \bom_{t-1} \Vert}{\sqrt{\bnu_{t-1}}} \le \frac{1-\bone}{\sqrt{1-\btwo} \sqrt{1-\frac{\bone^2}{\btwo}}}$, $\frac{\Vert \bg_t \Vert }{\sqrt{\bnu_t}} \le \frac{1}{\sqrt{1-\btwo}}$,  equation $(\circ)$ is due to $  \bu_{t}-\bw_t   = \frac{\frac{\bone}{\sqrt{\btwo}} }{1- \frac{\bone}{\sqrt{\btwo}} }( \bw_t-\bw_{t-1} ) $ and  $ \bu_{t+1}-\bw_t   = \frac{1 }{1- \frac{\bone}{\sqrt{\btwo}} } ( \bw_{t+1}-\bw_{t} )$, and the last inequality is due to $\bone \le \btwo$. As for the term $\Vert \bG_t \Vert \frac{ \Vert\bom_t\Vert}{\sqrt{\bnu_t}}  \frac{\Vert \bg_t \Vert }{\sqrt{\bnu_t}} $, we have
    \begin{align*}
        \E^{|\gF_t}\frac{\Vert \bG_t \Vert \Vert \bom_t \Vert  \Vert \bg_t \Vert  }{\bnu_t}  \le &    \E^{|\gF_t}\frac{\Vert \bG_t \Vert \Vert \bom_t \Vert  \Vert \bg_t \Vert  }{\sqrt{\bnu_t} \sqrt{\btwo \bnu_{t-1}}} 
        \le \frac{\Vert \bG_t \Vert }{\sqrt{\btwo \bnu_{t-1}}} \sqrt{\E^{|\gF_t} \Vert \bg_t \Vert^2}\sqrt{\E^{|\gF_t} \frac{\Vert \bom_t \Vert^2}{\bnu_t}}
        \\
        \le & \frac{\Vert \bG_t \Vert }{\sqrt{\btwo\bnu_{t-1}}} (\sigma_0+\sigma_1 \Vert \bG_t \Vert )\sqrt{\E^{|\gF_t} \frac{\Vert \bom_t \Vert^2}{\bnu_t}}
        \\
        \le &\frac{\sqrt{(1-\bone)^3}}{256\eta L_1}\frac{\Vert \bG_t \Vert^2 }{\sqrt{\btwo \bnu_{t-1}}} + \frac{64 \eta\sigma_0^2 L_1}{\sqrt{(1-\bone)^3}{\sqrt{\btwo \bnu_{t-1}}} } \E^{|\gF_t} \frac{\Vert \bom_t \Vert^2}{\bnu_t}+\frac{\sqrt{(1-\bone)^3}}{256\eta L_1}\frac{\Vert \bG_t \Vert^2 }{\sqrt{\btwo \bnu_{t-1}}}
        \\
        &+ \frac{64\eta\sigma_1^2 L_1 \Vert \bG_t \Vert^2}{\sqrt{(1-\bone)^3}{\sqrt{\btwo \bnu_{t-1}}}  } \E^{|\gF_t} \frac{\Vert \bom_t \Vert^2}{\bnu_t}.
    \end{align*}
    
    By applying Lemma \ref{lem: momentum_sum_2}, we further obtain
    \begin{align*}
        & \E^{|\gF_t}\frac{\Vert \bG_t \Vert \Vert \bom_t \Vert  \Vert \bg_t \Vert  }{\bnu_t}
         \\
         \le&  \frac{\sqrt{(1-\bone)^3}}{256\eta L_1}\frac{\Vert \bG_t \Vert^2 }{\sqrt{\btwo \bnu_{t-1}}} + \frac{64\eta\sigma_0^2 L_1}{\sqrt{(1-\bone)^3}{\sqrt{\btwo \bnu_{t-1}}} } \E^{|\gF_t} \frac{\Vert \bom_t \Vert^2}{\bnu_t}+\frac{\sqrt{(1-\bone)^3}}{256\eta L_1}\frac{\Vert \bG_t \Vert^2 }{\sqrt{\btwo \bnu_{t-1}}}
        \\
        &+ \frac{64\eta\sigma_1^2 L_1 }{\sqrt{(1-\bone)^3}  } \E^{|\gF_t} \left(4 (1-\bone) \left( \sum_{s=1}^t\frac{ \sqrt[8]{\bone^{t-s}} \Vert \bg_s\Vert^2\Vert \bG_s \Vert^2}{\sqrt{\btwo\bnu_{s-1}}\sqrt{\bnu_s^2}}\right) + 8 \frac{1-\bone}{1-\btwo} \frac{L_1^2}{L_0^2}\left(\sum_{s=1}^t\sqrt[8]{\bone^{t-s}}\left(\frac{1}{\sqrt{\btwo \bnu_{s-1}}} -\frac{1}{\sqrt{\bnu_s}}\right) \right)\right),
    \end{align*}
    which further indicates that
 \begin{align*}
        &\frac{8L_1 \eta^2}{(1-\bone)^{\frac{3}{2}}}\E^{|\gF_t}\Vert \bG_t \Vert  \frac{ \Vert \bg_t\Vert}{\sqrt{\bnu_t}}\frac{\Vert \bom_t \Vert }{\sqrt{\bnu_t}}
        \\
        \le &  
\frac{1}{16} \eta\frac{\Vert \bG_t \Vert^2 }{\sqrt{\btwo \bnu_{t-1}}} + \frac{8L_1 \eta^2}{(1-\bone)^{\frac{3}{2}}}\frac{64\eta\sigma_0^2 L_1}{\sqrt{(1-\bone)^3}{\sqrt{\btwo \bnu_{t-1}}} } \E^{|\gF_t} \frac{\Vert \bom_t \Vert^2}{\bnu_t}
        \\
        &+\frac{64\eta\sigma_1^2 L_1 }{\sqrt{(1-\bone)^3}  } \E^{|\gF_t}  \frac{32 L_1 \eta^2}{(1-\bone)^{\frac{1}{2}}}  \left( \sum_{s=1}^t\frac{ \sqrt[8]{\bone^{t-s}} \Vert \bg_s\Vert^2\Vert \bG_s \Vert^2}{\sqrt{\btwo\bnu_{s-1}}\sqrt{\bnu_s^2}}\right) 
        \\
        &+ \frac{64 L_1 \eta^2}{(1-\bone)^{\frac{3}{2}}}\frac{64\eta\sigma_1^2 L_1 }{\sqrt{(1-\bone)^3}  } \frac{1-\bone}{1-\btwo} \frac{L_1^2}{L_0^2}\left(\sum_{s=1}^t\sqrt[8]{\bone^{t-s}}\left(\frac{1}{\sqrt{\btwo \bnu_{s-1}}} -\frac{1}{\sqrt{\bnu_s}}\right) \right)
        \\
        \le & \frac{1}{16} \eta\frac{\Vert \bG_t \Vert^2 }{\sqrt{\btwo \bnu_{t-1}}} + \frac{8L_1 \eta^2}{(1-\bone)^{\frac{3}{2}}}\frac{64\eta\sigma_0^2 L_1}{\sqrt{(1-\bone)^3} } \E^{|\gF_t} 4 (1-\bone) \left( \sum_{s=1}^t\sqrt[4]{\bone^{t-s}} \frac{2}{1-\btwo} \left(\frac{1}{\sqrt{\btwo \bnu_{s-1}}}- \frac{1}{\sqrt{ \bnu_{s}}}\right) \right)
        \\
        &+\frac{64\eta\sigma_1^2 L_1 }{\sqrt{(1-\bone)^3}  } \E^{|\gF_t}  \frac{32 L_1 \eta^2}{(1-\bone)^{\frac{1}{2}}}  \left( \sum_{s=1}^t\frac{ \sqrt[8]{\bone^{t-s}} \Vert \bg_s\Vert^2\Vert \bG_s \Vert^2}{\sqrt{\btwo\bnu_{s-1}}\sqrt{\bnu_s^2}}\right) 
        \\
        &+ \frac{64 L_1 \eta^2}{(1-\bone)^{\frac{3}{2}}}\frac{64\eta\sigma_1^2 L_1 }{\sqrt{(1-\bone)^3}  } \frac{1-\bone}{1-\btwo} \frac{L_1^2}{L_0^2}\left(\sum_{s=1}^t\sqrt[8]{\bone^{t-s}}\left(\frac{1}{\sqrt{\btwo \bnu_{s-1}}} -\frac{1}{\sqrt{\bnu_s}}\right) \right).
    \end{align*}
    Here the last inequality is due to Lemma \ref{lem: momentum_sum_1}.

%     Then, following the similar reasoning as Eq. (\ref{eq: lower_bound_G_t}) and due to $\frac{\eta }{\sqrt{1-\btwo}} \le \frac{1-\bone}{64 L_1\sigma_1}$, we have
%     \begin{align*}
%         &\frac{8L_1 \eta^2}{(1-\bone)^{\frac{3}{2}}}\E^{|\gF_t}\Vert \bG_t \Vert  \frac{ \Vert \bg_t\Vert}{\sqrt{\bnu_t}}\frac{\Vert \bom_t \Vert }{\sqrt{\bnu_t}}
%         \\
%         \le &  
% \frac{1}{8}\frac{\Vert \bG_t \Vert^2 }{\sqrt{\btwo \bnu_{t-1}}} + \frac{8L_1 \eta^2}{(1-\bone)^{\frac{3}{2}}}\frac{32\eta\sigma_0^2 L_1}{\sqrt{(1-\bone)^3}{\sqrt{\btwo \bnu_{t-1}}} } \E^{|\gF_t} \frac{\Vert \bom_t \Vert^2}{\bnu_t}
%         \\
%         &+ \frac{128 L_1 \eta^2}{(1-\bone)^{\frac{3}{2}}}\frac{32\eta\sigma_1^2 L_1 }{\sqrt{(1-\bone)^3}  } \E^{|\gF_t} (1-\sqrt[8]{\bone})8  \left( \sum_{s=1}^t \sqrt[8]{\bone^{t-s}}\left(\frac{\Vert \bG_s \Vert ^2}{\sqrt{\btwo\bnu_{s-1}}}-\frac{\Vert \bG_{s+1} \Vert ^2}{\sqrt{\bnu_{s}}}\right)\right)
%         \\
%         &+ \E^{|\gF_t} (1-\sqrt[8]{\bone})  \left( \sum_{s=1}^t \sqrt[8]{\bone^{t-s}}\frac{1}{8}\eta\frac{\Vert \bG_t \Vert^2 }{\sqrt{\btwo \bnu_{t-1}}}\right)+ \E^{|\gF_t} (1-\sqrt[8]{\bone})  \left( \sum_{s=1}^t \sqrt[8]{\bone^{t-s}} \frac{32 L_0^2}{(1-\bone)^2} \frac{\Vert \bw_{s+1}-\bw_s \Vert^2}{\sqrt{\bnu_s}}\right)
%         \\
%         &+ \frac{64 L_1 \eta^2}{(1-\bone)^{\frac{3}{2}}}\frac{32\eta\sigma_1^2 L_1 }{\sqrt{(1-\bone)^3}  } \frac{1-\bone}{1-\btwo} \frac{L_1^2}{L_0^2}\left(\sum_{s=1}^t\sqrt[8]{\bone^{t-s}}\left(\frac{1}{\sqrt{\btwo \bnu_{s-1}}} -\frac{1}{\sqrt{\bnu_s}}\right) \right).
%     \end{align*}

Following similar reasoning, we have $\Vert \bu_{t+1} -\bu_t \Vert \le \frac{4 \eta}{\sqrt{1-\btwo}} \frac{\Vert \bg_t \Vert }{\sqrt{\bnu_t}}$, and 

    \begin{align*}
        & \E^{|\gF_t}\frac{\Vert \bG_t \Vert \Vert \bg_t \Vert  \Vert \bg_t \Vert  }{\bnu_t}
         \\
         \le&  \frac{\sqrt{(1-\bone)^3}}{256\eta L_1}\frac{\Vert \bG_t \Vert^2 }{\sqrt{\btwo \bnu_{t-1}}} + \frac{64 \eta\sigma_0^2 L_1}{\sqrt{(1-\bone)^3}{\sqrt{\btwo \bnu_{t-1}}} } \E^{|\gF_t} \frac{\Vert \bg_t \Vert^2}{\bnu_t}+\frac{\sqrt{(1-\bone)^3}}{256\eta L_1}\frac{\Vert \bG_t \Vert^2 }{\sqrt{\btwo \bnu_{t-1}}}
        \\
        &+ \frac{64\eta\sigma_1^2 L_1 \Vert \bG_t \Vert^2}{\sqrt{(1-\bone)^3}{\sqrt{\btwo \bnu_{t-1}}}  } \E^{|\gF_t} \frac{\Vert \bg_t \Vert^2}{\bnu_t}
        \\
        \le &\frac{\sqrt{(1-\bone)^3}}{256\eta L_1}\frac{\Vert \bG_t \Vert^2 }{\sqrt{\btwo \bnu_{t-1}}} +\frac{128\eta\sigma_1^2 L_1 }{\sqrt{(1-\bone)^3}(1-\btwo)  } \E^{|\gF_t} \left(\frac{1}{\sqrt{\btwo \bnu_{t-1}}}-\frac{1}{\sqrt{ \bnu_{t}}}\right)+\frac{\sqrt{(1-\bone)^3}}{256\eta L_1}\frac{\Vert \bG_t \Vert^2 }{\sqrt{\btwo \bnu_{t-1}}}
        \\
        &+ \frac{128\eta\sigma_1^2 L_1 \Vert \bG_t \Vert^2}{\sqrt{(1-\bone)^3}(1-\btwo)  } \E^{|\gF_t} \left(\frac{1}{\sqrt{\btwo \bnu_{t-1}}}-\frac{1}{\sqrt{ \bnu_{t}}}\right).
    \end{align*}
Then, following the similar routine as Eq. (\ref{eq: lower_bound_G_t}) and due to $\frac{\eta }{\sqrt{1-\btwo}} \le \frac{1-\bone}{128 L_1\sigma_1}$, we have
    \begin{align*}
        &\frac{2L_1 \eta}{(1-\bone)^{\frac{1}{2}}}\E^{|\gF_t}\Vert \bG_t \Vert \Vert \bu_{t+1}-\bu_t \Vert \frac{\Vert \bg_t \Vert }{\sqrt{\bnu_t}} \le \frac{8L_1 \eta^2}{(1-\bone)}\E^{|\gF_t}\Vert \bG_t \Vert  \frac{\Vert \bg_t \Vert^2 }{\bnu_t}
        \\
        \le &  
\frac{1}{16}\frac{\Vert \bG_t \Vert^2 }{\sqrt{\btwo \bnu_{t-1}}} + \frac{8L_1 \eta^2}{(1-\bone)^{\frac{3}{2}}}\frac{128\eta\sigma_1^2 L_1 }{\sqrt{(1-\bone)^3}(1-\btwo)  } \E^{|\gF_t} \left(\frac{1}{\sqrt{\btwo \bnu_{t-1}}}-\frac{1}{\sqrt{ \bnu_{t}}}\right)
        \\
        &+ \frac{8L_1 \eta^2}{(1-\bone)^{\frac{3}{2}}}\frac{128\eta\sigma_1^2 L_1 }{\sqrt{(1-\bone)^3}(1-\btwo)  } \E^{|\gF_t} \left(\frac{\Vert\bG_t \Vert^2}{\sqrt{\btwo \bnu_{t-1}}}-\frac{\Vert\bG_{t+1} \Vert^2}{\sqrt{ \bnu_{t}}}\right)
        \\
        &+ \frac{1}{16}\eta  \frac{\Vert \bG_t \Vert^2 }{\sqrt{\btwo \bnu_{t-1}}}+ \eta^3  \frac{64 L_0^2}{(1-\bone)^2}\E^{|\gF_t}  4 (1-\bone) \left( \sum_{s=1}^t\sqrt[4]{\bone^{t-s}} \frac{2}{1-\btwo} \left(\frac{1}{\sqrt{\btwo \bnu_{s-1}}}- \frac{1}{\sqrt{ \bnu_{s}}}\right) \right)
        .
    \end{align*}

    Putting all the estimations together, we have that the second-order term can be bounded by (note here due to the complexity of coefficients, we use $\Poly(L_0,L_1,\sigma_0,\sigma_1, \frac{1}{1-\bone})$ to denote the polynomial function of $L_0,L_1,\sigma_0,\sigma_1$, $\frac{1}{1-\bone}$)

    \begin{align}
    \nonumber
        &\E^{|\gF_t}\frac{1}{2}(L_0+L_1 \Vert \nabla f(\bw_t)\Vert)(\Vert \bu_{t+1} -\bw_t\Vert+  \Vert \bu_t-\bw_t\Vert) \Vert\bu_{t+1}-\bu_t\Vert
    \\
    \nonumber
    \le & \frac{L_0\eta^2 }{2} \left(\frac{32}{(1-\bone)^2}\left\Vert \frac{\bg_t}{\sqrt{\bnu_t}} \right\Vert^2 +  \frac{4}{(1-\bone)^2}\left\Vert \frac{\bom_{t-1}}{\sqrt{\bnu_{t-1}}} \right\Vert^2\right)+  \frac{3}{16}  \eta  \frac{\Vert \bG_t \Vert^2 }{\sqrt{\btwo \bnu_{t-1}}} 
    \\
    \nonumber
    &+ \frac{\eta^3}{1-\btwo} \Poly \E^{|\gF_t}  \left( \sum_{s=1}^t\sqrt[8]{\bone^{t-s}} \left(\frac{1}{\sqrt{\btwo \bnu_{s-1}}}- \frac{1}{\sqrt{ \bnu_{s}}}\right) \right)
    \\
    \nonumber
     &+\frac{\eta^3}{1-\btwo} \Poly \E^{|\gF_t} \left(\frac{\Vert\bG_t \Vert^2}{\sqrt{\btwo \bnu_{t-1}}}-\frac{\Vert\bG_{t+1} \Vert^2}{\sqrt{ \bnu_{t}}}\right)
     \\
     \nonumber & + \frac{64\eta\sigma_1^2 L_1 }{\sqrt{(1-\bone)^3}  } \E^{|\gF_t}  \frac{32 L_1 \eta^2}{(1-\bone)^{\frac{1}{2}}}  \left( \sum_{s=1}^t\frac{ \sqrt[8]{\bone^{t-s}} \Vert \bg_s\Vert^2\Vert \bG_s \Vert^2}{\sqrt{\btwo\bnu_{s-1}}\sqrt{\bnu_s^2}}\right) 
    \end{align}

    Here in the second inequality we use $\btwo\ge \bone$, and in the last inequality we use $\frac{\eta}{\sqrt{1-\btwo}} =\frac{\sqrt{1-\frac{\bone^2}{\btwo}}(1- \frac{\bone}{\sqrt{\btwo}})^2}{1024\sigma_1^2(L_1+L_0)(1-\bone)}$.

Applying the estimations of the first-order term  and the  second-order term back into the descent lemma, we derive that 
\begin{align}
\nonumber
        \E^{|\gF_t} f(\bu_{t+1})\le & f(\bu_t) - \frac{3}{16}    \eta   \E^{|\gF_t}\frac{\Vert \bG_t\Vert^2}{\sqrt{\btwo \bnu_{t-1}} }+ \frac{L_0\eta^2 }{2} \left(\frac{32}{(1-\bone)^2}\left\Vert \frac{\bg_t}{\sqrt{\bnu_t}} \right\Vert^2 +  \frac{4}{(1-\bone)^2}\left\Vert \frac{\bom_{t-1}}{\sqrt{\bnu_{t-1}}} \right\Vert^2\right)  
    \\\nonumber
    &+ \frac{\eta^3}{1-\btwo} \Poly \E^{|\gF_t}  \left( \sum_{s=1}^t\sqrt[8]{\bone^{t-s}} \left(\frac{1}{\sqrt{\btwo \bnu_{s-1}}}- \frac{1}{\sqrt{ \bnu_{s}}}\right) \right)
    \\\nonumber
     &+\left(\frac{\eta^3}{1-\btwo}+\eta\right) \Poly \E^{|\gF_t} \left(\frac{\Vert\bG_t \Vert^2}{\sqrt{\btwo \bnu_{t-1}}}-\frac{\Vert\bG_{t+1} \Vert^2}{\sqrt{ \bnu_{t}}}\right)
     \\\nonumber
     &+ \eta \Poly \dE^{|\fil_t}\left( \frac{\Vert \bG_{t} \Vert^2}{\sqrt{\btwo \bnu_{t-1}}}- \frac{\Vert \bG_{t+1} \Vert^2}{\sqrt{ \bnu_{t}}}\right)
     \\
     \label{eq: martingale_evidence}
    & + \frac{64\eta\sigma_1^2 L_1 }{\sqrt{(1-\bone)^3}  } \E^{|\gF_t}  \frac{32 L_1 \eta^2}{(1-\bone)^{\frac{1}{2}}}  \left( \sum_{s=1}^t\frac{ \sqrt[8]{\bone^{t-s}} \Vert \bg_s\Vert^2\Vert \bG_s \Vert^2}{\sqrt{\btwo\bnu_{s-1}}\sqrt{\bnu_s^2}}\right). 
    \end{align}

Constructing stopping time as $\tau \triangleq \min\{t: \Vert \nabla f(\bw_{t+1}) \Vert \le \frac{\sqrt[4]{L_0\sigma_0^2(f(\bw_1)-f^*)}}{\sqrt[4]{T}}\} \wedge T $. Then, denote 
\begin{align*}
    x_t \triangleq& f(\bu_t)- f(\bu_{t+1}) - \frac{3}{16}    \eta\dE\frac{\Vert \bG_t\Vert^2}{\sqrt{\btwo \bnu_{t-1}} }+ \frac{L_0\eta^2 }{2} \left(\frac{32}{(1-\bone)^2}\left\Vert \frac{\bg_t}{\sqrt{\bnu_t}} \right\Vert^2 +  \frac{4}{(1-\bone)^2}\left\Vert \frac{\bom_{t-1}}{\sqrt{\bnu_{t-1}}} \right\Vert^2\right)  
    \\
    &+ \frac{\eta^3}{1-\btwo} \Poly  \left( \sum_{s=1}^t\sqrt[8]{\bone^{t-s}} \left(\frac{1}{\sqrt{\btwo \bnu_{s-1}}}- \frac{1}{\sqrt{ \bnu_{s}}}\right) \right)
    \\
     &+\left(\frac{\eta^3}{1-\btwo}+\eta\right) \Poly \left(\frac{\Vert\bG_t \Vert^2}{\sqrt{\btwo \bnu_{t-1}}}-\frac{\Vert\bG_{t+1} \Vert^2}{\sqrt{ \bnu_{t}}}\right)
     \\
    & + \frac{64\eta\sigma_1^2 L_1 }{\sqrt{(1-\bone)^3}  }   \frac{32 L_1 \eta^2}{(1-\bone)^{\frac{1}{2}}}  \left( \sum_{s=1}^t\frac{ \sqrt[8]{\bone^{t-s}} \Vert \bg_s\Vert^2\Vert \bG_s \Vert^2}{\sqrt{\btwo\bnu_{s-1}}\sqrt{\bnu_s^2}}\right),
\end{align*}
and due to Eq. (\ref{eq: martingale_evidence}), we have $\E^{|\gF_t } x_t  \ge 0$, and thus  $S_t\triangleq \sum_{s=1}^t x_s$ ($S_0 =0 $) is a submartingale with respect to $\{\gF_t\}_t$. Also, as $\tau$ is a bounding stopping theorem, by optional stopping time, we obtain that $\E S_{\tau} \ge 0$, which gives 
\begin{align*}
   \frac{3}{16} \eta \E \sum_{t=1}^\tau  \frac{\Vert \bG_t \Vert^2}{\sqrt{\btwo\bnu_{t-1}}} \le& f(\bu_1)- \E f(\bu_{\tau+1})+ \frac{L_0\eta^2 }{2}  \E \sum_{t=1}^\tau\left(\frac{32}{(1-\bone)^2}\left\Vert \frac{\bg_t}{\sqrt{\bnu_t}} \right\Vert^2 +  \frac{4}{(1-\bone)^2}\left\Vert \frac{\bom_{t-1}}{\sqrt{\bnu_{t-1}}} \right\Vert^2\right)  
    \\
    &+ \frac{\eta^3}{1-\btwo} \Poly   \E \sum_{t=1}^\tau\left( \sum_{s=1}^t\sqrt[8]{\bone^{t-s}} \left(\frac{1}{\sqrt{\btwo \bnu_{s-1}}}- \frac{1}{\sqrt{ \bnu_{s}}}\right) \right)
    \\
     &+ \E \sum_{t=1}^\tau\left(\frac{\eta^3}{1-\btwo}+\eta\right) \Poly \left(\frac{\Vert\bG_t \Vert^2}{\sqrt{\btwo \bnu_{t-1}}}-\frac{\Vert\bG_{t+1} \Vert^2}{\sqrt{ \bnu_{t}}}\right)
     \\
    & +  \E \sum_{t=1}^\tau\frac{64\eta\sigma_1^2 L_1 }{\sqrt{(1-\bone)^3}  }   \frac{32 L_1 \eta^2}{(1-\bone)^{\frac{1}{2}}}  \left( \sum_{s=1}^t\frac{ \sqrt[8]{\bone^{t-s}} \Vert \bg_s\Vert^2\Vert \bG_s \Vert^2}{\sqrt{\btwo\bnu_{s-1}}\sqrt{\bnu_s^2}}\right)
    \\
   \le & f(\bu_1)- \E f(\bu_{\tau+1})+ \frac{L_0\eta^2 }{2}  \E \sum_{t=1}^\tau\left(\frac{32}{(1-\bone)^2}\left\Vert \frac{\bg_t}{\sqrt{\bnu_t}} \right\Vert^2 +  \frac{4}{(1-\bone)^2}\left\Vert \frac{\bom_{t-1}}{\sqrt{\bnu_{t-1}}} \right\Vert^2\right)  
    \\
    &+ \frac{\eta^3}{1-\btwo} \Poly   \E \sum_{t=1}^\tau \left(\frac{1}{\sqrt{\btwo \bnu_{t-1}}}- \frac{1}{\sqrt{ \bnu_{t}}}\right)
    \\
     &+ \E \sum_{t=1}^\tau\left(\frac{\eta^3}{1-\btwo}+\eta\right) \Poly \left(\frac{\Vert\bG_t \Vert^2}{\sqrt{\btwo \bnu_{t-1}}}-\frac{\Vert\bG_{t+1} \Vert^2}{\sqrt{ \bnu_{t}}}\right)
     \\
    & +  \E \sum_{t=1}^\tau\frac{64\eta\sigma_1^2 L_1 }{\sqrt{(1-\bone)^3}  }   \frac{256 L_1 \eta^2}{(1-\bone)^{\frac{3}{2}}}  \left( \frac{  \Vert \bg_t\Vert^2\Vert \bG_t \Vert^2}{\sqrt{\btwo\bnu_{t-1}}\sqrt{\bnu_t^2}}\right)
    \\
    \le& f(\bu_1)- \E f(\bu_{\tau+1})+ \frac{L_0\eta^2 }{2}  \E \sum_{t=1}^\tau\left(\frac{32}{(1-\bone)^2}\left\Vert \frac{\bg_t}{\sqrt{\bnu_t}} \right\Vert^2 +  \frac{4}{(1-\bone)^2}\left\Vert \frac{\bom_{t-1}}{\sqrt{\bnu_{t-1}}} \right\Vert^2\right)  
    \\
    &+ \frac{\eta^3}{1-\btwo} \Poly   \E \sum_{t=1}^\tau \left(\frac{1}{\sqrt{\btwo \bnu_{t-1}}}- \frac{1}{\sqrt{ \bnu_{t}}}\right)
    \\
     &+ \E \sum_{t=1}^\tau\left(\frac{\eta^3}{1-\btwo}+\eta\right) \Poly \left(\frac{\Vert\bG_t \Vert^2}{\sqrt{\btwo \bnu_{t-1}}}-\frac{\Vert\bG_{t+1} \Vert^2}{\sqrt{ \bnu_{t}}}\right)
    \\
   &+  \frac{1}{16} \eta \E \sum_{t=1}^\tau  \frac{\Vert \bG_t \Vert^2}{\sqrt{\btwo\bnu_{t-1}}},
\end{align*}
where the last inequality is because due to $\frac{\eta }{\sqrt{1-\btwo}} \le \frac{1-\bone}{128 L_1\sigma_1}$, following the similar reasoning of Eq. (\ref{eq: lower_bound_G_t}), we have 
\begin{align*}
    \frac{64\eta\sigma_1^2 L_1 }{\sqrt{(1-\bone)^3}  }   \frac{256 L_1 \eta^2}{(1-\bone)^{\frac{3}{2}}}  \left( \frac{  \Vert \bg_t\Vert^2\Vert \bG_t \Vert^2}{\sqrt{\btwo\bnu_{t-1}}\sqrt{\bnu_t^2}}\right) \le &\frac{128 \eta\sigma_1^2 L_1 }{\sqrt{(1-\bone)^3}  }   \frac{256 L_1 \eta^2}{(1-\btwo)(1-\bone)^{\frac{3}{2}}}   \left(\frac{\Vert\bG_t \Vert^2}{\sqrt{\btwo \bnu_{t-1}}}-\frac{\Vert\bG_{t+1} \Vert^2}{\sqrt{ \bnu_{t}}}\right)
    \\
     &+ \frac{\eta^3}{1-\btwo} \Poly   \sum_{s=1}^t\sqrt[8]{\bone^{t-s}} \left(\frac{1}{\sqrt{\btwo \bnu_{t-1}}}- \frac{1}{\sqrt{ \bnu_{t}}}\right)
     \\
     &+  \frac{1}{16} \eta \E \sum_{t=1}^\tau  \frac{\Vert \bG_t \Vert^2}{\sqrt{\btwo\bnu_{t-1}}}.
\end{align*}

By rearranging the inequality and due to Lemma \ref{lem: sum_momentum}, we obtain
\begin{align*}
   &\frac{1}{8} \eta \E \sum_{t=1}^\tau  \frac{\Vert \bG_t \Vert^2}{\sqrt{\btwo\bnu_{t-1}}} 
   \\
   \le & f(\bu_1)- \E f(\bu_{\tau+1})+\frac{ 64L_0\eta^2}{(1-\btwo)(1-\bone)^2}  \E  \left(\ln \frac{\bnu_\tau}{\bnu_0} -\tau\ln \btwo \right)
    \\
    &+ \frac{\eta^3}{1-\btwo} \Poly \frac{1}{\sqrt{\btwo \bnu_0}}+ \eta^3 \Poly  \E \sum_{t=1}^{\tau-1}  \frac{1}{\sqrt{ \btwo \bnu_{t}}}
    \\
     &+\left(\frac{\eta^3}{1-\btwo}+\eta\right) \Poly \left(\frac{\Vert\bG_1 \Vert^2}{\sqrt{\btwo \bnu_{0}}}\right)
     \\
     &+\E \sum_{t=1}^\tau\left(\eta^3+\eta (1-\btwo)\right) \Poly \left(\frac{\Vert\bG_t \Vert^2}{\sqrt{\btwo \bnu_{t-1}}}\right),
\end{align*}

Furthermore, as $T \ge \Poly$, $\eta^3= \frac{\eta}{T} \Poly$,  $\eta(1-\btwo)= \frac{\eta}{T} \Poly$, and $\Vert \bG_t\Vert  \ge \frac{\Poly}{\sqrt[4]{T}}$, we have
\begin{gather*}
    \eta^3 \Poly  \E \sum_{t=1}^{\tau-1}  \frac{1}{\sqrt{ \btwo \bnu_{t}}} \le \frac{1}{32}  \eta \E \sum_{t=1}^\tau  \frac{\Vert \bG_t \Vert^2}{\sqrt{\btwo\bnu_{t-1}}} ,
    \\
    \E \sum_{t=1}^\tau\left(\eta^3+\eta (1-\btwo)\right) \Poly \left(\frac{\Vert\bG_t \Vert^2}{\sqrt{\btwo \bnu_{t-1}}}\right) \le  \frac{1}{32}  \eta \E \sum_{t=1}^\tau  \frac{\Vert \bG_t \Vert^2}{\sqrt{\btwo\bnu_{t-1}}},
\end{gather*}
which thus leads to
\begin{align}
\nonumber
 &  \frac{1}{16} \eta \E \sum_{t=1}^\tau  \frac{\Vert \bG_t \Vert^2}{\sqrt{\btwo\bnu_{t-1}}} \le  f(\bu_1)- \E f(\bu_{\tau+1})+\frac{ 64 L_0\eta^2}{(1-\btwo)(1-\bone)^2}   \E \left(\ln \frac{\bnu_\tau}{\bnu_0} -\tau\ln \btwo \right)
    \\
    \label{eq: stopping_time}
     & + \frac{\eta^3}{1-\btwo} \Poly \frac{1}{\sqrt{\btwo \bnu_0}}+\left(\frac{\eta^3}{1-\btwo}+\eta\right) \Poly \left(\frac{\Vert\bG_1 \Vert^2}{\sqrt{\btwo \bnu_{0}}}\right)
    .
\end{align}
Similar to the proof of Theorem \ref{thm: stochastic  Adam}, we  then transfer the above bound to the bound of $\sum_{t=1}^{\tau} \Vert \bG_t \Vert$ by two rounds of divide-and-conquer. In the first round, we will bound $\E \ln \bnu_{\tau}$. To start with, we have that
\begin{align*}
   & \frac{\Vert \bG_{t}\Vert^2}{\sqrt{\btwo \bnu_{t-1}}}\mathds{1}_{\Vert G_{t}\Vert \ge \frac{\sigma_0}{\sigma_1}}\ge \frac{\frac{1}{2\sigma_1^2}\mathbb{E}^{|\gF_t}\Vert \bg_{t}\Vert^2}{\sqrt{\btwo \bnu_{t-1}}}\mathds{1}_{\Vert G_{t}\Vert \ge \frac{\sigma_0}{\sigma_1}}
    % \\
    % \ge &\frac{1}{2\sigma_1^2}\mathbb{E}^{|\gF_t} \frac{ \btwo^{T-t}\Vert \bg_{t}\Vert^2}{\sqrt{\bnu_T+(1-\btwo) \sigma_0^2}}\mathds{1}_{\Vert G_{t}\Vert \ge \frac{\sigma_0}{\sigma_1}},
\end{align*}
% where the last inequality is due to that
% \begin{align}
%     \btwo\bnu_{t-1}+(1-\btwo) \sigma_0^2\le  \btwo^{t-T}\bnu_{T}+(1-\btwo) \sigma_0^2
%    \label{eq: estimation_specific}
%    \le (\bnu_T+(1-\btwo) \sigma_0^2) \btwo^{2(t-T)}.
% \end{align}

Furthermore, we have 
\begin{align}
\nonumber
   & \E \frac{\frac{ \bnu_{0}}{1-\btwo}}{\sqrt{\sum_{t=1}^\tau \Vert \bg_t \Vert^2 +\frac{\bnu_0}{1-\btwo}}}+\mathbb{E}  \sum_{t=1}^{\tau}\frac{ \Vert \bg_{t}\Vert^2}{\sqrt{\sum_{t=1}^\tau \Vert \bg_t \Vert^2 +\frac{\bnu_0}{1-\btwo}}}\mathds{1}_{\Vert \bG_{t}\Vert  < \frac{\sigma_0}{\sigma_1}}
    \\
\nonumber
    \le &\E \frac{\frac{ \bnu_{0}}{1-\btwo}}{\sqrt{\sum_{t=1}^\tau \Vert \bg_t \Vert^2\mathds{1}_{\Vert \bG_{t}\Vert  < \frac{\sigma_0}{\sigma_1}} +\frac{\bnu_0}{1-\btwo}}}+\mathbb{E} \sum_{t=1}^\tau\frac{\Vert \bg_{t}\Vert^2}{\sqrt{\sum_{t=1}^\tau \Vert \bg_t \Vert^2\mathds{1}_{\Vert \bG_{t}\Vert  < \frac{\sigma_0}{\sigma_1}} +\frac{\bnu_0}{1-\btwo}}}\mathds{1}_{\Vert \bG_{t}\Vert  < \frac{\sigma_0}{\sigma_1}}
    \\
    =& \dE\sqrt{\sum_{t=1}^\tau \Vert \bg_t \Vert^2\mathds{1}_{\Vert \bG_{t}\Vert  < \frac{\sigma_0}{\sigma_1}} +\frac{\bnu_0}{1-\btwo}}
    \nonumber
    \le \sqrt{\E\sum_{t=1}^\tau \Vert \bg_t \Vert^2\mathds{1}_{\Vert \bG_{t}\Vert  < \frac{\sigma_0}{\sigma_1}} +\frac{ \bnu_{0}}{1-\btwo}}
    \\
    \nonumber
    \overset{(\star)}{ \le }&\sqrt{2\sigma_0^2\E \tau  +\frac{ \bnu_{0}}{1-\btwo}}\le \sqrt{2\sigma_0^2T +\frac{ \bnu_{0}}{1-\btwo}},
\end{align}
where inequality $(\star)$ is due to $E^{|\gF_t} \Vert \bg_t \Vert^2\mathds{1}_{\Vert \bG_{t}\Vert  < \frac{\sigma_0}{\sigma_1}} \le 2\sigma_0^2$ and optimal stopping theorem.

Conclusively, we obtain 
\begin{align*}
&\dE{\sqrt{\sum_{t=1}^\tau \Vert \bg_t \Vert^2 +\frac{\bnu_0}{1-\btwo}}}
\\
    =&\left(\E \frac{\frac{ \bnu_{0}}{1-\btwo}}{\sqrt{\sum_{t=1}^\tau \Vert \bg_t \Vert^2 +\frac{\bnu_0}{1-\btwo}}}+\mathbb{E} \sum_{t=1}^\tau \frac{ \Vert \bg_{t}\Vert^2}{\sqrt{\sum_{t=1}^\tau \Vert \bg_t \Vert^2 +\frac{\bnu_0}{1-\btwo}}}\mathds{1}_{\Vert \bG_{t}\Vert  < \frac{\sigma_0}{\sigma_1}}\right.
    \\
    &+\left. \mathbb{E} \sum_{t=1}^\tau\frac{ \Vert \bg_{t}\Vert^2}{\sqrt{\sum_{t=1}^\tau \Vert \bg_t \Vert^2 +\frac{\bnu_0}{1-\btwo}}}\mathds{1}_{\Vert \bG_{t}\Vert  \ge \frac{\sigma_0}{\sigma_1}}\right)
    \\
    \le & \sqrt{\frac{\bnu_{0}}{1-\btwo}  +2\sigma_0^2T}+2\sqrt{1-\btwo}\dE\sum_{t=1}^\tau\frac{\Vert \bg_{t}\Vert^2}{\sqrt{\btwo \bnu_{t-1}}}\mathds{1}_{\Vert \bG_{t}\Vert\ge \frac{\sigma_0}{\sigma_1}}
    \\
 \overset{(\circ)}{\le} & \sqrt{ \frac{\bnu_{0}}{1-\btwo} +2\sigma_0^2T}+2\sigma_1^2 \sqrt{1-\btwo}\dE\sum_{t=1}^\tau\frac{\Vert \bG_{t}\Vert^2}{\sqrt{\btwo \bnu_{t-1}}},
\end{align*}
where inequality $(\circ)$ is due to optimal stopping theorem.

Then by substituting $\dE\sum_{t=1}^\tau\frac{\Vert \bG_{t}\Vert^2}{\sqrt{\btwo \bnu_{t-1}}}$ we obtain that
\small
\begin{align*}
   &\dE{\sqrt{\sum_{t=1}^\tau \Vert \bg_t \Vert^2 +\frac{\bnu_0}{1-\btwo}}}
\\
    \le & \sqrt{  \frac{\bnu_{0}}{1-\btwo} +2\sigma_0^2T }+\frac{32\sigma_1^2 \sqrt{1-\btwo}}{\eta} \frac{1}{16}\eta \dE\sum_{t=1}^\tau\frac{\Vert \bG_{t}\Vert^2}{\sqrt{\btwo \bnu_{t-1}}}
    \\
    \le &  \sqrt{  \frac{\bnu_{0}}{1-\btwo} +2\sigma_0^2T }+\frac{32\sqrt{1-\btwo}\sigma_1^2}{\eta} \left(f(\bu_1)- f^*+\frac{ 64 L_0\eta^2}{(1-\btwo)(1-\bone)^2}   \E \left(\ln \frac{\bnu_\tau}{\bnu_0} -T\ln \btwo \right)
    \right.
    \\
     &+\left.+ \frac{\eta^3}{1-\btwo} \Poly \frac{1}{\sqrt{\btwo \bnu_0}}\left(\frac{\eta^3}{1-\btwo}+\eta\right) \Poly \left(\frac{\Vert\bG_1 \Vert^2}{\sqrt{\btwo \bnu_{0}}}\right)\right)
    \\
  \le &   \sqrt{  \frac{\bnu_{0}}{1-\btwo} +2\sigma_0^2T }+\frac{32\sqrt{1-\btwo}\sigma_1^2}{\eta} \left(f(\bw_1)- f^*+\frac{ 128 L_0\eta^2}{(1-\btwo)(1-\bone)^2}   \left(\ln \frac{  \E \sqrt{1-\btwo}\sqrt{\sum_{t=1}^\tau \Vert \bg_t \Vert^2 +\frac{\bnu_0}{1-\btwo}}}{\sqrt{\bnu_0}} -T\ln \btwo \right)\right.
 \\
     &   \left. + \frac{\eta^3}{1-\btwo} \Poly \frac{1}{\sqrt{\btwo \bnu_0}}
   +\left(\frac{\eta^3}{1-\btwo}+\eta\right) \Poly \left(\frac{\Vert\bG_1 \Vert^2}{\sqrt{\btwo \bnu_{0}}}\right)\right).
\end{align*}
\normalsize
Multiplying both sides of the above inequality by $\sqrt{1-\btwo}$ then gives 
\begin{align*}
   \sqrt{1-\btwo} \dE{\sqrt{\sum_{t=1}^\tau \Vert \bg_t \Vert^2 +\frac{\bnu_0}{1-\btwo}}}& \le 3 \sqrt{  \bnu_{0} +2\sigma_0^2T(1-\btwo) } +\frac{1}{4} \ln \E \sqrt{1-\btwo}\sqrt{\sum_{t=1}^\tau \Vert \bg_t \Vert^2 +\frac{\bnu_0}{1-\btwo}},
\end{align*}
where we use $\eta = \frac{\Poly}{\sqrt{T}}$, $1-\beta_2 = \frac{\Poly}{T}$, and $T \ge \Poly$. Therefore, we obtain  $ \sqrt{1-\btwo} \dE{\sqrt{\sum_{t=1}^\tau \Vert \bg_t \Vert^2 +\frac{\bnu_0}{1-\btwo}}}\le  6  \sqrt{  \bnu_{0} +2\sigma_0^2T(1-\btwo) }$.

Therefore, Eq. (\ref{eq: stopping_time}) can be rewritten as 
\begin{align}
     \nonumber
        &\frac{1}{16} \eta \E \sum_{t=1}^\tau\frac{\Vert \bG_t \Vert^2}{\sqrt{\btwo\bnu_{t-1}}}\le  f(\bu_1)- \E f(\bu_{\tau+1})+\frac{ 128 L_0\eta^2}{(1-\btwo)(1-\bone)^2}   \left(\ln \frac{6  \sqrt{  \bnu_{0} +2\sigma_0^2T(1-\btwo) }}{\sqrt{\bnu_0}} -T\ln \btwo \right)
        \\
        \nonumber
   & + \frac{\eta^3}{1-\btwo} \Poly \frac{1}{\sqrt{\btwo \bnu_0}}
    \\
\nonumber
     &+\left(\frac{\eta^3}{1-\btwo}+\eta\right) \Poly \left(\frac{\Vert\bG_1 \Vert^2}{\sqrt{\btwo \bnu_{0}}}\right). 
    \end{align}

We then execute the second round of divide-and-conquer. To begin with, we have that
\begin{equation}
\nonumber
   \dE \sum_{t=1}^\tau\left[  \frac{\Vert \bG_{t}\Vert^2}{\sqrt{\btwo \bnu_{t-1}}}\mathds{1}_{\Vert G_{t}\Vert \ge \frac{\sigma_0}{\sigma_1}} \right]\le \dE\sum_{t=1}^\tau\left[  \frac{\Vert \bG_{t}\Vert^2}{\sqrt{\btwo \bnu_{t-1}}} \right].
\end{equation}

On the other hand, we have that
\begin{align*}
    \frac{\Vert \bG_{t}\Vert^2}{\sqrt{\btwo \bnu_{t-1}}}\mathds{1}_{\Vert G_{t}\Vert \ge \frac{\sigma_0}{\sigma_1}}\ge
   \frac{1}{2\sigma_1^2} \dE^{|\fil_t}\frac{\Vert \bg_t\Vert^2}{\sqrt{\bnu_{t}}+\sqrt{\btwo \bnu_{t-1}}}\mathds{1}_{\Vert G_{t}\Vert \ge \frac{\sigma_0}{\sigma_1}} \ge \frac{1}{2(1-\btwo)\sigma_1^2} \dE^{|\gF_t}\left[\left(\sqrt{\bnu_{t}}-\sqrt{\btwo\bnu_{t-1}}\right)\mathds{1}_{\Vert G_{t}\Vert \ge \frac{\sigma_0}{\sigma_1}}\right].
\end{align*}
%  As a conclusion, 
% \begin{align*}
%       &\sum_{t=1}^\tau\dE\left[  \frac{\Vert \bG_{t}\Vert^2}{\sqrt{\btwo \bnu_{t-1}}}\mathds{1}_{\Vert G_{t}\Vert \ge \frac{\sigma_0}{\sigma_1}} \right] \ge  \frac{1}{2}\sum_{t=1}^\tau\dE\left[\frac{\frac{\btwo}{3\sigma_1^2}\Vert \bg_{t}\Vert^2+\frac{1-\btwo}{3\sigma_1^2}\sigma_0^2}{\sqrt{\bnut_{t+1}}+\sqrt{\btwo \bnut_{t}}}\mathds{1}_{\Vert G_{t}\Vert \ge \frac{\sigma_0}{\sigma_1}}\right]
%       \\
%       \ge& \frac{1}{6(1-\btwo)\sigma_1^2}\sum_{t=1}^\tau \dE\left[\left(\sqrt{\bnut_{t+1}}-\sqrt{\btwo\bnut_{t}}\right)\mathds{1}_{\Vert G_{t}\Vert \ge \frac{\sigma_0}{\sigma_1}}\right].
% \end{align*}

Meanwhile, recall $\{\bar{\bnu}_{t}\}_{t=0}^{\infty}$ as $\bar{\bnu}_{0}= \bnu_{0}$, $\bar{\bnu}_{t}= \btwo\bar{\bnu}_{t-1}+(1-\btwo)\vert g_{t} \vert^2\mathds{1}_{\Vert \bG_{t}\Vert  < \frac{\sigma_0^2}{\sigma_1^2}}$ and $\bar{\bnu}_{t}\le \bnu_{t} $. Therefore
\begin{align*}
    & \dE\sum_{t=1}^\tau\left[\left(\sqrt{\bnu_{t}}-\sqrt{\btwo\bnu_{t-1}}\right)\mathds{1}_{\Vert \bG_{t}\Vert  < \frac{\sigma_0^2}{\sigma_1^2}} \right]
    \\
    =& \dE \sum_{t=1}^\tau\left(\sqrt{\btwo \bnu_{t-1}+(1-\btwo)\Vert g_{t}\Vert^2 }-\sqrt{\btwo\bnu_{t-1}}\right)\mathds{1}_{\Vert \bG_{t}\Vert  < \frac{\sigma_0^2}{\sigma_1^2}}
    \\
    \le & \dE\sum_{t=1}^\tau\left(\sqrt{\btwo^2 \bar{\bnu}_{t-1}+ \btwo(1-\btwo)\Vert g_{t}\Vert^2 + (1-\btwo )\sigma_0^2}-\sqrt{\btwo(\btwo\bar{\bnu}_{t-1}+ (1-\btwo )\sigma_0^2)}\right)\mathds{1}_{\Vert \bG_{t}\Vert  < \frac{\sigma_0^2}{\sigma_1^2}}
    \\
    \le &  \dE\sum_{t=1}^\tau\left(\sqrt{\btwo \bar{\bnu}_{t-1}+ (1-\btwo)\Vert g_{t}\Vert^2\mathds{1}_{\Vert \bG_{t}\Vert  < \frac{\sigma_0^2}{\sigma_1^2}} }-\sqrt{\btwo\bar{\bnu}_{t-1}}\right)
    \\
    =&  \dE\sum_{t=1}^\tau\left(\sqrt{\bar{\bnu}_{t} }-\sqrt{\btwo\bar{\bnu}_{t-1}}\right)
    \\
    =& \dE \sqrt{\bar{\bnu}_{\tau}}+(1-\sqrt{\btwo})\sum_{t=1}^{\tau-1} \dE\sqrt{\bar{\bnu}_{t}} - \dE\sqrt{\btwo\bnu_0}
    \\
    \le & \dE \sqrt{\bar{\bnu}_{\tau}}+(1-\sqrt{\btwo})\sum_{t=1}^{T} \dE\sqrt{\bar{\bnu}_{t}} - \dE\sqrt{\btwo\bnu_0}
    \\
    \le & \dE \sqrt{\bar{\bnu}_{\tau}}+(1-\sqrt{\btwo})T \sigma_0 - \dE\sqrt{\btwo\bnu_0}.
\end{align*}
All in all, summing the above two inequalities together, we obtain that
\small
\begin{align}
\nonumber
&\dE \sqrt{\bnu_{\tau} }+(1-\sqrt{\btwo}) \dE\sum_{t=1}^{\tau-1}\sqrt{\bnu_{t} } - \sqrt{\btwo{\bnu}_{0}}
\\
\nonumber
    =& \dE \sum_{t=1}^\tau \left(\sqrt{\bnu_{t} }-\sqrt{\btwo{\bnu}_{t-1}}\right)
    \\
\nonumber
    \le & \dE  \sum_{t=1}^\tau\left(\sqrt{\bnu_{t} }-\sqrt{\btwo{\bnu}_{t-1}}\right)\mathds{1}_{\Vert G_{t}\Vert \ge \frac{\sigma_0}{\sigma_1}}
    +\dE \sum_{t=1}^\tau \left(\sqrt{\bnu_{t} }-\sqrt{\btwo{\bnu}_{t-1}}\right)\mathds{1}_{\Vert \bG_{t}\Vert  < \frac{\sigma_0^2}{\sigma_1^2}}
    \\
\nonumber
\le & 2(1-\btwo)\sigma_1^2\sum_{t=1}^\tau\dE\left[  \frac{\Vert \bG_{t}\Vert^2}{\sqrt{\btwo \bnu_{t-1}}} \right]+\dE \sqrt{\bar{\bnu}_{\tau}}+(1-\sqrt{\btwo})T \sigma_0 - \dE\sqrt{\btwo\bnu_0}.
\end{align}
\normalsize
Since  $\bnu_0 = \bar{\bnu}_0$ and $\dE \sqrt{\bnu_{\tau} }  \ge \dE\sqrt{\bar{\bnu}_{\tau} } $, we obtain

\begin{align}
\nonumber
    (1-\sqrt{\btwo}) \dE\sum_{t=0}^{\tau-1}\sqrt{\bnut_{t} }
\le&  {2(1-\btwo)\sigma_1^2}\dE\sum_{t=1}^\tau\left[  \frac{\Vert \bG_{t}\Vert^2}{\sqrt{\btwo \bnu_{t-1}}} \right]+ (1-\sqrt{\btwo}) T \sigma_0
.
\end{align}
\normalsize

Dividing both sides of the above equation by $1-\sqrt{\btwo}$ then gives 
\begin{align}
\nonumber
   \dE  \sum_{t=0}^{\tau-1}\sqrt{\bnut_{t} }
\le&   {4\sigma_1^2}\dE\sum_{t=1}^\tau\left[  \frac{\Vert \bG_{t}\Vert^2}{\sqrt{\btwo \bnu_{t-1}}} \right]+ T \sigma_0.
\end{align}
By applying Eq. (\ref{eq: first_stage_renewed}) and the constraint of $\tau$, we obtain that 
\begin{align*}
    \dE \sum_{t=0}^{\tau-1}\sqrt{\bnu_{t} } \le &T \sigma_0+64\frac{\sigma_1^2}{\eta} \left( f(\bu_1)- \E f(\bu_{\tau+1})+\frac{ 64 L_0\eta^2}{1-\btwo}   \left(\ln \frac{6  \sqrt{  \bnu_{0} +2\sigma_0^2T(1-\btwo) }}{\sqrt{\bnu_0}} -T\ln \btwo \right)\right.
        \\
        \nonumber
   & + \left. \frac{\eta^3}{1-\btwo} \Poly \frac{1}{\sqrt{\btwo \bnu_0}}\right.
    \\
\nonumber
     &+\left. \left(\frac{\eta^3}{1-\btwo}+\eta\right) \Poly \left(\frac{\Vert\bG_1 \Vert^2}{\sqrt{\btwo \bnu_{0}}}\right)\right)
         \\
         \le & 4T\sigma_0.
\end{align*}
Here the last inequality is due to $\eta = \frac{\Poly}{\sqrt{T}}$, $1-\beta_2 = \frac{\Poly}{T}$, and $T \ge \Poly$.
Therefore, by Cauchy's inequality, we obtain that
\begin{align*}
    \left(\E \sum_{t=1}^{\tau} \Vert \nabla f(\bw_t) \Vert\right)^2 \le & \left(  \dE \sum_{t=0}^{\tau-1}\sqrt{\bnu_{t} } \right) \left(\dE\sum_{t=1}^{\tau}\left[  \frac{\Vert \bG_{t}\Vert^2}{\sqrt{\btwo \bnu_{t-1}}} \right]\right)
    \\
    \le & 4T\sigma_0\times \frac{16}{\eta} \left(f(\bw_1)- f^*+\frac{ 64 L_0\eta^2}{1-\btwo}   \left(\ln \frac{6  \sqrt{  \bnu_{0} +2\sigma_0^2T(1-\btwo) }}{\sqrt{\bnu_0}} -T\ln \btwo \right)\right.
        \\
        \nonumber
   & + \left. \frac{\eta^3}{1-\btwo} \Poly \frac{1}{\sqrt{\btwo \bnu_0}}\right.
    \\
\nonumber
     &+\left. \left(\frac{\eta^3}{1-\btwo}+\eta\right) \Poly \left(\frac{\Vert\bG_1 \Vert^2}{\sqrt{\btwo \bnu_{0}}}\right)\right)
     \\
     \overset{(*)}{\le }& 4T\sigma_0\times \frac{16}{\eta} \left(3(f(\bw_1)- f^*)+\frac{ 128 L_0\eta^2}{(1-\btwo)(1-\bone)^2}   \left(\ln \frac{6  \sqrt{  \bnu_{0} +2\sigma_0^2T(1-\btwo) }}{\sqrt{\bnu_0}} -T\ln \btwo \right)\right)
     \\
    \overset{(\bullet)}{\le} &  64T\sigma_0\left(3\sqrt{TL_0(f(\bw_1)- f^*)}+\frac{ 128 L_0\eta}{(1-\btwo)(1-\bone)^2T} T  \left(\ln \frac{6  \sqrt{  \bnu_{0} +2\sigma_0^2T(1-\btwo) }}{\sqrt{\bnu_0}} \right)\right)
    \\
    \le & 64T\sigma_0\left(387(1-\bone)\sqrt{TL_0(f(\bw_1)- f^*)}\right),
\end{align*}
where inequality $(*)$ is due to   $\eta = \frac{\Poly}{\sqrt{T}}$, $1-\beta_2 = \frac{\Poly}{T}$, and $T \ge \Poly$, inequality $(\bullet)$ is due to that $\eta = \frac{\sqrt{f(\bw_1)-f^*}}{\sqrt{L_0}}$, and last inequality is due to $\frac{1}{(1-\btwo) T}\ln \frac{6  \sqrt{  \bnu_{0} +2\sigma_0^2T(1-\btwo) }}{\sqrt{\bnu_0}}\le 6 $.

We then consider two cases: $\tau <T$ and $\tau =T$: for the first case, we have that according to the definition of $\tau$ 
\begin{equation*}
    \E \min_{t\in [1,T]} \Vert \bG_t \Vert \mathds{1}_{\tau<T} \le \frac{\sqrt[4]{\sigma_0^2 L_0 (f(\bw_1)-f^*)}}{\sqrt[4]{T}} .
\end{equation*}
For the latter case, we have 
\begin{align*}
    \E \min_{t\in [1,T]} \Vert \bG_t \Vert \mathds{1}_{\tau=T} \le  \frac{1}{T}\left(\E \sum_{t=1}^{T} \Vert \nabla f(\bw_t) \Vert \mathds{1}_{\tau=T}\right) \le  \frac{1}{T}\left(\E \sum_{t=1}^{\tau} \Vert \nabla f(\bw_t) \Vert \right)\le \frac{256\sqrt[4]{\sigma_0^2 L_0 (f(\bw_1)-f^*)} }{\sqrt[4]{T}}.
\end{align*}

Summing the two inequalities above complete the proof.
 \end{proof}

\subsection{Proof of Theorem \ref{thm: parameter_agnostic}}

\begin{proof}
   As described in Section \ref{sec: reach_lower_bound}, the proof immediately follows by several modifications of the proof of Theorem \ref{thm: attain_lower_bound_appen}: 
    \begin{itemize}
        \item We start our analysis for $t \ge \frac{L_1^4128^4\sigma_1^4}{(1-\bone)^4}$. For $t\le \frac{L_1^4128^4\sigma_1^4}{(1-\bone)^4}$, the function value can be bounded by constant as $\frac{L_1^4128^4\sigma_1^4}{(1-\bone)^4}$ is independent of $t$;
        \item For $t \ge \frac{L_1^4128^4\sigma_1^4}{(1-\bone)^4}$, we have $\max\{\Vert \bu_{t}-\bw_t \Vert,\Vert \bu_{t+1}-\bw_t \Vert\} \le \frac{1}{L_1}$ since $\Vert \bw_{t+1}-\bw_t \Vert \le \frac{1}{\sqrt[4]{t}}$ which also can be used to prove Eq. (\ref{eq: lower_bound_G_t}). 
    \end{itemize}

    The proof is completed.
\end{proof}

%%%%%%%%%%%%%%%%%%%%%%%%%%%%%%%%%%%%%%%%%%%%%%%%%%%%%%%%%%%%%%%%%%%%%%%%%%%%%%%
%%%%%%%%%%%%%%%%%%%%%%%%%%%%%%%%%%%%%%%%%%%%%%%%%%%%%%%%%%%%%%%%%%%%%%%%%%%%%%%

\end{document}